\documentclass{article}

\pdfoutput=1

\usepackage{microtype}
\usepackage{graphicx}
\usepackage{booktabs} 
\usepackage{amsfonts}
\usepackage{amsmath}
\usepackage{subfig}

\usepackage{hyperref}

\usepackage{definitions}
\usepackage[hmargin=2.5cm, vmargin=2.3cm]{geometry}

\title{\textbf{\Large{Approximate Leave-One-Out for Fast Parameter Tuning\\ in High Dimensions}}}
\author{\small Shuaiwen Wang\tsup{1,*}, Wenda Zhou\tsup{1,*}, Haihao Lu\tsup{2}, Arian Maleki\tsup{1}, Vahab Mirrokni\tsup{3}}

\date{}

\begin{document}
\maketitle

{\let\thefootnote\relax\footnotetext{
    \tsup{1}Department of Statistics, Columbia University, New York, USA;
    \tsup{2}Mathematics Department and Operation Research Center, Massachusetts Institute of Technology, Massachusetts, USA;
    \tsup{3}Google Research, New York, USA;
    \tsup{*}Equal contributions.
}}

\vskip 0.3in

\begin{abstract}
Consider the following class of learning schemes:
\begin{equation}\label{eq:main-problem1}
    \hat{\betav} := \argmin_{\betav}\;\sum_{j=1}^n \ell(\xv_j^\top\betav; y_j) + \lambda R(\betav),
\end{equation}
where $\xv_i \in \mathbb{R}^p$ and $y_i \in \mathbb{R}$ denote the $i^{\rm
th}$ feature and response variable respectively. Let $\ell$ and $R$ be
the loss function and regularizer, $\betav$ denote the unknown weights, and
$\lambda$ be a regularization parameter. Finding the optimal choice of
$\lambda$ is a challenging problem in high-dimensional regimes where both $n$
and $p$ are large.
We propose two frameworks to obtain
a computationally efficient approximation ALO of the leave-one-out cross validation
(LOOCV) risk for nonsmooth losses and regularizers.
Our two frameworks are based on the primal and dual formulations of
\eqref{eq:main-problem1}. We prove the equivalence of the two approaches under
smoothness conditions. This
equivalence enables us to justify the accuracy of both methods under such conditions.
We use our approaches to obtain a
risk estimate for several standard problems, including generalized
LASSO, nuclear norm regularization, and support vector machines. We 
empirically demonstrate the effectiveness
of our results for non-differentiable cases.

\end{abstract}

\section{Introduction}
\label{intro}
\subsection{Motivation}
Consider a standard prediction problem in which a dataset $\{(y_j,
\xv_j)\}_{j=1}^n \subset \mathbb{R}\times\mathbb{R}^{p}$ is employed to learn
a model for inferring information about new datapoints that are yet to
be observed. One of the most popular classes of learning schemes, specially
in high-dimensional settings, studies the following optimization
problem:
\begin{equation}\label{eq:main-problem}
    \hat{\betav} := \argmin_{\betav}\;\sum_{j=1}^n \ell(\xv_j^\top\betav; y_j) + \lambda R(\betav),
\end{equation}
where $\ell: \mathbb{R}^2 \rightarrow \mathbb{R}$ is the loss function, $R:
\mathbb{R}^p \rightarrow \mathbb{R}$ is the regularizer, and $\lambda$ is
the tuning parameter that specifies the amount of regularization. By applying
an appropriate regularizer in \eqref{eq:main-problem},
we are able to achieve better bias-variance trade-off
and pursue special structures such as sparsity and low rank structure.
However, the performance of such techniques hinges upon the selection of tuning
parameters.

The most generally applicable tuning method is cross validation
\cite{stone1974cross}. One common choice
is $k$-fold cross validation, which however presents potential bias issues in
high-dimensional settings where $n$ is comparable to $p$. For instance,
the phase transition phenomena that happen in such regimes
\cite{amelunxen2014living, donoho2009message, donoho2005neighborliness}
indicate that any data splitting may cause dramatic effects on the solution of
\eqref{eq:main-problem} (see Figure \ref{fig:lasso-risk-loocv-kfold} for an
example). Hence, the
risk estimates obtained from $k$-fold cross validation may not be reliable.
The bias issues of $k$-fold cross validation may be alleviated by choosing
the number of folds $k$ to be large. However, such schemes
are computationally demanding and may not be useful for emerging
high-dimensional applications. An alternative choice of cross validation is
LOOCV, which is unbiased in high-dimensional problems. However, the computation
of LOOCV requires training the model $n$ times, which is unaffordable for large
datasets.

\begin{figure}[t!]
    \begin{center}
        \includegraphics[scale=0.4]{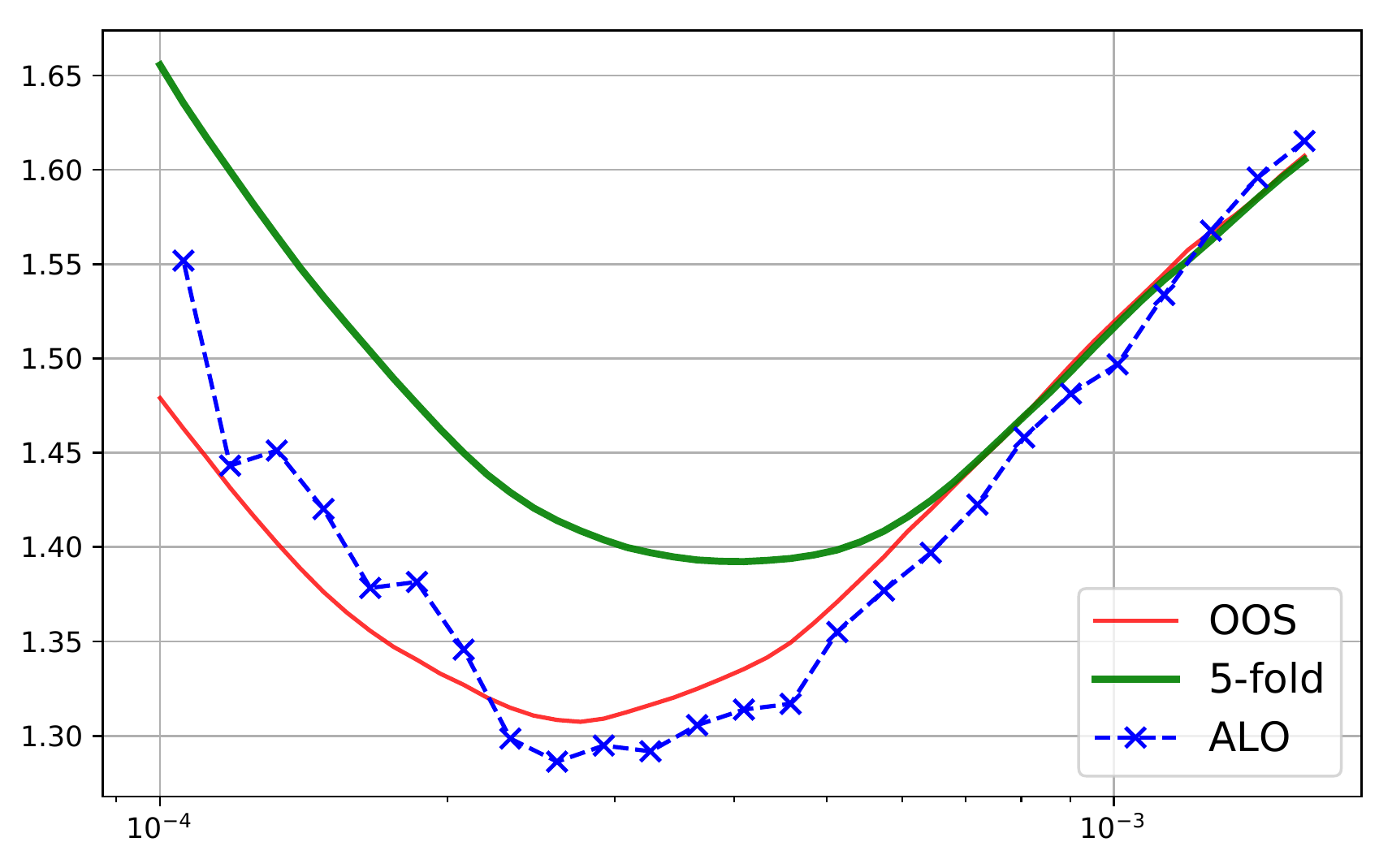}
        \caption{Risk estimates of LASSO based on 5-fold CV and ALO proposed in
        this paper, compared with the true out-of-sample prediction error
        (OOS). In this example, 5-fold CV provides biased estimates of OOS,
        while ALO works just fine. Here we use $n=5000$, $p=4000$ and $iid$ Gaussian design.}
        \label{fig:lasso-risk-loocv-kfold}
    \end{center}
\end{figure}

The high computational complexity of LOOCV has motivated researchers to propose
computationally less demanding approximations of the quantity. Early examples
offered approximations for the case $R(\betav) = \frac{1}{2}\|\betav\|_2^2$ and
the loss function being smooth \cite{allen1974relationship, o1986automatic,
le1992ridge, cawley2008efficient, meijer2013efficient, opper2000gaussian}.
In \cite{beirami2017optimal}, the authors considered such approximations for
smooth loss functions and smooth regularizers. In
this line of work, the accuracy of the approximations was either not studied or
was only studied in the $n$ large, $p$ fixed regime. In a recent paper,
\cite{kamiar2018scalable} employed a similar approximation strategy to obtain
approximate leave-one-out formulas for smooth loss functions and smooth
regularizers. They show that under some mild conditions, such approximations
are accurate in high-dimensional settings. Unfortunately, the approximations
offered in \cite{kamiar2018scalable} only cover twice differentiable loss
functions and regularizers. On the other hand, numerous modern regularizers,
such as generalized LASSO and nuclear norm, and also many loss functions are
not smooth.

In this paper, we propose two powerful frameworks for
calculating an approximate leave-one-out estimator (ALO) of the LOOCV risk
that are capable of offering accurate parameter tuning
even for non-differentiable losses and regularizers. Our first approach is
based on the smoothing and quadratic approximation of the primal problem
\eqref{eq:main-problem}. The second approach is based on the approximation of
the dual of \eqref{eq:main-problem}. While the
two approaches consider different approximations that happen in different domains,
we will show that when both $\ell$ and $r$ are twice
differentiable, the two frameworks produce the same ALO formulas, which are
also the same as the formulas proposed in \cite{kamiar2018scalable}.

We use our platforms to obtain concise formulas for several popular examples
including generalized LASSO, support vector machine (SVM) and nuclear norm
minimization. As will be clear from our examples, despite of the equivalence of
the two frameworks for smooth loss functions and regularizers, the technical
aspects of the derivations involved for obtaining ALO formulas have major
variations in different examples. Finally, we present extensive simulations to
confirm the accuracy of our formulas on various important machine learning
models. Code is available at \url{github.com/wendazhou/alocv-package}.

\subsection{Other Related Work}
The importance of parameter tuning in learning systems has encouraged many
researchers to study this problem from different perspectives. In addition to
cross validation, several other
approaches have been proposed including Stein's unbiased risk estimate
(SURE), Akaike information criterion (AIC), and Mallow's $C_p$. While AIC is
designed for smooth parametric models, SURE has been extended to emerging
optimization problems, such as generalized LASSO and nuclear norm
minimization \cite{candes2013unbiased, dossal2013degrees, tibshirani2012degrees,
vaiter2017degrees, zou2007degrees}.

Unlike cross validation which approximates the out-of-sample prediction
error, SURE, AIC, and $C_p$ offer estimates for in-sample prediction error
\cite{EOSL:chapter7}. This makes cross validation more appealing for
many learning systems. Furthermore, unlike ALO,
both SURE and $C_p$ only work on linear models (and not generalized linear
models) and their unbiasedness is only guaranteed under the Gaussian model
for the errors. There has been little success in extending SURE beyond this
model \cite{efron2004estimation}. 

Another class of parameter tuning schemes are based on approximate message
passing framework \cite{bayati2013estimating, mousavi2017consistent,
obuchi2016cross}. As pointed out in \cite{obuchi2016cross}, this approach is
intuitively related to LOOCV. It offers
consistent parameter tuning in high-dimensions \cite{mousavi2017consistent},
but the results strongly depend on the independence of the elements of $\Xv$.

\subsection{Notation}

Lowercase and uppercase bold letters denote vectors and matrices,
respectively. For subsets $A \subset \{1,2, \ldots, n\}$ and $B \subset \{ 1,2,
\ldots, p\}$ of indices and a matrix $\Xv$, let $\Xv_{A, \cdot}$
and $\Xv_{\cdot, \Bv}$ denote the submatrices that include only
rows of $\Xv$ in $A$, and columns of $\Xv$ in $B$ respectively. Let
$\{a_i\}_{i \in S}$ denote the vector whose components are $a_i$ for $i \in S$. We
may omit $S$, in which case we consider all indices valid in the context.
For a function $f: \mathbb{R} \rightarrow \mathbb{R}$, let
$\dot{f}$, $\ddot{f}$ denote its 1\textsuperscript{st} and
2\textsuperscript{nd} derivatives. For a
vector $\av$, we use $\diag[\av]$ to denote a diagonal matrix
$\Av$ with $A_{ii} = a_i$. Finally,
let $\nabla R$ and $\nabla^2 R$ denote the gradient and Hessian of
a function $R: \mathbb{R}^p \rightarrow \mathbb{R}$.

\section{Preliminaries}

\subsection{Problem Description}
In this paper, we study the statistical learning models in form
\eqref{eq:main-problem}. For each value of $\lambda$, we evaluate the
following LOOCV risk estimate with respect to some error function $d$:
\begin{equation}\label{eq:err-func}
    \loo_{\lambda}: = \frac{1}{n}\sum_{i=1}^n d(y_i, \xv_i^\top \leavei{\estim{\betav}}),
\end{equation}
where $\estimi{\betav}$ is the solution of the leave-$i$-out problem
\begin{equation}\label{eq:leave-i-out}
    \estimi{\betav} := \argmin_{\betav}\; \sum_{j\neq i} \ell(\xv_j^\top\betav; y_j) + \lambda R(\betav).
\end{equation}

Calculating \eqref{eq:leave-i-out} requires training the model $n$ times,
which may be time-consuming in high-dimensions. As an alternative, we
propose an estimator $\surrogi{\betav}$ to approximate $\estimi{\betav}$
based on the full-data estimator $\estim{\betav}$ to reduce the computational
complexity. We consider two frameworks for obtaining $\surrogi{\betav}$, and
denote the corresponding risk estimate by:
\begin{equation} \label{eq:alo-formula}
    \alo_\lambda: = \frac{1}{n}\sum_{i = 1}^n d(y_i, \xv_i^\top \surrogi{\betav}).
\end{equation}

The estimates we obtain will be called approximated leave-one-out (ALO) throughout the paper. 

\subsection{Primal and Dual Correspondence}

The objective function of penalized regression problem with loss $\ell$ and
regularizer $R$ is given by:
\begin{equation} \label{eq:methods:primal}
    P(\betav) := \sum_{j = 1}^n \ell(\xv_j^\top \betav; y_j) + R(\betav).
\end{equation}

Here and subsequently, we absorb the value of $\lambda$ into $R$
to simplify the notation. We also consider the Lagrangian dual problem, which can be written in the form:
\begin{equation} \label{eq:methods:dual}
    \min_{\dualv \in \mathbb{R}^n} D(\thetav) := \sum_{j = 1}^n
    \ell^*(-\theta_j; y_j) + R^*(\Xv^\top \thetav),
\end{equation}
where $\ell^*$ and $R^*$ denote the \textit{Fenchel conjugates}
\footnote{The Fenchel conjugate $f^*$ of a function $f$ is defined as
$f^*(x):=\sup_y\{\langle x, y\rangle - f(y)\}$.}
of $\ell$ and $R$ respectively. See the derivation in Appendix \ref{sec:dual-derivation}.

It is known that under mild conditions, \eqref{eq:methods:primal}
and \eqref{eq:methods:dual} are equivalent \cite{boyd2004convex}. In this case,
we have the primal-dual correspondence relating the primal optimal
$\estim{\betav}$ and the dual optimal $\estim{\thetav}$:
\begin{equation} \label{eq:primal-dual-correspondence}
\begin{gathered}
    \estim{\betav} \in \partial R^* (\Xv^\top  \estim{\dualv}), \quad
    \Xv^\top \estim{\dualv} \in \partial R(\estim{\betav}), \\
    \xv_j^\top\estim{\betav} \in \partial \ell^* (-\estim{\duals}_j; y_j), \quad
    - \estim{\duals}_j \in \partial \ell(\xv_j^\top \estim{\betav}; y_j),
\end{gathered}
\end{equation}
where $\partial f$ denotes the set of subgradients of a function $f$.
Below we will use both primal and dual perspectives for approximating
$\loo_\lambda$.

\section{Approximation in the Dual Domain}\label{sec:approximatedual}

\subsection{The First Example: LASSO}
\label{sec:example:lasso}

Let us first start with a simple example that illustrates our dual method in deriving
an approximate leave-one-out (ALO) formula for the standard LASSO. The LASSO
estimator, first proposed in \cite{tibshirani1996regression}, can be formulated as
the penalized regression framework in \eqref{eq:methods:primal} by setting
$\ell(\mu; y) = (\mu - y)^2 / 2$, and $R(\betav) = \lambda\norm{\betav}_1$.

We recall the general formulation of the dual for penalized regression problems
\eqref{eq:methods:dual}, and note that in the case of the LASSO we have:
\begin{equation*}
    \ell^*(\duals_i; y_i) = \frac{1}{2}(\duals_i - y_i)^2, \\
    \quad
    R^*(\betav) = \begin{cases}
        0 & \text{ if } \norm{\betav}_\infty \leq \lambda, \\
        +\infty & \text{ otherwise.}
    \end{cases}
\end{equation*}

In particular, we note that the solution of the dual problem
\eqref{eq:methods:dual} can be obtained from:
\begin{equation}
    \label{eq:lasso:projection}
    \estim{\dualv} = \Pi_{\Delta_X}(\yv).
\end{equation}

Here $\Pi_{\Delta_X}$ denotes the projection onto $\Delta_X$, where $\Delta_X$
is the polytope given by:
\begin{equation*}
    \Delta_X
    =
    \{ \dualv \in \RR^n : \norm{\Xv^\top \dualv}_{\infty} \leq \lambda \}.
\end{equation*}

Let us now consider the leave-$i$-out problem. Unfortunately, the dimension of
the dual problem is reduced by 1 for the leave-$i$-out problem, making it
difficult to leverage the information from the full-data solution to help
approximate the leave-$i$-out solution. We propose to augment the leave-$i$-out
problem with a virtual $i$\tsup{th} observation which does not affect the result of
the optimization, but restores the dimensionality of the problem.

More precisely, let $\yv_a$ be the same as $\yv$, except that its
$i$\tsup{th} coordinate is replaced by $\leavei{\estim{y}}_i = \xv_i^\top
\leavei{\estim{\betav}}$, the leave-$i$-out predicted value. We note that the
leave-$i$-out solution $\leavei{\estim{\betav}}$  is also the solution for the following
augmented problem:
\begin{eqnarray}\label{eq:augmentedmethod}
   \min_{\betav \in \mathbb{R}^p}  \sum_{j = 1}^n \ell(\xv_j^\top \betav; y_{a,j}) + R(\betav).
\end{eqnarray}

Let $\leavei{\estim{\dualv}}$ be the corresponding dual solution of
\eqref{eq:augmentedmethod}. Then, by \eqref{eq:lasso:projection}, we know
that
\begin{equation*}
    \leavei{\estim{\dualv}} = \Pi_{\Delta_X} (\yv_a).
\end{equation*}

Additionally, the primal-dual correspondence
\eqref{eq:primal-dual-correspondence} gives that $\leavei{\estim{\dualv}}=\yv_a
- \Xv \leavei{\estim{\betav}}$, which is
the residual in the augmented problem, and hence that $\leavei{\estim{\duals}}_i = 0$.
These two features allow us to characterize the leave-$i$-out
predicted value $\leavei{\estim{y}}_i$, as satisfying:
\begin{equation}
    \label{eq:augmented-dual-equation}
    \ev_i^\top \Pi_{\Delta_X}\big(\yv - (y_i - \leavei{\estim{y}}_i)
    \ev_i\big) = 0,
\end{equation}
where $\ev_i$ denotes the $i$\tsup{th} standard vector. Solving exactly for the above
equation is in general a procedure that is computationally comparable to
fitting the model, which may be expensive. However, we may attempt to obtain an
approximate solution of \eqref{eq:augmented-dual-equation}
by linearizing the projection operator at the full data solution
$\estim{\dualv}$, or equivalently
performing a single Newton step to solve the leave-$i$-out problem from the
full data solution. The approximate leave-$i$-out fitted value
$\leavei{\surrog{y}}_i$ is thus given by:
\begin{equation}\label{eq:dualsimpleJaco}
    \leavei{\surrog{y}}_i = y_i - \frac{\estim{\duals}_i}{J_{ii}},
\end{equation}
where $\Jv$ denotes the Jacobian of the projection operator $\Pi_{\Delta_X}$ at
the full data problem $\yv$.
Note that $\Delta_X$ is a polytope, and thus the projection onto $\Delta_X$ is
almost everywhere locally affine \cite{tibshirani2012degrees}.
Furthermore, it is straightforward to calculate the Jacobian of
$\Pi_{\Delta_X}$. Let $E = \{ j: \abs{\Xv_j^\top \estim{\dualv}} = \lambda \}$ be
the equicorrelation set (where $\Xv_j$ denotes the $j^{\text{th}}$ column of $\Xv$),
then we have that the projection at the full data problem $\yv$ is locally
given by a projection onto the orthogonal complement of the span of $\Xv_{\cdot, E}$,
thus giving $\Jv = \Iv - \Xv_{\cdot, E} (\Xv_{\cdot, E}^\top \Xv_{\cdot,
E})^{-1} \Xv_{\cdot, E}^\top$. We can then obtain $\leavei{\surrog{y}}$ by
plugging $\Jv$ in \eqref{eq:dualsimpleJaco}. Finally, by replaceing $
\xv_i^\top \surrogi{\betav}$ with $\leavei{\surrog{y}}_i$ in
\eqref{eq:alo-formula} we obtain an estimate of the risk.

\subsection{General Case}
\label{ssec:dual:general-case}

In this section we extend the dual approach outlined in Section
\ref{sec:example:lasso} to more general loss functions and regularizers.

\paragraph{General regularizers}
Let us first extend the dual approach
to other regularizers, while the loss function remains $\ell(\mu, y) = (\mu-y)^2/2$.
In this case the dual problem \eqref{eq:methods:dual} has the following form:
\begin{equation}\label{eq:dual:general-regularizer}
    \min_{\dualv} \frac{1}{2} \sum_{j = 1}^n (\duals_j - y_j)^2 + R^*(\Xv^\top \dualv).
\end{equation}

We note that the optimal value of $\dualv$ is by definition the value of the proximal operator of
$R^*(\Xv^\top\cdot)$ at $\yv$:
\begin{equation*}
    \estim{\dualv} = \proxv_{R^*(\Xv^\top \cdot)}(\yv).
\end{equation*}

Following the argument of Section \ref{sec:example:lasso}, we obtain
\begin{equation}\label{eq:dual:alo-least-square}
    \leavei{\surrog{y}}_i = y_i - \frac{\estim{\duals}_i}{J_{ii}},
\end{equation}
with $\Jv$ now denoting the Jacobian of $\proxv_{R^*(\Xv^\top \cdot)}$. We
note that the Jacobian matrix $\Jv$ exists almost everywhere, because the
non-expansiveness of the proximal operator guarantees its almost-everywhere
differentiability \cite{Combettes2011}. In particular, if $\yv$ has distribution
which is absolutely continuous with respect to the Lebesgue measure, $\Jv$
exists with probability 1. This approach is particularly useful when $R$ is a norm, as
its Fenchel conjugate is then the convex indicator of the unit ball of the dual
norm, and the proximal operator reduces to a projection operator.

\paragraph{General smooth loss}
Let us now assume we have a convex smooth loss in \eqref{eq:methods:primal}, such as those that
appear in generalized linear models. As we are arguing from a second-order
perspective by considering Newton's method, we will attempt to expand the
loss as a quadratic form around the full data solution. We will thus consider the
approximate problem obtained by expanding $\ell^*$ around the dual optimal
$\estim{\dualv}$:
\begin{equation}\label{eq:dual-general-optimal}
    \min_{\dualv} \frac{1}{2}\sum_{j = 1}^n \ddot{\ell}^*(-\estim{\duals}_j; y_j)
    \Bigg(
        \duals_j - \estim{\duals}_j
        - \frac{\dot{\ell}^*(-\estim{\duals}_j; y_j)}{\ddot{\ell}^*(-\estim{\duals}_j; y_j)}
    \Bigg)^2 + R^*(\Xv^\top \dualv).
\end{equation}

The constant term has been removed from \eqref{eq:dual-general-optimal} for
simplicity. We note that we have reduced the problem to a problem with a
weighted $\ell_2$ loss which may be further reduced to a simple $\ell_2$
problem by a change of variable and a rescaling of $\Xv$. Indeed, let $\Kv$ be
the diagonal matrix such that $K_{jj} = \sqrt{\ddot{\ell}^*(-\estim{\duals}_j;
y_j)}$, and note that we have: $\dot{\ell}^*(-\estim{\duals}_j; y_j) =
\xv_j^\top \estim{\betav} := \estim{y}_j$ by the primal-dual correspondence
\eqref{eq:primal-dual-correspondence}. Consider the change of variable $\uv =
\Kv \dualv$ to obtain:
\begin{equation*}\label{eq:dual:alo-general}
    \min_{\uv} \frac{1}{2} \sum_{j = 1}^n \left( u_j
    - \frac{\estim{\duals}_j \ddot{\ell}^*(-\estim{\duals}_j; y_j) + \estim{y}_j}
                       {\sqrt{\ddot{\ell}^*(-\estim{\duals}_j; y_j)}}
    \right)^2 + R^*(\Xv^\top \Kv^{-1} \uv).
\end{equation*}

We may thus reduce to the $\ell_2$ loss case in \eqref{eq:dual:general-regularizer}
with a modified $\Xv$ and $\yv$:
\begin{equation}\label{eq:dual:general-loss}
    \Xv_u = \Kv^{-1} \Xv,
    \quad
    \yv_u = \left\{\frac{\estim{\duals}_j \ddot{\ell}^*(-\estim{\duals}_j; y_j) + \estim{y}_j}
    {\sqrt{\ddot{\ell}^*(-\estim{\duals}_j; y_j)}}\right\}_j.
\end{equation}

Similar to \eqref{eq:dual:alo-least-square}, the ALO formula in the case of
general smooth loss can be obtained as $\leavei{\surrog{y}_{i}} =
K_{ii}\leavei{\surrog{y}_{u,i}}$, with
\begin{equation}\label{eq:dual:alo-general-loss}
    \leavei{\surrog{y}_{u,i}}
    =
    y_{u, i} - \frac{K_{ii}\estim{\theta}_i}{J_{ii}},
\end{equation}
where $\Jv$ is the Jacobian of $\proxv_{R^*(\Xv_u^\top\cdot)}$.

\section{Approximation in the Primal Domain}
\label{sec:primal-smoothing}

\subsection{Smooth Loss and Regularizer}\label{ssec:primal-smooth}
To obtain $\loo_\lambda$ we need to solve
\begin{equation}\label{eq:leaveoneoutopt_1}
    \estimi{\betav} := \argmin_{\betav}\; \sum_{j\neq i}
    \ell(\xv_j^\top\betav; y_j) + R(\betav).
\end{equation}

Assuming $\estimi{\betav}$ is close to $\estim{\betav}$, we can take a
\textit{Newton step} from $\estim{\betav}$ towards $\estimi{\betav}$ to obtain
its approximation $\surrogi{\betav}$ as:
\begin{equation}\label{eq:primal-alo-smooth0}
    \surrogi{\betav}
    =
    \estim{\betav}
    +
    \Big[\sum_{j\neq i} \xv_j
    \xv_j^\top\ddot{\ell}(\xv_j^\top\estim{\betav}; y_j) + \nabla^2
R(\estim{\betav})\Big]^{-1} \xv_i\dot{\ell}(\xv_i^\top\estim{\betav}; y_i).
\end{equation}

By employing the matrix inversion lemma \cite{hager1989updating} we obtain:
\begin{equation}\label{eq:primal-alo-smooth}
    \xv_i^\top\surrogi{\betav}
    =
    \xv_i^\top\estim{\betav}
    +
    \frac{H_{ii}}{1 - H_{ii}\ddot{\ell}(\xv_i^\top\estim{\betav};
    y_i)}\dot{\ell}(\xv_i^\top\estim{\betav}; y_i),
\end{equation}
where
\begin{equation}\label{eq:smooth-H}
    \Hv = \Xv[\Xv^\top \diag[\{\ddot{\ell}(\xv_i^\top\estim{\betav};
    y_i)\}_i]\Xv + \nabla^2 R(\estim{\betav})]^{-1}\Xv^\top.
\end{equation}

This is the formula reported in \cite{kamiar2018scalable}. By calculating
$\estim{\betav}$ and $\Hv$ in advance, we can cheaply approximate the
leave-$i$-out prediction for all $i$ and efficiently evaluate the LOOCV risk.
On the other hand, in order to use the above strategy, twice differentiability
of both the loss and the regularizer is necessary in a neighborhood of
$\estim{\betav}$. However, this assumption is violated for many machine
learning models including LASSO, Nuclear norm, and SVM.
In the next two sections, we introduce a smoothing technique which lifts the scope
of the above primal approach to nondifferentiable losses and regularizers.

\subsection{Nonsmooth Loss and Smooth Regularizer}\label{ssec:nonsmooth-loss}
In this section we study the piecewise smooth loss functions and twice
differentiable regularizers. Such problems arise in SVM \cite{cortes1995support}
and robust regression \cite{huber1973robust}. Before proceeding further, we
clarify our assumptions on the loss function.
\begin{definition}
    A singular point of a function is called $q$\textsuperscript{th} order, if
    at this point the function is $q$ times differentiable, but its
    $(q+1)$\textsuperscript{th} order derivative does not exist.
\end{definition}

Below we assume the loss $\ell$ is piecewise twice differentiable with $k$
zero-order singularities $v_1, \ldots, v_k \in \mathbb{R}$. 
The existence of singularities prohibits us from directly applying strategies in
\eqref{eq:primal-alo-smooth0} and \eqref{eq:primal-alo-smooth}, where twice
differentiability of $\ell$ and $R$ is necessary. A natural solution is to
first smooth out the loss function $\ell$, then apply the framework in the
previous section to the smoothed version and finally reduce the smoothness to
recover the ALO formula for the original nonsmooth problem.

As the first step, consider the following smoothing idea:
\begin{equation*}
    \ell_h(\mu; y) =: \frac{1}{h} \int \ell(u; y) \phi ((\mu- u)/h) du,
\end{equation*}
where $h>0$ is a fixed number and $\phi$ is a symmetric, infinitely many times
differentiable function with the following properties:
\vspace{-2mm}
\begin{description}
    \setlength{\itemsep}{1pt}
    \setlength{\parskip}{1pt}
    \setlength{\parsep}{1pt}
    \item[] \textit{Normalization}:
        $\int \phi(w)dw = 1$, $\phi(w) \geq 0$, $\phi(0)>0$;
    \item[] \textit{Compact support}:
        $\mathrm{supp}(\phi) = [-C, C]$ for some $C>0$.
\end{description}

Now plug in this smooth version $\ell_h$ into \eqref{eq:leaveoneoutopt_1} to
obtain the following formula from \eqref{eq:primal-alo-smooth0}:
\begin{equation}\label{eq:smoothedbeta}
    \surrogi{\betav}_h
    :=
    \estim{\betav}_h + \Big[\sum_{j\neq i} \xv_j
    \xv_j^\top\ddot{\ell}_h(\xv_j^\top\estim{\betav}_h; y_j) + \nabla^2 R(\estim{\betav}_h)\Big]^{-1}
    \xv_i\dot{\ell}_h(\xv_i^\top\estim{\betav}_h; y_i).
\end{equation}
where $\estim{\betav}_h$ is the minimizer on the full data from loss
$\ell_h$ and $R$.  $\surrogi{\betav}_h$ is a good approximation to the leave-$i$-out
estimator $\estimi{\betav}_h$ based on smoothed loss $\ell_h$.

Setting $h \rightarrow 0$, we have that $\ell_h(\mu, y)$ converges to
$\ell(\mu, y)$ uniformly in the region of interest (see Appendix
\ref{append:ssec:property-kernel-smoothing} for the proof), implying that
$\lim_{h\rightarrow 0} \surrogi{\betav}_h$ serves as a good estimator of
$\lim_{h\rightarrow 0}\estimi{\betav}_h$, which is heuristically close to the
true leave-$i$-out $\estimi{\betav}$. Equation \eqref{eq:smoothedbeta} can be
simplified in the limit $h \rightarrow 0$. We define the sets of indices $V$
and $S$ for the samples at singularities and smooth parts respectively:
\begin{align*}\label{eq:indexset}
    V &:= \{ j: \xv_j^\top \estim{\betav} = v_t \text{ for some } t \in \{1, \ldots, k\} \}, \\
    S &:= \{ 1, \dotsc, n \} \setminus V.
\end{align*}
We characterize the limit of $\xv_i^\top\surrogi{\betav}_h$ below.

\begin{theorem}\label{thm:nonsmooth-loss-approx}
    Under some mild conditions, as $h \rightarrow 0$,
    \begin{equation*}
        \xv_i^\top\surrogi{\betav}_h \rightarrow  \xv_i^\top\estim{\betav}+
        a_i g_{\ell,i},
    \end{equation*}
    where
    \begin{align*}\label{eq:nonsmooth-H}
        a_i &= \begin{cases}
            \frac{W_{ii}}{1 - W_{ii}\ddot{\ell}(\xv_i^\top\estim{\betav};y_i)}
            & \text{ if } i \in S, \\
            \frac{1}{[(\Xv_{V\cdot}\Yv^{-1}\Xv_{V\cdot}^\top)^{-1}]_{ii}} & \text{ if } i \in V, \\
        \end{cases} \\
        \Yv &= \nabla^2 R(\estim{\betav})
        + \Xv_{S\cdot}^\top \diag[\{\ddot{\ell}(\xv_j^\top\estim{\betav})\}_{j \in S}] \Xv_{S\cdot}, \\
        W_{ii} &= \xv_i^\top \Yv^{-1} \xv_i - \xv_i^\top
        \Yv^{-1}\Xv_{V,\cdot}^\top(\Xv_{V,\cdot}
        \Yv^{-1}\Xv_{V,\cdot}^\top)^{-1}\Xv_{V,\cdot} \Yv^{-1}\xv_i.
    \end{align*}
    
    For $i \in S$, $g_{\ell,i} = \dot{\ell}(\xv_i^\top\estim{\betav}; y_i)$, and for $i \in V$, we have:
    \begin{equation}\label{eq:nonsmooth-subg}
        \gv_{\ell,V} = (\Xv_{V,\cdot}\Xv_{V,\cdot}^\top)^{-1}\Xv_{V,\cdot}[\nabla R(\estim{\betav}) -
        \sum_{j \in S}\xv_j\dot{\ell}(\xv_j^\top \estim{\betav}; y_j)]. \nonumber
    \end{equation}
\end{theorem}

We can obtain the ALO estimate of prediction error by plugging
$\xv_i^\top\hat{\betav} + a_i g_{\ell, i}$ instead of $\xv_i^\top
\surrogi{\betav}$ in \eqref{eq:alo-formula}. The conditions and proof of
Theorem \ref{thm:nonsmooth-loss-approx} can be found in the Appendix
\ref{append:ssec:primal-nonsmooth-loss}.

\subsection{Nonsmooth Regularizer and Smooth Loss}\label{ssec:nonsmooth-regularizer}

The smoothing technique proposed in the last section can also handle many
nonsmooth regularizers. In this section we focus on separable regularizers $R$,
defined as $R(\betav) = \sum_{l=1}^p r(\beta_l)$, where $r: \mathbb{R}
\rightarrow \mathbb{R}$ is piecewise twice differentiable with finite number of
zero-order singularities in $v_1, \ldots, v_k \in \mathbb{R}$. (Examples on
non-separable regularizers are studied in Section \ref{sec:applications}.)
We further assume the loss function $\ell$ to be twice differentiable and denote by
$A = \{l: \estim{\beta}_l \neq v_t, \text{ for any } t \in \{1, \ldots, k\} \}$ the active set.

For the coordinates of $\estim{\betav}$ that lie in $A$, our objective
function, constrained to these coordinates, is locally twice differentiable.
Hence we expect $\estimi{\betav}_A$ to be well approximated by the ALO
formula using only $\estim{\betav}_A$. On the other hand, components not in
$A$ are trapped at singularities. Thus as long as they are not on the
boundary of being in or out of the singularities, we expect these locations
of $\estimi{\betav}$ to stay at the same values.

Technically, consider a similar smoothing scheme for $r$:
\begin{equation*}
    r_h(w) = \frac{1}{h}\int r(u)\phi((w - u) / h) du,
\end{equation*}
and let $R_h(\betav) = \sum_{l=1}^p r_h(\beta_l)$. We then consider the ALO
formula of Model \eqref{eq:leaveoneoutopt_1} with regularizer $R_h$.
\begin{equation}\label{eq:nonsmooth-reg:smoothedbeta}
    \surrogi{\betav}_h := \estim{\betav}_h
    + \Big[\sum_{j\neq i} \xv_j\xv_j^\top\ddot{\ell}(\xv_j^\top\estim{\betav}_h; y_j) + \nabla^2 R_h(\estim{\betav}_h)\Big]^{-1}
    \xv_i\dot{\ell}_h(\xv_i^\top\estim{\betav}_h; y_i).
\end{equation}

Setting $h \rightarrow 0$, \eqref{eq:nonsmooth-reg:smoothedbeta} reduces to a
simplified formula which heuristically serves as a good approximation to the true
leave-$i$-out estimator $\estimi{\betav}$, stated as the following theorem:
\begin{theorem}\label{thm:nonsmooth-reg-approx}
    Under some mild conditions, as $h \rightarrow 0$,
    \begin{equation*}\label{eq:nonsmooth-reg:1}
        \xv_i^\top\surrogi{\betav}_h
        \rightarrow
        \xv_i^\top\estim{\betav}
        +
        \frac{H_{ii}\dot{\ell}(\xv_i^\top\estim{\betav}; y_i)}{1 -
        H_{ii}\ddot{\ell}(\xv_i^\top\estim{\betav}; y_i)},
    \end{equation*}
    with
    \begin{equation*}\label{eq:nonsmooth-reg:2}
        \Hv = \Xv_{\cdot,A}[\Xv_{\cdot,A}^\top
        \diag[\{\ddot{\ell}(\xv_i^\top\estim{\betav}; y_i)\}_i]\Xv_{\cdot,A} +
    \nabla^2 R(\estim{\betav}_{A})]^{-1}\Xv_{\cdot,A}^\top.
    \end{equation*}
    \normalsize
\end{theorem}

The conditions and proof of Theorem \ref{thm:nonsmooth-reg-approx} can be found in the
Appendix \ref{append:ssec:primal-nonsmooth-reg}.

\begin{remark}
    For nonsmooth problems, higher order singularities do not cause issues: the
    set of tuning values which cause $\estim{\beta}_l$ (for regularizer) or
    $\xv_j^\top\estim{\betav}$ (for loss) to fall at those higher order
    singularities has measure zero.
\end{remark}

\begin{remark}
    For both nonsmooth losses and regularizers, we need to invert some matrices
    in the ALO formula. Although the invertibility does not seem guaranteed in
    the general formula, as we apply ALO to specific models, the structures of
    the loss and/or the regularizer ensures this invertibility. For example,
    for LASSO, we have that the size of the active set $|E| \leq \min(n, p)$.
\end{remark}

\begin{remark}
    We note that the dual approach is typically powerful for models with smooth
    losses and norm-type regularizers, such as the SLOPE norm and the
    generalized LASSO. On the other hand, the primal approach is valuable for
    models with nonsmooth loss or when the Hessian of the regularizer is
    feasible to calculate. Such regularizers often exhibit some type of
    separability or symmetry, such as in the case of SVM or nuclear norm.
\end{remark}

\section{Equivalence Between Primal and Dual Methods}\label{ssec:primal-dual-equiv}

Although the primal and dual methods may be harder or easier to carry out depending
on the specific problem at hand, one may wonder if they always obtain the same result.
In this section, we outline a unifying view for both methods, and state an equivalence
theorem.

As both the primal and dual methods are based on a first-order approximation strategy,
we will study them not as approximate solutions to the leave-$i$-out problem, but will
instead show that they are exact solutions to a surrogate leave-$i$-out problem. Indeed,
recall that the leave-$i$-out problem is given by \eqref{eq:leave-i-out}, which cannot
be solved in closed form. However, we note that the solution does exist in closed form
in the case where both $\ell$ and $R$ are quadratic functions.

We may thus consider the approximate leave-$i$-out problem, where both $\ell$ and $R$
have been replaced in the leave-$i$-out problem \eqref{eq:leave-i-out}
by their quadratic expansion at the full data solution:
\begin{equation}
    \label{eq:alo:surrog}
    \min_{\leavei{\betav}} \sum_{j \neq i} \surrog{\ell}(\xv_j^\top \leavei{\betav}; y_j)
    + \surrog{R}(\leavei{\betav}).
\end{equation}

When both $\ell$ and $R$ are twice differentiable at the full data solution, $\surrog{\ell}$
and $\surrog{R}$ can be taken to simply be their respective second order Taylor expansions at $\estim{\betav}$.
When $\ell$ or $R$ is not twice differentiable at the full data solution, we have
seen that it is still possible to obtain an ALO estimator through the proximal map (in the case
of the dual) or through smoothing arguments (in the case of the primal). The corresponding
quadratic surrogates may then be formulated as partial quadratic functions, that is, convex
quadratic functions restricted to an affine subspace. However, due to space
limitations we only focus on twice differentiable losses and regularizers
here.

The way we obtain $\leavei{\surrog{\betav}}$ in \eqref{eq:primal-alo-smooth0}
indicates that the primal formula in \eqref{eq:primal-alo-smooth} and
\eqref{eq:smooth-H} are the exact leave-$i$-out solution of the 
surrogate primal problem \eqref{eq:alo:surrog}.
On the other hand, we may also wish to consider the surrogate dual problem, by replacing
$\ell^*$ and $R^*$ by their quadratic expansion at full data dual solution $\estim{\dualv}$
in the dual problem \eqref{eq:methods:dual}. One may possibly worry that the surrogate dual
problem is then different from the dual of the surrogate primal problem \eqref{eq:alo:surrog}. This
does not happen, and we have the following theorem.
\begin{theorem}\label{thm:primal-dual-equivalence}
   Let $\ell$ and $R$ be twice differentiable convex functions. Let
   $\surrog{\ell}$ and $\surrog{R}$ denote the quadratic surrogates of the
   loss and regularizer
    at the full data solution $\estim{\betav}$, and let $\surrog{\ell}_D^*$ and $\surrog{R}_D^*$
    denote the quadratic surrogates of the conjugate loss and regularizer at the dual full data
    solution $\estim{\dualv}$. We have that the following problems are equivalent (have the same minimizer):
    \begin{gather}
        \label{eq:equivalence:primal-surrogate}
        \min_{\dualv} \sum_{j = 1}^n \surrog{\ell}^*(-\duals_j; y_j) +
        \surrog{R}^*(\Xv^\top \dualv), \\
        \label{eq:equivalence:dual-surrogate}
        \min_{\dualv} \sum_{j = 1}^n \surrog{\ell}^*_D(-\duals_j; y_j) +
        \surrog{R}^*_D(\Xv^\top \dualv).
    \end{gather}
\end{theorem}

Additionally, we note that the dual method described in Section
\ref{sec:approximatedual} solves the surrogate dual problem
\eqref{eq:equivalence:dual-surrogate}.
\begin{theorem}\label{thm:primal-dual-equivalence-2}
    Let $\Xv_u$, $\yv_u$ be as in \eqref{eq:dual:general-loss},
    and let $\leavei{\surrog{y}}_{u,i}$ be the transformed ALO obtained in
    \eqref{eq:dual:alo-general-loss}. Let $\tilde{\yv}_a$ be the same as $\yv_u$
    except $\tilde{y}_{a,i} = \leavei{\surrog{y}}_{u,i}$. Then
    $\surrog{\yv}_a$ satisfies
    \begin{equation}
        [\proxv_{\tilde{g}}(\tilde{\yv}_a)]_i = 0,
    \end{equation}
    where $\tilde{g}(\uv) = \tilde{R}^*(\Xv_u^\top \uv)$ and $\tilde{R}$
    denotes the quadratic surrogate of the regularizer.

    In particular, $\leavei{\surrog{y}}_i=K_{ii}\leavei{\surrog{y}}_{u,i}$ is
    the exact leave-$i$-out predicted value for the surrogate problem described
    in Theorem \ref{thm:primal-dual-equivalence}.
\end{theorem}

We refer the reader to the Appendix \ref{append:sec:primal-dual-equiv} for the
proofs. These two theorems imply that for twice differentiable losses and
regularizers, the frameworks we laid out in Sections \ref{sec:approximatedual}
and \ref{sec:primal-smoothing} lead to exactly the same ALO formulas.
This equivalence theorem reflects the deep connections between
the primal and dual optimization problem. The central
property used by the proof is captured in the following lemma:
\begin{lemma}
    Let $f$ be a proper closed convex function, such that both $f$ and $f^*$ are twice differentiable.
    Then, we have for any $\xv$ in the domain of $f$:
    \begin{equation*}
        \nabla^2 f^*(\nabla f(\xv)) = [\nabla^2 f(\xv)]^{-1}.
    \end{equation*}
\end{lemma}

By combining this lemma with the primal dual correspondence \eqref{eq:primal-dual-correspondence},
we obtain a relation between the curvature of the primal and dual problems at the optimal value,
ensuring that the approximation is consistent with the dual structure.

\section{Applications}\label{sec:applications}
\subsection{Generalized LASSO}

The generalized LASSO \cite{tibshirani2011genlasso} is a generalization of the LASSO
problem which captures many applications such as the fused LASSO \cite{tibshirani2005fused},
$\ell_1$ trend filtering \cite{kim2009trend} and wavelet smoothing in a
unified framework. The generalized LASSO problem corresponds to the following
penalized regression problem:
\begin{equation}\label{eq:genlasso:statement}
    \min_{\betav} \frac{1}{2}\sum_{j = 1}^n (y_j - \xv_j^\top \betav)^2 + \lambda \norm{\Dv \betav}_1.
\end{equation}
where the regularizer  is parameterized by a fixed matrix $\Dv \in \RR^{m \times p}$
which captures the desired structure in the data. We note that the
regularizer is a semi-norm, and hence we can formulate the dual problem as a
projection. In fact, a dual formulation of \eqref{eq:genlasso:statement} can
be obtained as (see Appendix \ref{append:sec:generalized-lasso-dual}):
\begin{gather*}
    \min_{\dualv, \uv} \frac{1}{2} \norm{\dualv - \yv}_2^2, \\
    \text{subject to: } \norm{\uv}_\infty \leq \lambda \text{ and }  \Xv^\top \dualv = \Dv^\top \uv.
\end{gather*}

The dual optimal solution satisfies $\estim{\dualv} = \Pi_{\Delta_X}(\yv)$,
where $\Delta_X$ is the polytope given by:
\begin{equation*}
    \Delta_X = \{ \dualv \in \RR^n : \exists \uv, \norm{u}_\infty \leq \lambda
    \text{ and } \Xv^\top \dualv = \Dv^\top u \}.
\end{equation*}

The projection onto the polytope $C = \{ \Dv^\top \uv: \norm{\uv}_\infty \leq
\lambda \}$ is given in \cite{tibshirani2011genlasso} as locally being the projection
onto the affine space orthogonal to the nullspace of $\Dv_{\cdot, -E}$, where
$E = \{i : \abs{\estim{u}_i} = \lambda \}$ and $-E = \{ 1, \dotsc, p \} \setminus E$.
Since $\Delta_X = [\Xv^\top]^{-1} C$ is the inverse image of $C$ under the
linear map given by $\Xv^\top$, the projection onto $\Delta_X$ is given locally
by the projection onto the affine space normal to the space spanned by the
columns of $[\Xv^\top]^+ \mathrm{null} \, \Dv_{\cdot, -E}$, provided $\Xv$ has
full column rank. Here, $[\Xv^\top]^+$ denotes the Moore-Penrose pseudoinverse
of $\Xv^\top$. Finally, to obtain a spanning set of this space, we may
consider $\Av = \Xv \Bv$, where $\Bv$ is a set of vectors spanning the nullspace
of $\Dv_{\cdot,-E}$. This allows us to compute $\Hv = \Av \Av^+$, the
projection onto the normal space required to compute the ALO.

\subsection{Nuclear Norm}
Consider the following matrix sensing problem
\begin{equation}\label{eq:nuclear-norm-main}
    \estim{\Bv}:
    =
    \argmin_{\Bv} \frac{1}{2}\sum_{j=1}^n (y_j - \langle \Xv_j, \Bv\rangle)^2
    + \lambda \|\Bv\|_*,
\end{equation}
with $\Bv, \Xv_j \in \mathbb{R}^{p_1 \times p_2}$. $\langle \Xv, \Bv\rangle =
\mathrm{trace}(\Xv^\top\Bv)$ denotes the inner product. We use $\|\cdot\|_*$
for nuclear norm, which is defined as the sum of the singular values of a
matrix. The nuclear norm is a unitarily invariant function of the matrix
\cite{lewis1995convex}. Such functions are only indirectly related to the
components of the matrix, making their analysis
difficult even when they are smooth, and exacerbating the difficulties when they
are non-smooth such as in the case of the nuclear norm. In particular, the smoothing
framework described in Section \ref{ssec:nonsmooth-regularizer} cannot be
applied directly.

We are nonetheless able to leverage the specific structure of such functions
to obtain the following theorem. Let $R$ be a smooth unitarily invariant matrix function, with:
    \begin{equation*}
        R(\Bv) = \sum_{j = 1}^{\min(p_1, p_2)} r(\sigma_j),
    \end{equation*}
    where $\sigma_j$ denotes the $j$\tsup{th} singular value of $\Bv$. Consider the following matrix penalized regression problem:
    \begin{equation*}\label{eq:matrix-smooth-main}
        \estim{\Bv} = \argmin_{\Bv} \sum_{j=1}^n \ell(\langle \Xv_j,
        \Bv\rangle; y_j) + \lambda R(\Bv).
    \end{equation*}

    Without loss of generality, below we assume $p_1 \geq p_2$. Let
    $\estim{\Bv}=\estim{\Uv}\diag[\estim{\sigmav}]\estim{\Vv}^\top$ be the
    singular value decomposition (SVD) of the full data estimator $\estim{\Bv}$,
    where $\estim{\Uv} \in \mathbb{R}^{p_1 \times p_1}$, $\estim{\Vv} \in
    \mathbb{R}^{p_2 \times p_2}$. Let $\estim{\uv}_k$, $\estim{\vv}_l$ be the
    $k$\tsup{th} and $l$\tsup{th} column of $\estim{\Uv}$ and $\estim{\Vv}$
    respectively. $\diag[\estim{\sigmav}]$ in this section is a $p_1 \times
    p_2$ matrix with $\estim{\sigma}_j$ on the diagonal of its upper square
    sub-matrix and 0 elsewhere. In addition, we assume all the $\estim{\sigma}_j$'s are
    nonzero. We then have the following ALO formula:
    \begin{equation*}\label{eq:matrix-smooth-alo}
        \langle \Xv_i, \leavei{\surrog{\Bv}} \rangle
        =
        \langle \Xv_i, \estim{\Bv} \rangle
        +
        \frac{H_{ii}\dot{\ell}(\langle \Xv_i, \estim{\Bv} \rangle;y_i)}{1 -
        H_{ii}\ddot{\ell}(\langle \Xv_i, \estim{\Bv} \rangle;y_i)},
    \end{equation*}
    where
    \begin{equation*}
        \Hv = \cb{X}[\cb{X}^\top\diag[\{\ddot{\ell}(\langle \Xv_j, \Bv\rangle;
        y_j)\}_j] \cb{X} + \lambda \cb{G}]^{-1}\cb{X}^\top.
    \end{equation*}
    
    Here $\cb{X}$ is a $n \times p_1 p_2$ matrix and $\cb{G}$ is a symmetric square $p_1 p_2
    \times p_1 p_2$ matrix given by:
    \begin{equation}\label{eq:matrix-smooth-alo-definition}
    \begin{aligned}
        \cb{X}_{j,kl} &= \estim{\uv}_k^\top \Xv_j \estim{\vv}_l, \\
        \cb{G}_{kl, st} &= \begin{cases}
                \ddot{r}(\estim{\sigma}_t) & s = t = k = l, \\
                \frac{\estim{\sigma}_s \dot{r}(\estim{\sigma}_s) - \estim{\sigma}_t\dot{r}(\estim{\sigma}_t)}
                {\estim{\sigma}_s^2 - \estim{\sigma}_t^2} & s \neq t, s \leq p_2, (k, l) = (s, t), \\
                - \frac{\estim{\sigma}_s \dot{r}(\estim{\sigma}_t) - \estim{\sigma}_t\dot{r}(\estim{\sigma}_s)}
                {\estim{\sigma}_s^2 - \estim{\sigma}_t^2} & s \neq t, s \leq p_2, (k, l) = (t, s), \\
                \frac{\dot{r}(\estim{\sigma}_t)}{\estim{\sigma}_t} & s \neq t, s > p_2, (k,l)=(s,t), \\
                0 & \text{otherwise.}
        \end{cases}
    \end{aligned}
\end{equation}

Note that the rows of $\cb{X}$ and the index of $\cb{G}$ are vectorized
in a consistent way. The proof can be found in Appendix \ref{append:ssec:smooth-unitary}.
A nice property of this result is that the effect on singular values
decouples from the original matrix, enabling us to apply our smoothing
strategy in Section \ref{ssec:nonsmooth-regularizer} to function $r(\sigma)$
when it is nonsmooth. This leads to the following theorem for nuclear norm. For
more details on the derivation, please refer to Appendix
\ref{append:ssec:nuclear-proof}.
\begin{theorem}\label{thm:matrix-nonsmooth-alo}
    Consider the nuclear-norm penalized matrix regression problem
    \eqref{eq:nuclear-norm-main}, and let
    $\estim{\Bv}=\estim{\Uv}\diag[\estim{\sigmav}]\estim{\Vv}^\top$ be the
    SVD of the full data estimator $\estim{\Bv}$,
    with $\estim{\Uv} \in \mathbb{R}^{p_1 \times p_1}$, $\estim{\Vv} \in
    \mathbb{R}^{p_2 \times p_2}$. Let $m=\rank(\estim{\Bv})$ be the number of
    nonzero $\estim{\sigma}_j$'s for $\estim{\Bv}$. Let
    $\leavei{\surrog{\Bv}}_h$ denote the approximate of $\estimi{\Bv}$
    obtained from the smoothed problem. Then, as $h \rightarrow 0$
    \begin{equation*}\label{eq:matrix-nonsmooth-alo}
        \langle \Xv_i, \leavei{\surrog{\Bv}}_h \rangle
        \rightarrow
        \langle \Xv_i, \estim{\Bv} \rangle
        +
        \frac{H_{ii}}{1 - H_{ii}}(\langle \Xv_i, \estim{\Bv} \rangle - y_i),
    \end{equation*}
    where
    \begin{equation*}
        \Hv = \cb{X}_{\cdot,E}[\cb{X}_{\cdot,E}^\top\cb{X}_{\cdot,E} +
        \lambda \cb{G}]^{-1}\cb{X}_{\cdot,E}^\top,
    \end{equation*}
    with $\cb{X}$ as defined in \eqref{eq:matrix-smooth-alo-definition} and
    with $\cb{G} \in \mathbb{R}^{(mp_1 + mp_2 - m^2) \times (mp_1 + mp_2 - m^2)}$ given by:
    \begin{equation}\label{eq:nuclear-hessian-G}
        \cb{G}_{kl, st}
        =
        \begin{cases}
            0 & s = t = k = l \leq m, \\
            \frac{1}{\estim{\sigma}_s + \estim{\sigma}_t} & 1 \leq s \neq t \leq m,(k,l)=(s,t), \\
            \frac{1}{\estim{\sigma}_s} & 1 \leq s \leq m < t \leq p_2, (k, l) = (s, t), \\
            \frac{1}{\estim{\sigma}_t} & 1 \leq t \leq m < s \leq p_1, (k, l) = (s, t), \\
            -\frac{1}{\estim{\sigma}_s + \estim{\sigma}_t} & 1 \leq s \neq t \leq m, (k,l)=(t,s), \\
            -\frac{g_r[\estim{\sigma}_t]}{\estim{\sigma}_s} & 1 \leq s \leq m < t \leq p_2,
            (k, l) = (t, s), \\
            -\frac{g_r[\estim{\sigma}_s]}{\estim{\sigma}_t} & 1 \leq t \leq m < s \leq p_2,
            (k, l) = (t, s), \\
            0 & \text{otherwise.}
        \end{cases}
    \end{equation}
    where for $t > m$, $\estim{\sigma}_t=0$ and
    $g_r[\estim{\sigma}_t]$ is the corresponding subgradient at this singular
    value, which can be obtained through the SVD of
    $\frac{1}{\lambda}\sum_{j=1}^n (y_j - \langle \Xv_j,
    \estim{\Bv}\rangle)\Xv_j$.
    The set $E$ is then defined as:
    \begin{equation*}
        E = \{(k, l): k \leq m \text{ or } l \leq m\}.
    \end{equation*}
    Note that the indices of $\cb{G}$ and the index set $E$ are consistent.
\end{theorem}

\subsection{Linear SVM}
The linear SVM optimization can be written as
\begin{equation*}\label{eq:svm-main}
    \argmin_{\betav}\sum_{j = 1}^n (1 - y_j\xv_j^\top\betav)_+ +
    \frac{\lambda}{2}\|\betav\|_2^2,
\end{equation*}
with $y_j\in\{-1, 1\}$ and $(\cdot)_+=\max\{\cdot, 0\}$. Note that this is a special instance of the problem
we studied in Section \ref{ssec:nonsmooth-loss}. Here,
$\ell(u; y_j) = (1- y_j u)_+$ has only one zero order singularity at $y_j$.
Using Theorem \ref{thm:nonsmooth-loss-approx} and simplifying the expressions,
we obtain the following ALO formula for SVM:
\begin{equation*}
    \xv_i^\top\leavei{\surrog{\betav}} = \xv_i^\top\estim{\betav} +  a_i g_{\ell,i},
\end{equation*}
where 
\begin{equation*}\label{eq:svm-H}
    a_i =
    \left\{
        \begin{array}{ll}
            \frac{1}{\lambda}\xv_i^\top (\Iv_p -
            \Xv_{V,\cdot}^\top(\Xv_{V,\cdot}\Xv_{V,\cdot}^\top)^{-1}\Xv_{V,\cdot})\xv_i
            & i \in S, \\
            \big(\lambda[(\Xv_{V,\cdot}\Xv_{V,\cdot}^\top)^{-1}]_{ii}\big)^{-1} & i \in V, \\
        \end{array}
    \right.
\end{equation*}
and for $i \in S$, $g_{\ell,i} = -y_i$ if $y_i\xv_i^\top\estim{\betav} < 1$,
$g_{\ell,i} = 0$ if $y_i\xv_i^\top\estim{\betav} > 1$, and for $i \in V$
\begin{equation*}\label{eq:svm-subg}
    \gv_{\ell,V} =
    (\Xv_{V,\cdot}\Xv_{V,\cdot}^\top)^{-1}\Xv_{V,\cdot}[\lambda\estim{\betav} +
    \sum_{j: y_j\xv_j^\top\estim{\betav} < 1}y_j \xv_j].
\end{equation*}

Recall that $V=\{j: \xv_j^\top\estim{\betav} = y_j\}$ and $S=[1,\ldots,
n]\backslash V$.

\section{Numerical Experiments}
\label{sec:experiments}

We illustrate the performance of ALO through three experiments. The first two
compare the ALO risk estimate with that of LOOCV. The third experiment compares
the computational complexity of ALO with that of LOOCV. We have also evaluated
the performance of ALO on real-world datasets. Due to lack of space, these
results are presented in Appendix \ref{ssec:real-world-data}. For the first
experiment (Figure \ref{fg:experiment1}), we run ALO and LOOCV for the three
models studied in Section \ref{sec:applications} (using fused LASSO
\cite{tibshirani2005fused} as a special case of generalized LASSO) and compare
their risk estimates under the settings $ n> p$ and $n < p$ respectively. The
full details of the experiments are provided in Appendix
\ref{append:sec:experiment}.

For the second experiment (Figure \ref{fg:experiment2}), we consider the risk
estimates for LASSO from ALO and LOOCV under settings with model
mis-specification, heavy-tail noise and correlated design. For all three cases,
ALO approximates LOOCV well.

\begin{figure}[!htbp]
    \begin{center}
        \setlength\tabcolsep{2pt}
        \renewcommand{\arraystretch}{0.3}
        \subfloat[][]{
        \begin{tabular}{r|rrr}
            & \multicolumn{1}{c}{\small svm} & \multicolumn{1}{c}{\small fused
               lasso} & \multicolumn{1}{c}{\small nuclear norm} \\
            \hline
            \rotatebox{90}{\small \hspace{1.6cm} $n > p$} &
            \includegraphics[scale=0.4]{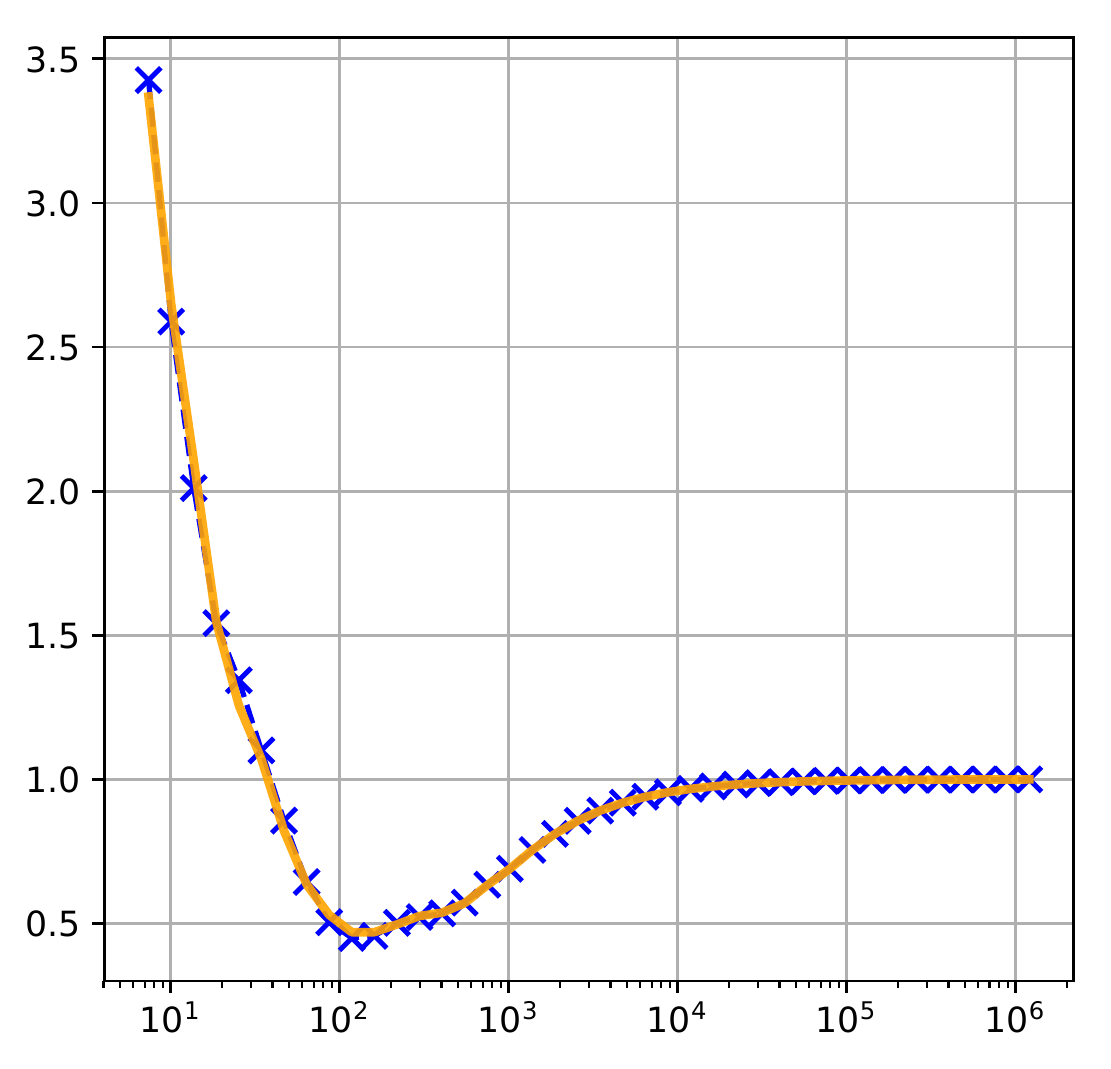} & 
            \includegraphics[scale=0.4]{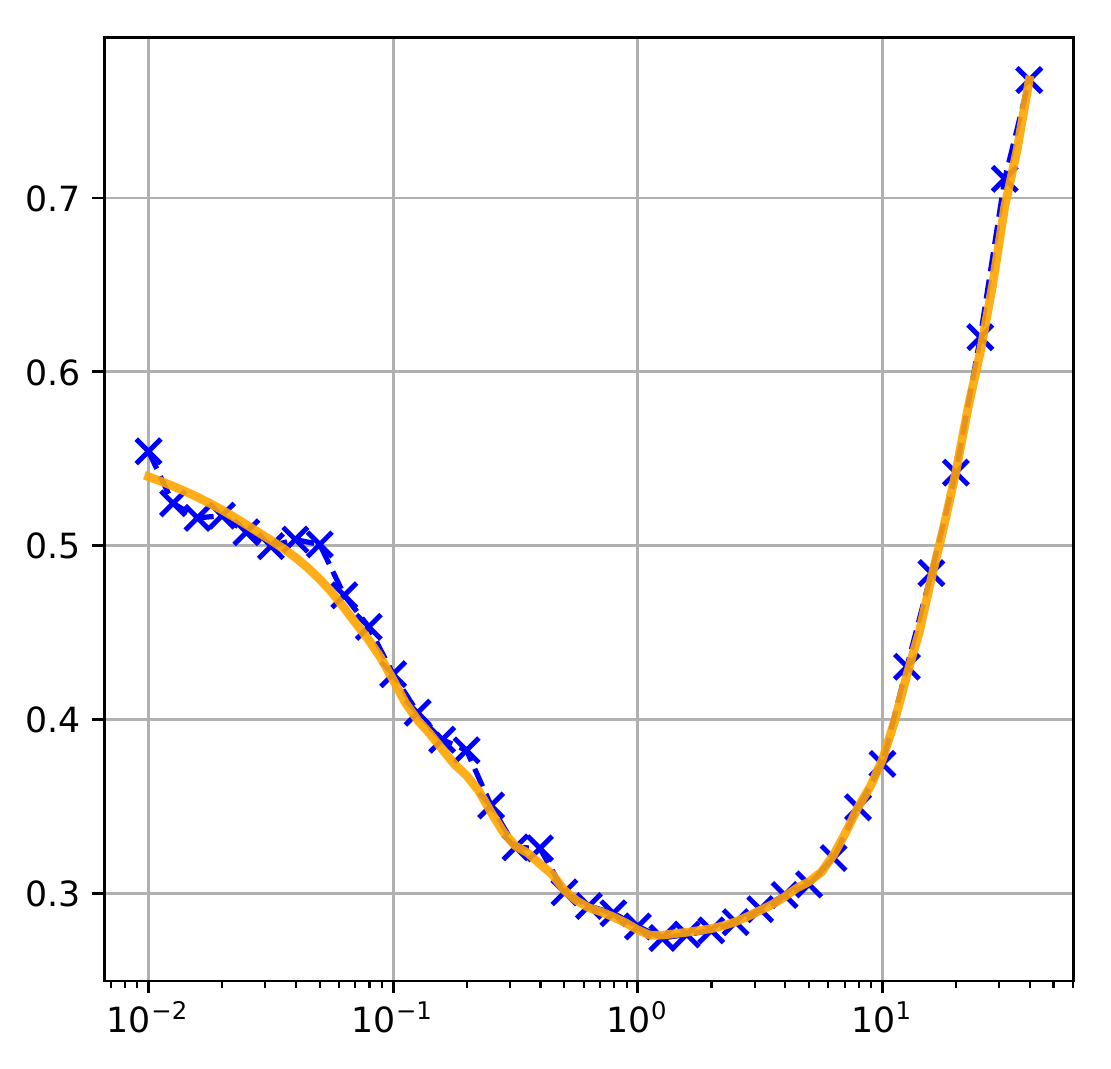} &
            \includegraphics[scale=0.4]{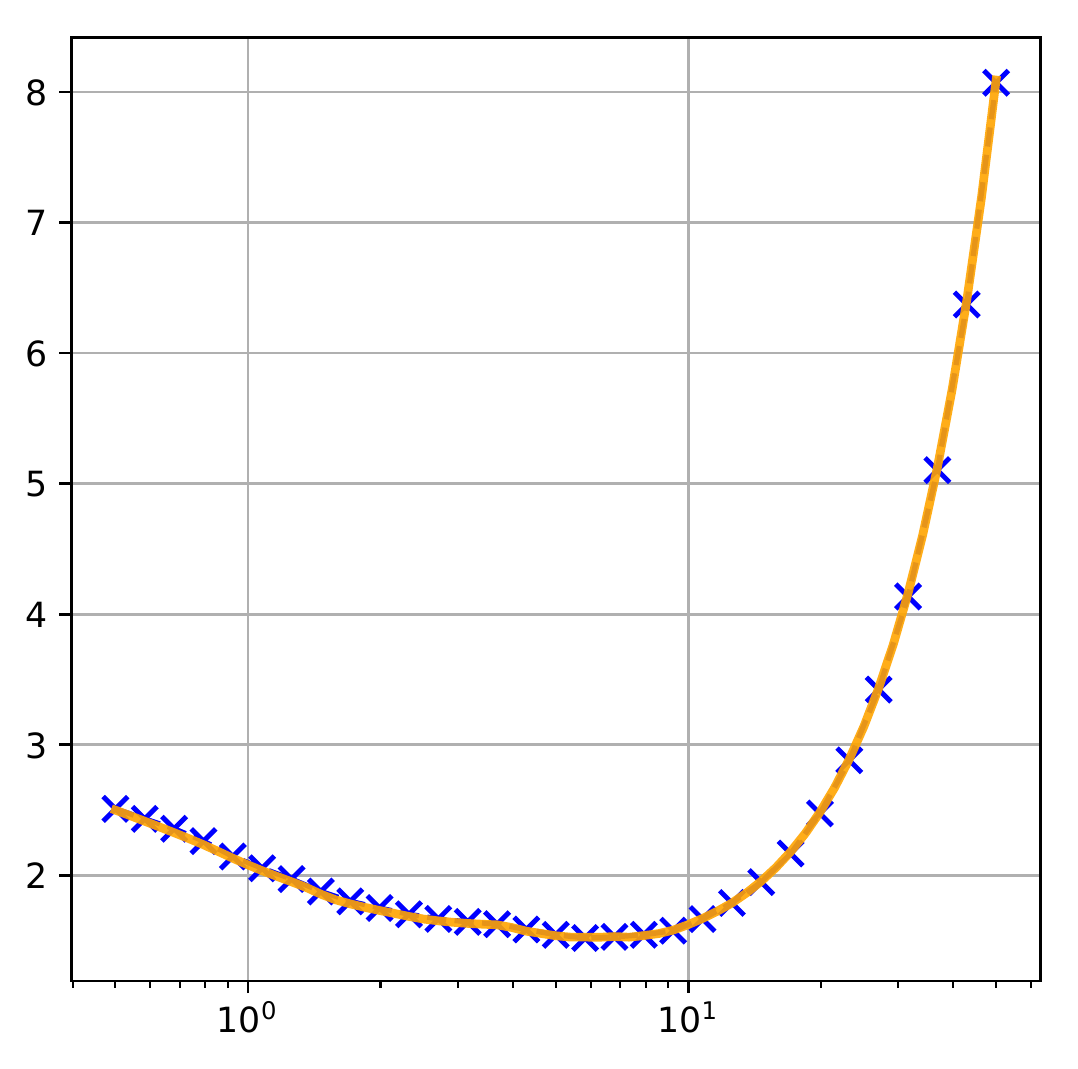} \\
            & \multicolumn{1}{c}{{\fontsize{7}{6} $n=300, p=80$}}
            & \multicolumn{1}{c}{{\fontsize{7}{6} $n=200, p=100$}}
            & \multicolumn{1}{c}{{\fontsize{7}{6} $n=600, p_1=p_2=20$}} \\
            \rotatebox{90}{\small \hspace{1.6cm} $n < p$} &
            \includegraphics[scale=0.4]{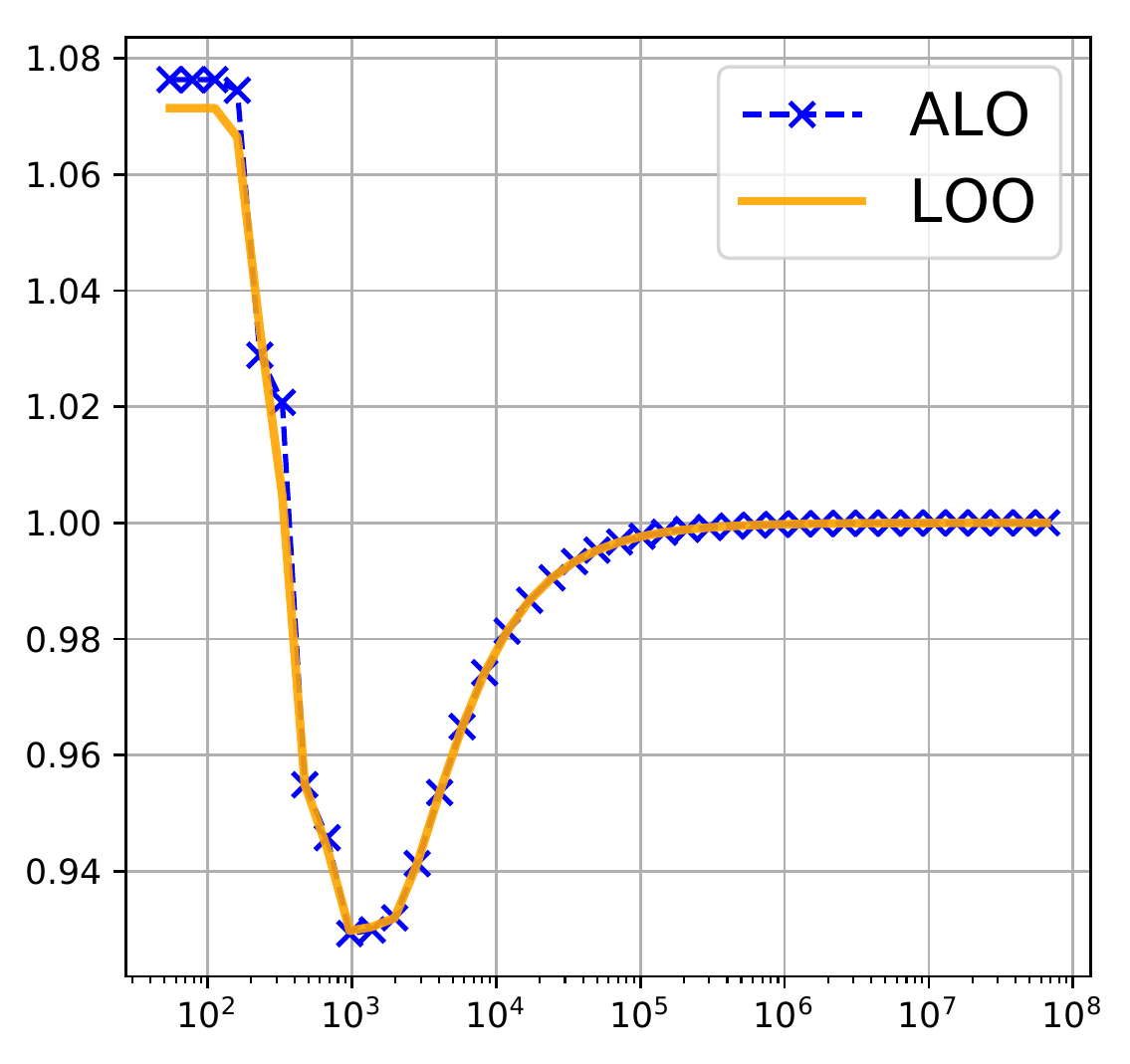} &
            \includegraphics[scale=0.4]{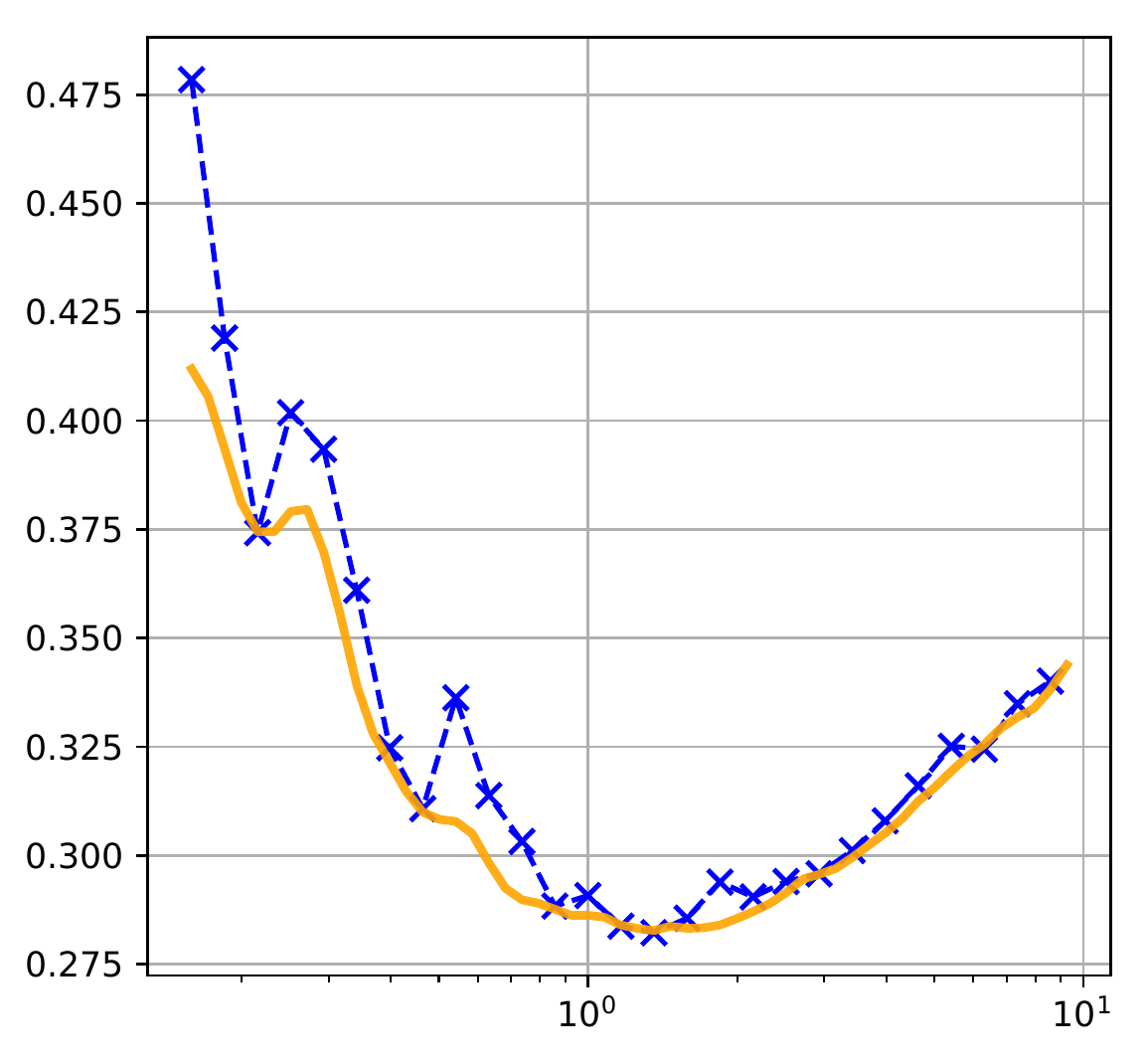} &
            \includegraphics[scale=0.4]{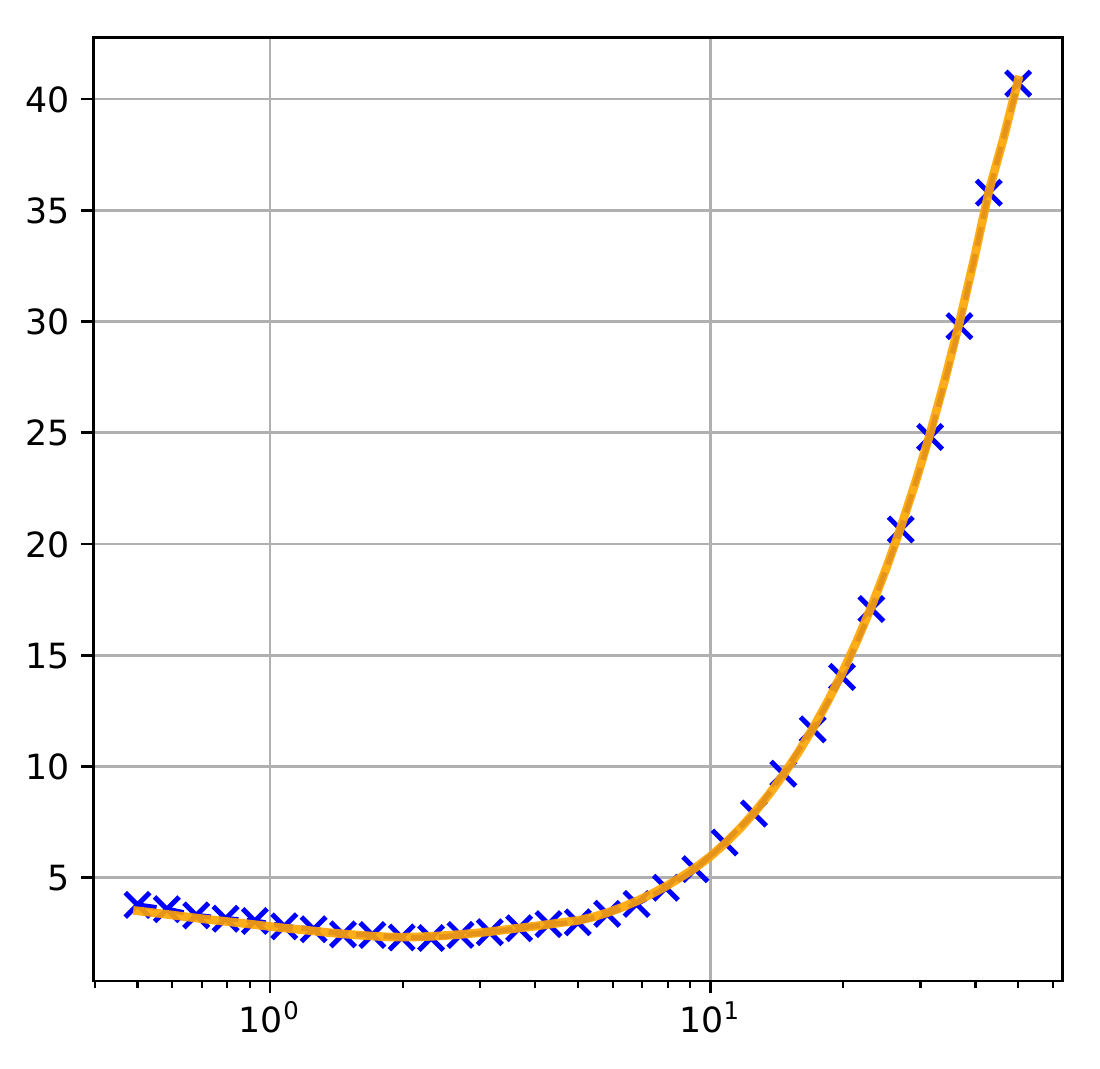} \\
            & \multicolumn{1}{c}{{\fontsize{7}{6} $n=300, p=600$}}
            & \multicolumn{1}{c}{{\fontsize{7}{6} $n=200, p=400$}}
            & \multicolumn{1}{c}{{\fontsize{7}{6} $n=200, p_1=p_2=20$}} \\
        \end{tabular} \label{fg:experiment1}}
        
        \subfloat[][]{
        \begin{tabular}{r|rrr}
            & \multicolumn{1}{c}{\small misspecification}
            & \multicolumn{1}{c}{\small heavy-tailed noise}
            & \multicolumn{1}{c}{\small correlated design} \\
            \hline
            \rotatebox{90}{\small \hspace{1.2cm} lasso risk} &
            \includegraphics[scale=0.4]{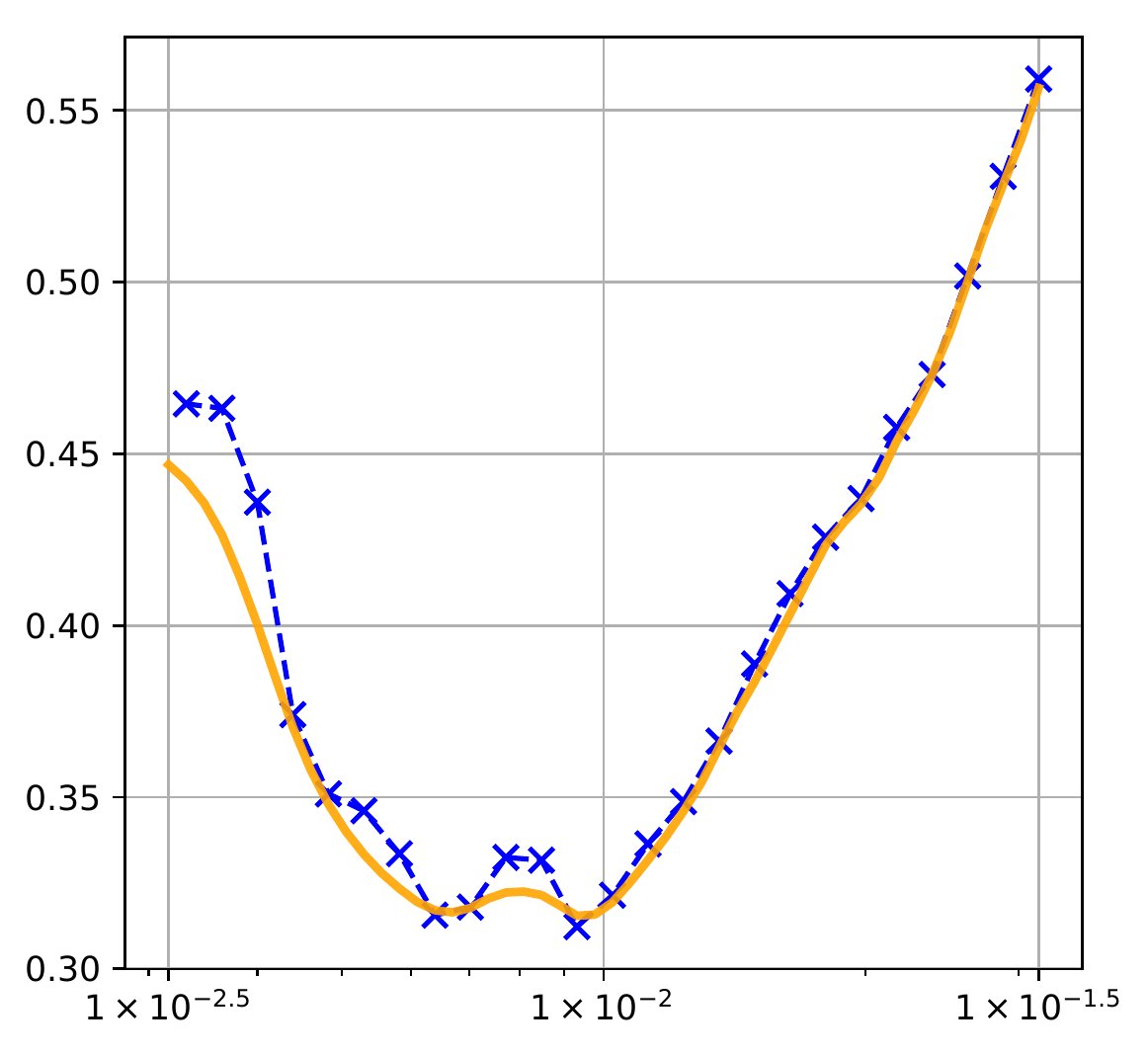} & 
            \includegraphics[scale=0.4]{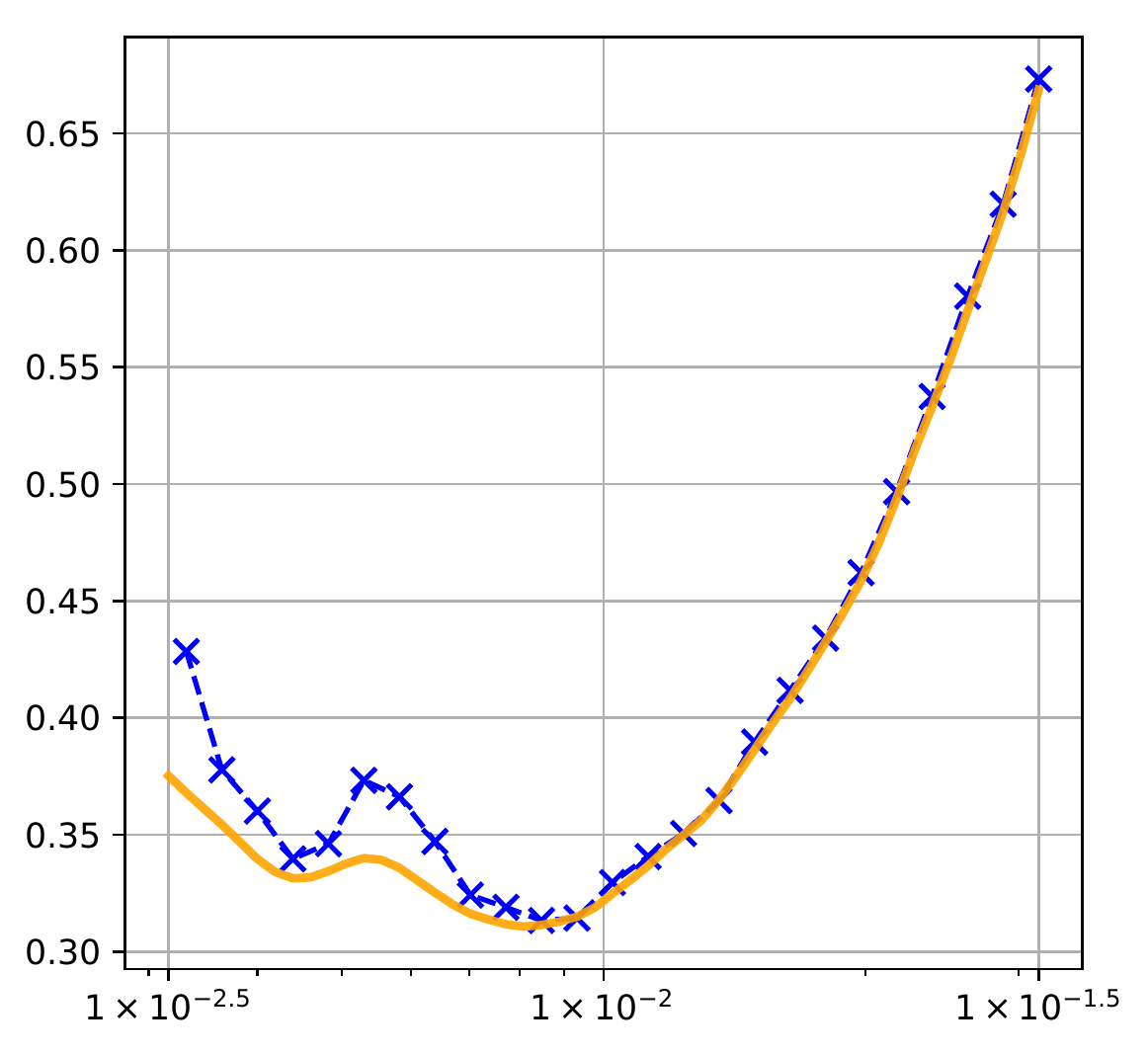} & 
            \includegraphics[scale=0.4]{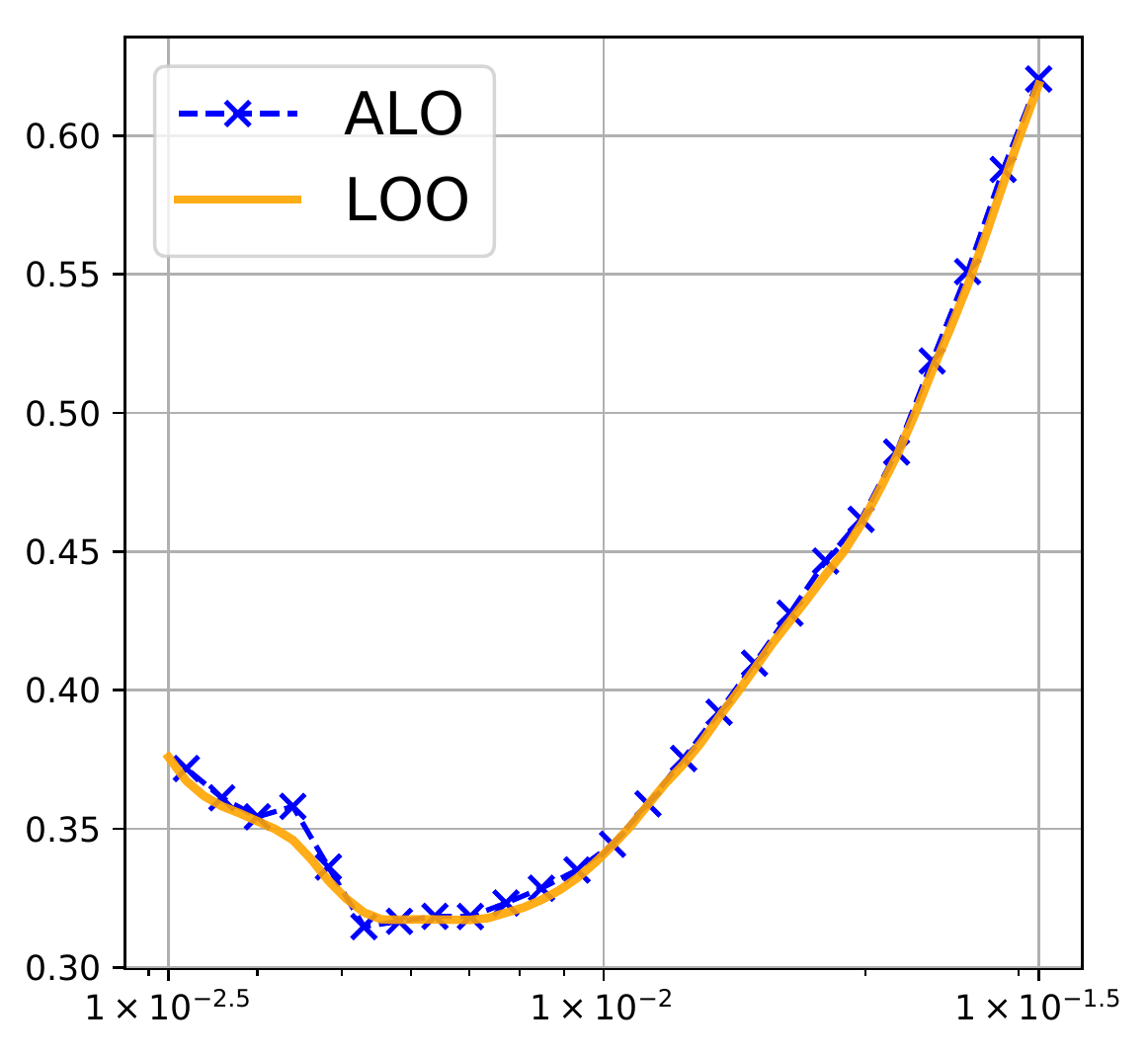} \\
    \end{tabular} \label{fg:experiment2}}
        \caption{Risk estimates from ALO versus LOOCV. The $x$-axis is the
        tuning parameter value on $\log$-scale, the $y$-axis is the risk
        estimate. In part \protect\subref{fg:experiment1}, the comparison is based on SVM,
        fused LASSO and nuclear norm. Different settings for the number of observations
        $n$ and the number of features $p$ are considered. For nuclear norm, $p_1, p_2$
        are dimensions of a matrix. In part \protect\subref{fg:experiment2}, we consider the
        risk estimates of LASSO under model mis-specification, heavy-tailed noise and
        correlated design scenarios. We use $n=300$, $p=600$ and $k=30$ for all three
        where $k$ is the number of nonzeros in the true
        $\betav$.}\label{fig:risk-alo-loo1}
    \end{center}
\end{figure}

\begin{table}[!htbp]
    \begin{center}
        \caption{Timing (in \textit{sec}) of one single fit, ALO and LOOCV. In
        the upper and lower tables, we fix $n=800$ and $p=800$ respectively.}
        \label{tab:timing}
        \begin{tabular}{l|llll}
            $p$ & 200 & 400 & 800 & 1600 \\
            \hline
            single fit & $0.035 \pm 0.001$ & $0.13 \pm 0.003$ & $0.56
            \pm 0.02$ & $0.60 \pm 0.01$ \\
            ALO & $0.060 \pm 0.001$ & $0.21 \pm 0.003$ & $0.77 \pm
            0.02$ & $0.89 \pm 0.01$ \\
            LOOCV & $27.52 \pm 0.03$ & $107.4 \pm
            0.5$ & $437.9 \pm 2.9$ & $479 \pm 2$ \\
            \hline
        \end{tabular}
        \begin{tabular}{l|llll}
            $n$ & 200 & 400 & 800 & 1600 \\
            \hline
            single fit & $0.055 \pm 0.002$ & $0.19 \pm 0.006$ & $0.56
            \pm 0.02$ & $0.76 \pm 0.02$ \\
            ALO & $0.065 \pm 0.001$ & $0.24 \pm 0.001$ & $0.77 \pm
            0.02$ & $1.20 \pm 0.01$ \\
            LOOCV & $11.44 \pm 0.049$ & $74.7 \pm
            0.5$ & $437.9 \pm 2.9$ & $1249 \pm 3$
        \end{tabular}
    \end{center}
\end{table}

In general, we observe that the estimates given by ALO are close
to LOOCV, although the performance may 
deteriorate for very small values of $\lambda$, as is clear in the fused-LASSO
($n<p$) example. These values of $\lambda$ correspond to ``dense'' solutions,
and are far from the optimal choice. Hence, such inaccuracies do not harm the
parameter tuning algorithm.


Our last experiment compares the computational complexity of ALO with that of
LOOCV. In Table \ref{tab:timing}, we provide the timing of LASSO for different
values of $n$ and $p$. The time required by ALO, which involves a single fit
and a matrix inversion (in the
construction of $\Hv$ matrix), is in all experiments no more than twice that of a single fit.
We refer the reader to Appendix \ref{append:sec:experiment} for the details of this experiment.


\section{Discussion}
ALO offers a highly efficient approach for parameter tuning and risk estimation
for a large class of statistical machine learning models. We focus on nonsmooth
models and propose two general frameworks for calculating ALO. One is from the
primal perspective, the other from the dual.

By approximating LOOCV, ALO inherits desirable properties of LOOCV in
high-dimensional settings where $n$ and $p$ are comparable. In particular, ALO
can overcome the bias issues that $k$-fold cross validation displays
in these settings.

\section*{Acknowledgements}
We acknowledge computing resources from Columbia University's Shared Research
Computing Facility project, which is supported by NIH Research Facility
Improvement Grant 1G20RR030893-01, and associated funds from the New York State
Empire State Development, Division of Science Technology and Innovation
(NYSTAR) Contract C090171, both awarded April 15, 2010.

\bibliography{reference}
\bibliographystyle{plain}

\clearpage

\appendix

\section{Proof of Equation \ref{eq:methods:dual}}
\label{sec:dual-derivation}
In this Section, we prove the primal-dual correspondence in
\eqref{eq:methods:primal} and \eqref{eq:methods:dual}.
Recall the form of the primal problem:
\begin{equation} \label{eq:primal-dual-proof:primal}
    \min_{\betav} \sum_{j = 1}^n \ell(\xv_j^\top \betav; y_j) + R(\betav).
\end{equation}

With a change of variable, we may transform \eqref{eq:primal-dual-proof:primal}
into the following form:
\begin{equation*}
    \min_{\betav, \muv} \sum_{j = 1}^n \ell(-\mu_j; y_j) + R(\betav),
    \quad
    \text{subject to: }
    \muv = -\Xv\betav.
\end{equation*}

We may further absorb the constraint into the objective function by adding a
Lagrangian multiplier $\thetav \in \mathbb{R}^n$:
\begin{equation} \label{eq:primal-dual-proof:primal1}
    \max_{\thetav}\min_{\betav, \muv}
    \sum_{j = 1}^n \ell(-\mu_j; y_j) + R(\betav) - \thetav^\top(\Xv\betav + \muv).
\end{equation}

Note that in \eqref{eq:primal-dual-proof:primal1}, $\betav$ and $\muv$ decoupled
from each other and we can optimize over them respectively. Specifically, we
have that
\begin{align}
    &\min_{\betav} R(\betav) - \thetav^\top\Xv\betav
    =
    - \max_{\betav} \big\{ \langle \betav, \Xv^\top\thetav \rangle - R(\betav)
    \big\}
    =
    - R^*(\Xv^\top\thetav), \label{eq:primal-dual-proof:beta} \\
    &\min_{\mu_j} \ell(-\mu_j; y_j) - \theta_j \mu_j
    =
    - \max \{ \mu_j\theta_j - \ell(-\mu_j; y_j) \}
    =
    - \ell^*(-\theta_j; y_j). \label{eq:primal-dual-proof:theta}
\end{align}

We plug \eqref{eq:primal-dual-proof:beta} and
\eqref{eq:primal-dual-proof:theta} in \eqref{eq:primal-dual-proof:primal1} and
obtain that
\begin{equation} \label{eq:primal-dual-proof:dual1}
    \max_{\thetav}
    \sum_{j = 1}^n - \ell^*(-\theta_j; y_j) - R^*(\Xv^\top \thetav).
\end{equation}

\section{Primal Dual Equivalence (Proofs of Theorems
\ref{thm:primal-dual-equivalence} and \ref{thm:primal-dual-equivalence-2})}
\label{append:sec:primal-dual-equiv}

In this section we prove the equivalence between the two stated methods in the case
where the loss and regularizer are twice differentiable.
Let $\ell$, $\ell^*$, $R$ and $R^*$ be twice differentiable. We
construct quadratic surrogates by Taylor extensions. The following lemma
plays a key role in our analysis:\\

\begin{lemma}\label{lemma:hessianprimaldual}
    Let $f$ be a proper closed convex function, such that both $f$ and $f^*$ are twice differentiable.
    Then, we have for any $\xv$ in the domain of $f$ and any $\uv$ in the domain of $f^*$:
    \begin{align*}
        \nabla^2 f^*(\nabla f(\xv)) =& [\nabla^2 f(\xv)]^{-1}, \nonumber \\
        \nabla^2 f(\nabla f^*(\uv)) =& [\nabla^2 f^*(\uv)]^{-1}. 
    \end{align*}
\end{lemma}
\begin{proof}
    This lemma is a known result in convex optimization. However, since the
    proof is short and for the sake of completeness we include the proof here.
    For $f$ a proper closed convex function, we have by Theorem 23.5 of
    \cite{rockafellar1970convex} that for all $\xv, \xv^*$:
    \begin{equation*}
        \xv^* \in \partial f(\xv) \Rightarrow \xv \in \partial f^*(\xv^*).
    \end{equation*}
    
    In particular, if $f$ and $f^*$ are differentiable, we obtain:
    \begin{equation*}
        \xv = \nabla f^*(\nabla f(\xv)).
    \end{equation*}
    
    Taking derivative in $\xv$ once more, we obtain that:
    \begin{equation*}
        \Iv = [\nabla^2 f^*( \nabla f(\xv))] [\nabla^2 f(\xv)],
    \end{equation*}
    
    which immediately gives:
    \begin{equation*}
        \nabla^2 f^*(\nabla f(\xv)) = [\nabla^2 f(\xv)]^{-1}.
    \end{equation*}

    The proof of the second part is immediate by applying the existing
    result to $f^*$.
\end{proof}

\begin{proof}[Proof of Theorem \ref{thm:primal-dual-equivalence}]
    We have the following expressions for $\surrog{\ell}$ and $\surrog{R}$:
    \begin{align*}
        \surrog{\ell}(z_j; y_j) =& \frac{1}{2} \ddot{\ell}(\xv_j^\top \estim{\betav}; y_j) (z_j - \xv_j^\top \estim{\betav})^2
        + \dot{\ell}(\xv_j^\top \estim{\betav}; y_j) (z_j - \xv_j^\top \estim{\betav}) + c, \\
        \surrog{R}(\betav) =& \frac{1}{2} (\betav - \estim{\betav})^\top [\nabla^2 R(\estim{\betav})] (\betav - \estim{\betav})
        + [\nabla R(\estim{\betav})]^\top (\betav - \estim{\betav}) + d,
    \end{align*}
    where $c, d \in \RR$ are constants that do not affect the location of the
    optimizer. We now compute the convex conjugate of $\surrog{\ell}$ and
    $\surrog{R}$, and we obtain that:
    \begin{align}
        \surrog{\ell}^*(w_j; y_j) =& \frac{1}{2} \frac{1}{\ddot{\ell}(\xv_j^\top
        \estim{\betav}; y_j)} (w_j - \dot{\ell}(\xv_j^\top \estim{\betav}; y_j))^2
        + (\xv_j^\top \estim{\betav}) (w_j - \dot{\ell}(\xv_j^\top
        \estim{\betav}; y_j)) + c', \label{eq:equiv:dual-quad-primal-loss}\\
        \surrog{R}^*(\muv) =& \frac{1}{2} (\muv - \nabla R(\estim{\betav}))^\top
        [\nabla^2 R(\estim{\betav})]^{-1} (\muv - \nabla R(\estim{\betav}))
        + \estim{\betav}^\top (\muv - \nabla R(\estim{\betav})) +
        d',\label{eq:equiv:dual-quad-primal-reg}
    \end{align}
    where again $c', d' \in \RR$ are constants.

    Now, we wish to relate \eqref{eq:equiv:dual-quad-primal-loss} and
    \eqref{eq:equiv:dual-quad-primal-reg} to $\surrog{\ell}_D^*$ and
    $\surrog{R}_D^*$. By substituting the primal-dual correspondence described in
    \eqref{eq:primal-dual-correspondence} of the main text for components of
    \eqref{eq:equiv:dual-quad-primal-loss} and
    \eqref{eq:equiv:dual-quad-primal-reg}, we obtain that:
    \begin{align}
        \surrog{\ell}^*(w_j; y_j)
        &= \frac{1}{2} \frac{1}{\ddot{\ell}(\dot{\ell}^*(-\estim{\duals}_j; y_j); y_j)} (w_j + \estim{\duals}_j)^2
            + \dot{\ell}^*(-\estim{\duals}_j; y_j) (w_j + \estim{\duals}_j) +
            c', \label{eq:equivalence:primal-substituted-l} \\
        \surrog{R}^*(\muv)
        &= \frac{1}{2}(\muv - \Xv^\top \estim{\dualv})^\top[\nabla^2 R(\nabla
    R^*(\Xv^\top \estim{\dualv}))]^{-1} (\muv - \Xv^\top \estim{\dualv})
    \nonumber \\
        &\quad + [\nabla R^* (\Xv^\top \estim{\dualv})]^\top (\muv - \Xv^\top \estim{\dualv}) + d'.
    \label{eq:equivalence:primal-substituted-R}
    \end{align}
    
    To conclude, we note that according to Lemma \ref{lemma:hessianprimaldual} we have
    \begin{equation}\label{eq:equivalence:hessian-correspondence}
    \begin{gathered}
        \ddot{\ell}(\dot{\ell}^*(-\estim{\duals}_j; y_j); y_j) =
        (\ddot{\ell}^*(-\estim{\duals}_j; y_j))^{-1}, \\
        \nabla^2 R(\nabla R^*(\Xv^\top \estim{\dualv})) = [\nabla^2 R^*(\Xv^\top \estim{\dualv})]^{-1}.
    \end{gathered}
    \end{equation}
    
    Substitute \eqref{eq:equivalence:hessian-correspondence} in
    \eqref{eq:equivalence:primal-substituted-l} and
    \eqref{eq:equivalence:primal-substituted-R} we obtain the dual of the
    quadratic surrogate equals
    \begin{align}\label{eq:quadratic-dual}
        \frac{1}{2}\sum_j \surrog{\ell}^*(-\theta_j; y_j) +
        \surrog{R}^*(\Xv^\top\theta)
        &=
        \frac{1}{2} \sum_j \ddot{\ell}^*(-\estim{\duals}_j; y_j) \Big(-\theta_j +
        \estim{\duals}_j + \frac{\dot{\ell}^*(-\estim{\theta}_j;
        y_j)}{\ddot{\ell}^*(-\estim{\theta}_j; y_j)}\Big)^2 \nonumber \\
        &\quad +
        \frac{1}{2}(\Xv^\top\thetav - \Xv^\top \estim{\dualv}) \nabla^2
        R^*(\Xv^\top \estim{\dualv}) (\Xv^\top\thetav - \Xv^\top
        \estim{\dualv}) \nonumber \\
        &\quad + [\nabla R^*(\Xv^\top \estim{\dualv})]^\top
        (\Xv^\top\thetav - \Xv^\top \estim{\dualv}) + c'.
    \end{align}
    
    We recognize that the formula given in \eqref{eq:quadratic-dual} exactly
    corresponds to the second-order Taylor expansion of
    \eqref{eq:dual-general-optimal} in the main paper, which is just the form
    of $\surrog{\ell}_D^*$ and $\surrog{R}_D^*$.
\end{proof}

Additionally, we show that the augmented dual method solves the surrogate quadratic problem.
\begin{proof}[Proof of Theorem \ref{thm:primal-dual-equivalence-2}]
    We noted in Section \ref{ssec:dual:general-case} of the main text that our dual method as
    described explicitly approximates the loss by its quadratic expansion
    at the optimal value. We may thus assume without loss of generality that the loss is given by
    $\ell(\mu; y) = (\mu - y)^2 / 2$.

    In this case, as stated in Section \ref{ssec:dual:general-case}, we have
    that
    \begin{equation*}
        \estim{\dualv} = \proxv_{g}(\yv),
    \end{equation*}
    where we have defined $g(\uv) = R^*(\Xv^\top \uv)$. In addition, we note that the augmented observation
    vector $\yv_{a}$ must have its $i$\tsup{th} observation lie on the leave-$i$-out regression line by definition,
    and in particular we have that:
    \begin{equation*}
        [\proxv_{g}(\yv_{a})]_i = 0.
    \end{equation*}

    This motivated us to solve for $\leavei{\surrog{y}_{i}}$ by linearly
    expanding $\proxv_g$ and considering the intersection of its $i$\tsup{th}
    coordinate with 0. Specifically, the desired $\leavei{\surrog{y}_i}$ is
    obtained from the solution of the following linear equation in $z$:
    \begin{equation}\label{eq:dual-optim-linear}
        [\proxv_{g}(\yv) + \Jv_{\proxv_g}(\yv)\ev_i(z - y_i)]_i = 0.
    \end{equation}
    where $\Jv_{\proxv_g}(\yv)$ denotes the Jacobian matrix of $\proxv_g$ at
    $\yv$.

    We show that if $R^*$ is replaced with its quadratic surrogate $\surrog{R}^*$ as defined in the
    Theorem \ref{thm:primal-dual-equivalence}, then:
    \begin{equation*}
        [\proxv_{\surrog{g}}(\surrog{\yv}_a)]_i = 0,
    \end{equation*}
    where $\surrog{g}(\uv) = \surrog{R}^*(\Xv^\top\uv)$, and $\surrog{\yv}_a$ denotes the vector $\yv$,
    except with its $i$\tsup{th} coordinate replaced by the ALO value $\leavei{\surrog{\yv}}_i$.
    Let us note that as $\surrog{g}$ is quadratic, its proximal map $\proxv_{\surrog{g}}$ is linear,
    and the equation may thus be solved directly by a single Newton's step. As
    a linear map is characterized by its intercept and slope, compared with
    \eqref{eq:dual-optim-linear}, it remains to show that:
    \begin{align}
        \proxv_{g}(\yv) &= \proxv_{\surrog{g}}(\yv),
        \label{eq:equiv:check-prox-intercept} \\
        \Jv_{\proxv_{g}}(\yv) &= \Jv_{\proxv_{\surrog{g}}}(\yv).
        \label{eq:equiv:check-prox-coef}
    \end{align}
    
    We note that \eqref{eq:equiv:check-prox-intercept} is immediate from the definition of $\tilde{g}$, as
    both the left and right hand sides are equal to the dual optimal
    $\hat{\dualv}$. In order to show \eqref{eq:equiv:check-prox-coef}, since $\surrog{g}$ is quadratic, we may compute its proximal map
    exactly. From the previous section, we have that:
    \begin{equation*}
        \surrog{g}(\dualv)
        =
        \frac{1}{2}(\thetav - \estim{\dualv})^\top\Xv[\nabla^2 R(\nabla
        R^*(\Xv^\top \estim{\dualv}))]^{-1}
        \Xv^\top(\thetav - \estim{\dualv}) + [\nabla R^*
        (\Xv^\top\estim{\dualv})]^\top\Xv^\top(\thetav - \estim{\dualv}),
    \end{equation*}
    
    We minimize $\frac{1}{2}\|\yv - \thetav\|_2^2 + \surrog{g}(\thetav)$ in
    $\thetav$ and get
    \begin{equation*}
        \proxv_{\surrog{g}}(\yv)
        = (\Iv + \Xv [\nabla^2 R(\nabla R^*(\Xv^\top\estim{\thetav}))]^{-1} \Xv^\top)^{-1}(\yv -
        \Xv\nabla R^*(\Xv^\top\estim{\thetav})),
    \end{equation*}
    
    Notice the primal dual correspondence implies $\estim{\betav} = \nabla
    R^*(\Xv^\top\estim{\thetav})$. In particular we may compute the Jacobian of
    $\proxv_{\surrog{g}}$ at $\yv$ as $(\Iv + \Xv [\nabla^2
    R(\estim{\betav})]^{-1} \Xv^\top)^{-1}$.

    On the other hand, we know that the proximal operator $\proxv_g$ is exactly the resolvent of
    the subgradient $\partial g$:
    \begin{equation*}
        \proxv_g = (I + \partial g)^{-1},
    \end{equation*}
    and in particular we have that:
    \begin{equation*}
        \proxv_g(\yv) + \nabla g(\proxv_g(\yv)) = \yv.
    \end{equation*}

    Taking derivative again with respect to $\yv$ and applying the chain rule, we obtain that:
    \begin{equation*}
        \Jv_{\proxv_g}(\yv)(\Iv + \nabla^2 g(\proxv_g(\yv))) = \Iv,
    \end{equation*}
    and hence that:
    \begin{equation*}
        \Jv_{\proxv_g}(\yv) = (\Iv + \nabla^2 g(\proxv_g(\yv))^{-1}.
    \end{equation*}
    
    Now, note that we have $\proxv_g(\yv) = \estim{\dualv}$, and that:
    \begin{equation*}
        \nabla^2 g(\estim{\dualv}) = \Xv [\nabla^2 R^*(\Xv^\top\estim{\dualv})] \Xv^\top.
    \end{equation*}
    
    We are thus done by Lemma \ref{lemma:hessianprimaldual}.
\end{proof}

\section{Proof of Primal Approximation Approach}\label{sec:primal-approx}
In this section we prove the results of our primal approach on nonsmooth models presented
in Section \ref{sec:primal-smoothing} of the main paper rigorously. Since we use a kernel smoothing
strategy, we start with some useful preliminary results on kernel smoothing. We then
discuss nonsmooth loss and nonsmooth regularizer respectively.

\subsection{Properties of Kernel Smoothing}
\label{append:ssec:property-kernel-smoothing}

In the paper, we consider the following smoothing strategy for a convex
function $f: \mathbb{R} \rightarrow \RR$:
\begin{equation}\label{eq:smoothingffunc}
    f_h(z) = \frac{1}{h}\int f(u)\phi((z - u)/h) du.
\end{equation}

We make the following assumption about the kernel $\phi$:
\begin{description}
    \item Compact support:
        $\phi$ has  a compact support, i.e., $\mathrm{supp}(\phi) = [-C, C]$ for some $C>0$;
    \item Normalization:
        $\phi$ kernel: $\int \phi(w)dw = 1$, $\phi(0) > 0$; $\phi (x) \geq 0$
        for every $x$;
    \item Symmetry: $\phi$ is smooth and symmetric around 0 on $\mathbb{R}$.
\end{description}

Let $K := \{ v_1, \ldots, v_k \}$ denote the set of zero-order singularities of the
function $f$. Denote by $\dot{f}_-$ and $\dot{f}_+$ the left and right
derivative of $f$. Our next lemma summarizes some of the basic properties of
$f$ that may be used in the proofs of Theorem \ref{thm:nonsmooth-loss-approx}
and \ref{thm:nonsmooth-reg-approx} of the main text.

\begin{lemma}\label{lemma:kernel-smooth-property}
The smooth function $f_h$ verifies the following properties:
\begin{enumerate}
    \item
        $f_h(z) \geq f(z)$ for all $z \in \mathbb{R}$;
    \item For all $z \in K^C$, for all $h$ small enough:
        \begin{equation*}
            \dot{f}_h(z) = \frac{1}{h}\int \dot{f}(u)\phi((z - u)/h) du
            ,\quad
            \ddot{f}_h(z) = \frac{1}{h}\int \ddot{f}(u)\phi((z - u)/h) du.
        \end{equation*}
    \item For all $z \in K$:
        \begin{equation*}
            \lim_{h \rightarrow 0} \dot{f}_h(z) = \frac{\dot{f}_{-}(z) + \dot{f}_{+}(z)}{2}
            ,\quad
            \lim_{h \rightarrow 0} \ddot{f}_h(z) = +\infty.
        \end{equation*}
    \item
        If $f$ is locally Lipschiz in the sense that, for any $A > 0$, and for any
        $x, y \in [-A, A]$, we have $|f(x) - f(y)| \leq L_A|x - y|$, where
        $L_A$ is a constant that only depends on $A$; then $f_h(z)$ converges
        to $f(z)$ uniformly on any compact set.
\end{enumerate}
\end{lemma}

\begin{proof}
    For part 1, by the normalization property of $\phi$, we can treat $\phi$ as a probability
    density. Consider the random variable $U \sim \frac{1}{h}\phi(\frac{z -
    u}{h})$. From the convexity of $f$ and Jensen's inequality we have
    \begin{equation*}
        f_h(z) = \mathbb{E}f(U) \geq f(\mathbb{E}U) = f(z).
    \end{equation*}
   
    For part 2, note that 
    \begin{equation*}
        \dot{f}_h(z)
        =
        \frac{1}{h^2}\int f(u)\dot{\phi}((z - u) / h) du
        =
        \int \dot{f}(u)\frac{1}{h}\phi((z - u) / h) du.
    \end{equation*}
    
    A similar computation gives the stated equation for $\ddot{f}_h(z)$.
    
    For part 3, when $z \in K$, we have by compact support of $\phi$ that
    as $h \rightarrow 0$:
    \begin{align*}
        \dot{f}_h(z)
        &=
        \frac{1}{h^2}\int_{z-hC}^{z} f(u)\dot{\phi}((z - u) / h) du +
        \frac{1}{h^2}\int_{z}^{z+hC} f(u)\dot{\phi}((z - u) / h) du \nonumber \\
        &=
        \int_{-C}^{0} \dot{f}(z - hw)\phi(w) dw + \int_{0}^{C} \dot{f}(z - hw)\phi(w) dw \nonumber \\
        & \rightarrow
        \int_{-C}^{0} \dot{f}_+(z)\phi(w) dw
        + \int_{0}^{C} \dot{f}_{-}(z)\phi(w) dw \nonumber \\
        &= \frac{\dot{f}_+(z) + \dot{f}_-(z)}{2}.
    \end{align*}
    
    A similar computation for the second-order derivative yields:
    \begin{align*}
        \ddot{f}_h(z)
        &=
        \frac{1}{h^3}\int_{z-hC}^{z} f(u)\ddot{\phi}((z - u) / h) du +
        \frac{1}{h^3}\int_{z}^{z+hC} f(u)\ddot{\phi}((z - u) / h) du \nonumber
        \\
        &=
        \frac{1}{h}\phi(0)(\dot{f}_+(z) - \dot{f}_-(z)) + \int_0^C \ddot{f}(z -
        hw)\phi(w)dw + \int_{-C}^0 \ddot{f}(z - hw)\phi(w)dw \nonumber \\
        &\rightarrow
        \infty.
    \end{align*}
    noticing that $\dot{f}_+(z) > \dot{f}_-(z)$.

    For part 4, for any compact set $\mathcal{C}$ which can be covered by a large enough set
    $[-A, A]$ for some $A > 0$, we have
    \begin{equation*}
        \sup_{z \in \mathcal{C}}|f_h(z) - f(z)|
        \leq
        \sup_{z \in \mathcal{C}}\int_{-C}^C|f(z - hw) - f(z)|\phi(w)dw
        \leq 2hCL_{A+C}
        \rightarrow 0
        ,\quad
        \text{ as }h \rightarrow 0
    \end{equation*}
\end{proof}

Having established the basic properties of our kernel smoothing strategy,
we apply them to non-smooth loss and non-smooth regularizer respectively.

\subsection{Proof of Theorem \ref{thm:nonsmooth-reg-approx}:
Nonsmooth Separable Regularizer With Smooth Loss}
\label{append:ssec:primal-nonsmooth-reg}

Consider the penalized regression problem:
\begin{equation}\label{eq:nonsmooth-reg-main}
    \estim{\betav} = \argmin_{\betav} \sum_{j = 1}^n \ell(\xv_j^\top\betav; y_j)
    + \lambda \sum_l r(\beta_l).
\end{equation}
with $\ell$ and $r$ being twice differentiable and nonsmooth functions
respectively. Let $r_h$ be the smoothed version of $r$ constructed as
in \eqref{eq:smoothingffunc}. Define
\begin{equation*}
    \estim{\betav}_h=\arg\min_{\betav}\sum_j \ell(\xv_j^\top\betav; y_j) + \lambda \sum_l r_h(\beta_l).
\end{equation*}

As before, let $K$ denote the set of all zero-order singularities of $r$.
We make the following assumptions on the regularizer.

\begin{assumption}\label{assum:nonsmooth-reg-assump1}
    We will need the following assumptions on the problem.
    \begin{enumerate}
        \item
            $r$ is locally Lipschiz in the sense that, for any $A > 0$, and for any
            $x, y \in [-A, A]$, we have $|r(x) - r(y)| \leq L_A|x - y|$, where
            $L_A$ is a constant that only depends on $A$;
        \item
            $\estim{\betav}$ is the unique minimizer of
            \eqref{eq:nonsmooth-reg-main};
        \item
            When $\estim{\beta}_l = v \in K$, the subgradient $g_r(\estim{\beta}_l)$
            of $r$ at $\estim{\beta}_l$ satisfies
            $g_r(\estim{\beta}_l) \in (\dot{r}_-(v), \dot{r}_+(v))$.
        \item
            $r$ is coercive in the sense that $|r(z)| \rightarrow \infty$ as
            $|z| \rightarrow \infty$. \\
    \end{enumerate}
\end{assumption}

\begin{lemma}\label{lemma:nonsmooth-reg-boundedness}
    Suppose that Assumption \ref{assum:nonsmooth-reg-assump1} holds.
    There exists $M > 0$ that only depends on $r, \ell$ and $\lambda$, such that
    we have for any $h \leq 1$:
    \begin{equation*}
        \norm{\estim{\betav}}_\infty, \norm{\estim{\betav}_h}_\infty < M.
    \end{equation*}
\end{lemma}
\begin{proof}
    Let $h \leq 1$, then the minimizer of the smoothed version $\estim{\betav}_h$ satifies
    \begin{align*}
        \lambda \sum_{l=1}^p r([\estim{\betav}_h]_l)
        &\leq \lambda \sum_{l=1}^p r_h([\estim{\betav}_h]_l) \nonumber \\
        &\leq \sum_i \ell(y_i; 0) + \lambda p r_h(0) \nonumber \\
        &= \sum_i \ell(y_i; 0) + \lambda p \int_{-C}^C r(hw) \phi(w) dw \nonumber \\
        & \leq \sum_i \ell(y_i; 0) + \lambda p \sup_{|w|\leq C} r(w).
    \end{align*}

    On the other hand, the minimizer $\estim{\betav}$ of the original problem satisfies
    \begin{equation*}
        \lambda \sum_{l=1}^p r([\estim{\betav}]_l) \leq \sum_i \ell(y_i; 0) + \lambda p r(0)
        \leq \sum_i \ell(y_i; 0) + \lambda p \sup_{|w|\leq C} r(w).
    \end{equation*}
    
    The convexity and coerciveness of $r$ implies that there exists an $M$, such that
    for all $h \leq 1$:
    \begin{equation*}
        \norm{\estim{\betav}_h}_\infty \leq M \text{ and } \norm{\estim{\betav}}_\infty \leq M.
    \end{equation*}
\end{proof}

\begin{lemma}\label{lemma:nonsmooth-reg-estimator-converge}
    Suppose that Assumption \ref{assum:nonsmooth-reg-assump1} holds. Then the
    smoothed version converges to the original problem in the sense that:
    \begin{equation*}
        \norm{\estim{\betav}_h - \estim{\betav}}_2 \rightarrow 0 \text{ as } h \rightarrow 0.
    \end{equation*}
\end{lemma}
\begin{proof}
    By the local Lipschitz condition of $r$, we have for any $z \leq M$ and $h \leq 1$:
    \begin{equation}\label{eq:nonsmooth-reg-estimator-proof-lipschitz}
        0 \leq r_h(z) - r(z)
        =
        \int_{-C}^C [r(z - hw) - r(z)]\phi(w)dw \leq 2CL_{M + C}h.
    \end{equation}
    
    Let $P_h(\betav) := \sum_j \ell(\xv_j^\top\betav; y_j) + \lambda \sum_l
    r_h(\beta_l)$ denote the primal objective value. \eqref{eq:nonsmooth-reg-estimator-proof-lipschitz}
    implies that:
    \begin{equation*}
        \sup_{\|\betav\|_\infty \leq M} \abs{P(\betav) - P_h(\betav)} \leq 2hpCL_{M+C}
    \end{equation*}
    
    By Lemma \ref{lemma:nonsmooth-reg-boundedness} $\estim{\betav}_h$ is in a
    compact set. Hence, any of its subsequence
    contains a convergent sub-subsequence. Let us abuse the notation and denote
    by $\estim{\betav}_h$ any of such convergent sub-subsquence, that is, assume that
    $\estim{\betav}_h \rightarrow \estim{\betav}_0$. Along such a sub-subsequence, we
    have that:
    \begin{equation*}
        P(\estim{\betav}_0)
        =
        \lim_{h\rightarrow 0} P(\estim{\betav}_h)
        =
        \lim_{h\rightarrow 0} P_h(\estim{\betav}_h)
        \leq
        \lim_{h\rightarrow 0} P_h(\estim{\betav})
        =
        \lim_{h\rightarrow 0} P(\estim{\betav}).
    \end{equation*}

    The uniqueness of the minimizer implies $\estim{\betav}_0 = \estim{\betav}$.
    As the above holds along any convergent sub-subsequence, we have that:
    \begin{equation*}
        \norm{\estim{\betav}_h - \estim{\betav}}_2 \rightarrow 0 \text{ as } h \rightarrow 0.
    \end{equation*}
\end{proof}

\begin{lemma}[Convergence of the subgradients]\label{lemma:nonsmooth-reg-gradient-converge}
    Suppose that Assumption \ref{assum:nonsmooth-reg-assump1} holds.
    Recall that we use $R(\betav) = \sum_{l=1}^p r(\beta_l)$. We have that:
    \begin{equation*}
        \|\nabla R_h(\estim{\betav}_h) - \gv_R(\estim{\betav})\|_2
        \rightarrow 0
        ,\quad
        \text{ as } h  \rightarrow 0.
    \end{equation*}
    where $g_R(\estim{\betav})$ is the subgradient of $R$ at $\estim{\betav}$.
\end{lemma}
\begin{proof}
    By the first-order optimality conditions and the continuity of $\ell$, we
    have that as $h \rightarrow 0$:
    \begin{equation*}
        \|\nabla R_h(\estim{\betav}_h) - \gv_R(\estim{\betav}) \|_2
        =
        \Big\|\sum_j \ell(\xv_j^\top\estim{\betav}; y_j) - \sum_j
        \ell(\xv_j^\top\estim{\betav}_h; y_j)\Big\|_2
        \rightarrow
        0.
    \end{equation*}
\end{proof}

\begin{lemma}[Convergence of the Hessian]\label{lemma:nonsmooth-reg-hessian-converge}
    Suppose that Assumption \ref{assum:nonsmooth-reg-assump1} holds. We have that
    as $h \rightarrow 0$:
    \begin{equation*}
    \ddot{r}_h(\estim{\beta}_{h,i})
    \rightarrow
    \begin{cases}
        \ddot{r}(\estim{\beta}_i) & \text{ if } \estim{\beta}_i \notin K, \\
        +\infty & \text{ if } \estim{\beta}_i \in K.
    \end{cases}
    \end{equation*}
\end{lemma}
\begin{proof}
    Let us first consider the case $\estim{\beta}_i \notin K$.
    As $\RR \setminus K$ is open,
    there exists $\delta > 0$ such that
    $[\estim{\beta}_i - \delta, \estim{\beta}_i + \delta] \subset \mathbb{R}
    \backslash K$.
    Since $\estim{\beta}_{h,i} \rightarrow \estim{\beta}_i$ as $h \rightarrow 0$,
    we have for $h$ small enough that:
    \begin{equation*}
        [\estim{\beta}_{h,i} - hC, \estim{\beta}_{h,i} + hC]
        \subset [\estim{\beta}_i - \delta, \estim{\beta}_i + \delta]
        \subset \mathbb{R}\backslash K.
    \end{equation*}
    
    Since $\ddot{r}$ is smooth on $[\estim{\beta}_i - \delta, \estim{\beta}_i +
    \delta]$, by the bounded convergence theorem, we have as $h \rightarrow 0$:
    \begin{equation*}
        \ddot{r}_h(\estim{\beta}_{h,i})
        =
        \int_{-C}^C \ddot{r}(\estim{\beta}_{h,i} - hw)\phi(w)dw
        \rightarrow
        \int_{-C}^C \ddot{r}(\estim{\beta}_i)\phi(w)dw
        =
        \ddot{r}(\estim{\beta}_i)
    \end{equation*}

    Now, let us consider the case where $\estim{\beta}_i \in K$.
    By Lemma \ref{lemma:nonsmooth-reg-gradient-converge},
    we have that $\dot{r}_h(\estim{\beta}_{h,i}) \rightarrow g_r(\estim{\beta}_i)$,
    from which we deduce:
    \begin{equation*}
        |\estim{\beta}_{h,i} - \estim{\beta}_i| < hC.
    \end{equation*}
    
    Indeed, if we had $\estim{\beta}_i \geq \estim{\beta}_{h,i} + hC$, notice
    the assumption on the subgradient $g_r(\estim{\beta}_i)$, this would imply:
    \begin{equation*}
        \dot{r}_h(\estim{\beta}_{h,i}) = \int_{-C}^C \dot{r}(\estim{\beta}_{h,i} - hw)\phi(w)dw
        \leq
        \dot{r}_{-}(\estim{\beta}_i)
        < g_r(\estim{\beta}_i),
    \end{equation*}
    which is contradictory. The same happens if $\estim{\beta}_i \leq \estim{\beta}_{h,i} - hC$.
    To conclude, note that as $h \rightarrow 0$:
    \begin{align*}
        \ddot{r}_h(\estim{\beta}_{h,i})
        &= \int_{\estim{\beta}_{h, i}-hC}^{\estim{\beta}_i}
        r(u)\frac{1}{h^3}\ddot{\phi}\Big(\frac{\estim{\beta}_{h, i} - u}{h}\Big) du
        + \int_{\estim{\beta}_i}^{\estim{\beta}_{h, i}+hC}
        r(u)\frac{1}{h^3}\ddot{\phi}\Big(\frac{\estim{\beta}_{h, i} - u}{h}\Big) du \nonumber \\   
        &= \frac{1}{h}\phi\Big(\frac{\estim{\beta}_{h, i} -
        \estim{\beta}_i}{h}\Big)(\dot{r}_+(\estim{\beta}_i) - \dot{r}_-(\estim{\beta}_i))
        + \int_{\frac{\estim{\beta}_{h, i} - \estim{\beta}_i}{h}}^C \ddot{r}(\estim{\beta}_{h, i} - hw)\phi(w)dw \nonumber \\
        & \quad + \int_{-C}^{\frac{\estim{\beta}_{h, i}
        - \estim{\beta}_i}{h}} \ddot{r}(\estim{\beta}_{h, i} - hw)\phi(w)dw \nonumber \\
        &\rightarrow +\infty.
    \end{align*}
\end{proof}

\begin{lemma}\label{lemma:woodbury-block}
    Consider a sequence of matrices $\Av_n, n \in \mathbb{N}$, and let $\Av_n = \Big[\begin{array}{ll}
            \Av_{1n} & \Av_{2n} \\
            \Av_{3n} & \Av_{4n}
    \end{array}\Big]$ where $\Av_{1n}, \Av_{4n}$ are invertible for all $n$. Additionally,
    suppose that $\Av_{in} \rightarrow \Av_i, i = 1, 2, 3$, and $\Av_{4n}^{-1} \rightarrow \zerov$ as $n \rightarrow \infty$.
    Then we have as $n \rightarrow \infty$ that:
    \begin{equation*}
    \Av_n^{-1} \rightarrow
    \Big[\begin{array}{ll}
            \Av_1^{-1} & \zerov \\
            \zerov & \zerov
    \end{array}\Big].
    \end{equation*}
\end{lemma}
\begin{proof}
    By the Woodbury matrix identity \cite{woodbury1950inverting}, we have
    \begin{align*}
        \Av_n^{-1} =&
        \Big[\begin{array}{ll}
                (\Av_{1n} - \Av_{2n}\Av_{4n}^{-1}\Av_{3n})^{-1} & -(\Av_{1n} -
                \Av_{2n}\Av_{4n}^{-1}\Av_{3n})^{-1}\Av_{2n}\Av_{4n}^{-1} \\
                -\Av_{4n}^{-1}\Av_{3n}(\Av_{1n} - \Av_{2n}\Av_{4n}^{-1}\Av_{3n})^{-1} &
                \Av_{4n}^{-1}\Av_{3n}(\Av_{1n} -
                \Av_{2n}\Av_{4n}^{-1}\Av_{3n})^{-1}\Av_{2n}\Av_{4n}^{-1} + \Av_{4n}^{-1}
        \end{array}\Big] \nonumber \\
        \rightarrow&
        \Big[\begin{array}{ll}
                \Av_{1}^{-1} & \zerov \\
                \zerov & \zerov
        \end{array}\Big].
    \end{align*}
\end{proof}

\begin{proof}[Proof of Theorem \ref{thm:nonsmooth-reg-approx}]
    The proof of Theorem \ref{thm:nonsmooth-reg-approx}
    is a straightforward corollary of the Lemmas
    \ref{lemma:nonsmooth-reg-estimator-converge},
    \ref{lemma:nonsmooth-reg-gradient-converge},
    \ref{lemma:nonsmooth-reg-hessian-converge} and \ref{lemma:woodbury-block}.
\end{proof}

\subsection{Proof of Theorem \ref{thm:nonsmooth-loss-approx}:
Nonsmooth Loss With Smooth Regularizer}
\label{append:ssec:primal-nonsmooth-loss}

We now consider the case of non-smooth loss. The proof is very similar to
the previous section, so we briefly mention the common parts and focus
on the differences.

Consider nonsmooth loss $\ell$ and its smoothed version $\ell_h$. $R$ is assumed
to be smooth. Let us consider:
\begin{equation*}
    \begin{aligned}
        P(\betav) &= \sum_{j=1}^n \ell(\xv_j^\top\betav; y_j) + R(\betav), \\
        P_h(\betav) &= \sum_{j = 1}^ n \ell_h(\xv_j^\top\betav; y_j) + R(\betav).
    \end{aligned}
\end{equation*}

Let us still use $\estim{\betav} = \argmin_{\betav} P(\betav)$
and $\estim{\betav}_h =\argmin_{\betav} P_h(\betav)$ to denote the optimizers.
As before, let $K = \{ v_1, \dotsc, v_k \}$ denote the zero-order singularities
of $\ell$, and let $V = \{ i: \xv_i^\top \estim{\betav} \in K \}$ be the set of indices
of observations at such singularities.

\begin{assumption}\label{assum:nonsmooth-loss-assump1}
    We need the following assumptions on $\ell$, $R$ and $\estim{\betav}$:
    \begin{enumerate}
        \item
            $\ell$ is locally Lipschitz, that is, for any $A > 0$, for any $x, y
            \in [-A, A]$, we have $|\ell(x) - \ell(y)| \leq L_A|x - y|$, where $L_A$ is a
            constant depends only on $A$.
        \item
            $\lambda_{\min}(\Xv_{V}\Xv_{V}^\top) > 0$.
        \item
            $\estim{\betav}$ is the unique minimizer.
        \item
            Whenever $\xv_j^\top\estim{\betav} = v \in K$, the subgradient of
            $\ell$ at $\xv_j^\top\estim{\beta}$, $g_\ell(\xv^\top \estim{\betav})$ satisfies
            $g_\ell(\xv^\top \estim{\betav}) \in (\ell_-(v), \ell_+(v))$.
        \item
            $R$ is coercive in the sense that $|R(\betav)| \rightarrow \infty$ as
            $\|\betav\| \rightarrow \infty$.
    \end{enumerate}
\end{assumption}

\begin{lemma}\label{lemma:nonsmooth-loss-boundedness}
    Suppose that Assumption \ref{assum:nonsmooth-loss-assump1} holds.
    There exists $M > 0$ that only depends on $r, \ell$ and $\lambda$,
    such that for all $h \leq 1$, we have:
    \begin{equation*}
        \norm{\estim{\betav}}_\infty \leq M \text{ and } \norm{\estim{\betav}_h}_\infty \leq M.
    \end{equation*}
\end{lemma}
\begin{proof}
    Let $h \leq 1$, then $\estim{\betav}_h$ verifies:
    \begin{align*}
        R(\estim{\betav}_h)
        \leq& \sum_j \ell_h(0; y_j) + p R(0) \nonumber \\
        =& \sum_j \int_{-C}^C \ell(hw; y_j) \phi(w) dw + p R(0)
        \leq \sum_j \sup_{|w|\leq C}\ell(w; y_i) + p R(0).
    \end{align*}

    Additionally, $\estim{\betav}$ verifies:
    \begin{equation*}
        R(\estim{\betav})
        \leq \sum_j \ell(0; y_j) + p R(0)
        \leq \sum_j \sup_{|w|\leq C}\ell(w; y_i) + p R(0).
    \end{equation*}
    
    The convexity and coerciveness of $R$ implies that there exists a $M$, such that
    for all $h \leq 1$:
    \begin{equation*}
        \|\estim{\betav}_h\|_2 \leq M \text{ and } \|\estim{\betav}\|_2 \leq M.
    \end{equation*}
\end{proof}

\begin{lemma}\label{lemma:nonsmooth-loss-estimator-converge}
    Suppose that Assumption \ref{assum:nonsmooth-loss-assump1} holds. We have that as
    $h \rightarrow 0$:
    \begin{equation*}
        \|\estim{\betav}_h - \estim{\betav}\|_2 \rightarrow 0.
    \end{equation*}
\end{lemma}
\begin{proof}
    Let $M_x = \max_i\|\xv_i\|_2$.
    By the local Lipschitz condition of $\ell$, we have that for any $\|\betav\|_2 \leq M$ and $h \leq 1$ that:
    \begin{align*}
        0 &\leq \ell_h(y_i; \xv_i^\top\betav) - \ell(y_i; \xv_i^\top\betav) \\
          &= \int_{-C}^C [\ell(y_i; \xv_i^\top\betav - hw) - \ell(y_i;\xv_i^\top\betav)]\phi(w)dw \\
          &\leq 2CL_{M_xM + C}h.
    \end{align*}
    
    This implies:
    \begin{equation*}
        \sup_{\|\betav\|_2 \leq M} |P(\betav) - P_h(\betav)| \leq
        2nhCL_{M_xM+C}.
    \end{equation*}
    
    From Lemma \ref{lemma:nonsmooth-loss-boundedness}, we know
    $\estim{\betav}_h$ is in a compact set, thus any of its subsequence
    contains a convergent sub-subsequence. Again abuse the notation and
    let $\estim{\betav}_h$ denote this convergent sub-subsequence. Suppose that:
    $\estim{\betav}_h \rightarrow \estim{\betav}_0$. Now we have again:
    \begin{equation*}
        P(\estim{\betav}_0)
        =
        \lim_{h\rightarrow 0} P(\estim{\betav}_h)
        =
        \lim_{h\rightarrow 0} P_h(\estim{\betav}_h)
        \leq
        \lim_{h\rightarrow 0} P_h(\estim{\betav})
        =
        \lim_{h\rightarrow 0} P(\estim{\betav}).
    \end{equation*}
    
    The uniqueness implies $\estim{\betav}_0 = \estim{\betav}$. As the previous result
    holds along any sub-subsequence, we deduce that:
    \begin{equation*}
        \| \estim{\betav}_h - \estim{\betav} \|_2 \rightarrow 0.
    \end{equation*}
\end{proof}

\begin{lemma}[Convergence of gradients]\label{lemma:nonsmooth-loss-gradient-converge}
    Suppose that Assumption \ref{assum:nonsmooth-loss-assump1} holds. Then, we have that
    for any $j$, as $h \rightarrow 0$:
    \begin{equation*}
        \| \dot{\ell}_h(\xv_j^\top\estim{\betav}_{h}) -
        g_\ell(\xv_j^\top\estim{\betav}) \|_2 \rightarrow 0.
    \end{equation*}
\end{lemma}
\begin{proof}
    for $j \notin V$, the result is immediate. For $j \in V$, we have that
    as $h \rightarrow 0$:
    \begin{equation*}
        \Big\|\sum_{j \in V} \xv_j\dot{\ell}_h(\xv_j^\top\estim{\betav}_h; y_j) -
        \sum_{j \in V} \xv_j g_\ell(\xv_j^\top\estim{\betav}; y_j) \Big\|_2
        \rightarrow
        0.
    \end{equation*}

    This implies the desired result by the assumption on $\Xv_{V,\cdot}$.
\end{proof}

\begin{lemma}[Convergence of Hessian]\label{lemma:nonsmooth-loss-hessian-converge}
    Suppose that Assumption \ref{assum:nonsmooth-loss-assump1} holds. Then, we have
    that for any $j$, as $h \rightarrow 0$:
    \begin{equation*}
    \ddot{\ell}_h(\xv_j^\top\estim{\betav}_h; y_j)
    \rightarrow
    \begin{cases}
        \ddot{\ell}(\xv_j^\top\estim{\betav}; y_j) & \text{ if } j \notin V, \\
        +\infty & \text{ if } j \in V. \\
    \end{cases}
    \end{equation*}
\end{lemma}
\begin{proof}
    Again, the result follows through a similar argument as
    in the proof of Lemma \ref{lemma:nonsmooth-reg-hessian-converge}
    for $j \notin V$.
    For $j \in V$, we have by Lemma \ref{lemma:nonsmooth-loss-gradient-converge} that
    as $h \rightarrow 0$:
    \begin{equation*}
        \dot{\ell}_h(\xv_j^\top\estim{\betav}_h; y_j) \rightarrow g_\ell(\xv_j^\top\estim{\betav}; y_j).
    \end{equation*}
    
    Following a similar reasoning as in the proof of Lemma \ref{lemma:nonsmooth-reg-hessian-converge},
    we have that:
    \begin{equation*}
        |\xv_j^\top\estim{\betav}_h - \xv_j^\top\estim{\betav}| < hC.
    \end{equation*}

    Finally, we note that as $h \rightarrow 0$:
    \begin{equation*}
        \ddot{\ell}_h(\xv_j^\top\estim{\betav}_h; y_j)
        \geq
        \frac{1}{h}\phi\Big(\frac{\xv_j^\top\estim{\betav}_h -
        \xv_j^\top\estim{\betav}}{h}\Big)(\dot{\ell}_+(\xv_j^\top\estim{\betav}) -
        \dot{\ell}_-(\xv_j^\top\estim{\betav}))
        \rightarrow
        +\infty.
    \end{equation*}
\end{proof}

\begin{proof}[Proof of Theorem \ref{thm:nonsmooth-loss-approx}]
Recall $V = \{i: \xv_i^\top \estim{\betav} \in K\}$ and $S=[1:n]\backslash V$.
Let $\Hv_h$ be the matrix in ALO for smooth loss and smooth regularizer when
using $\ell_h$. Let $\Lv_h=\diag[\{\ddot{\ell}_h(\xv_j^\top\estim{\betav}; y_j)\}_j]$,
$\Lv_S=\diag[\{\ddot{\ell}(\xv_j^\top\estim{\betav}; y_j)\}_{j \in S}]$.
$\Lv_{h, S}$ and $\Lv_{h, V}$ are similarly defined. Recall
\begin{equation*}
    \Hv_h = \Xv(\lambda \nabla^2 R + \Xv^\top \Lv_h\Xv)^{-1}\Xv^\top.
\end{equation*}

We then have
\begin{align*}
    &(\lambda \nabla^2 R + \Xv^\top \Lv_h \Xv)^{-1} \nonumber \\
    =&
    (\underbrace{\lambda \nabla^2 R + \Xv_{S,\cdot}^\top \Lv_{h,S} \Xv_{S,\cdot}}_{\Yv_h} + \Xv_{V,\cdot}^\top \Lv_{h,V} \Xv_{V,\cdot})^{-1} \nonumber \\
    =&
    \Yv_h^{-1} - \Yv_h^{-1}\Xv_{V,\cdot}^\top (\Lv_{h,V}^{-1} +
    \Xv_{V,\cdot}\Yv_h^{-1}\Xv_{V,\cdot}^\top )^{-1}\Xv_{V,\cdot} \Yv_h^{-1}.
\end{align*}

As a result, we have
\begin{align*}
    &(\lambda \nabla^2 R + \Xv^\top \Lv_h X)^{-1}\Xv_{V,\cdot}^\top \nonumber \\
    =&
    \Yv_h^{-1}\Xv_{V,\cdot}^\top - \Yv_h^{-1}\Xv_{V,\cdot}^\top (\Lv_{h,V}^{-1} + \Xv_{V,\cdot}\Yv_h^{-1}\Xv_{V,\cdot}^\top )^{-1}\Xv_{V,\cdot} \Yv_h^{-1}\Xv_{V,\cdot}^\top \nonumber \\
    =&
    \Yv_h^{-1}\Xv_{V,\cdot}^\top (\Iv_p - (\Lv_{h,V}^{-1} + \Xv_{V,\cdot}\Yv_h^{-1}\Xv_{V,\cdot}^\top )^{-1}\Xv_{V,\cdot} \Yv_h^{-1}\Xv_{V,\cdot}^\top ) \nonumber \\
    =&
    \Yv_h^{-1}\Xv_{V,\cdot}^\top (\Lv_{h,V}^{-1} +
    \Xv_{V,\cdot}\Yv_h^{-1}\Xv_{V,\cdot}^\top )^{-1}\Lv_{h,V}^{-1}.
\end{align*}

Similarly we can get
\begin{align*}
    \Xv_{V,\cdot}(\lambda \nabla^2 R + \Xv^\top \Lv_h \Xv)^{-1} =& \Lv_{h,V}^{-1} (\Lv_{h,V}^{-1} + \Xv_{V,\cdot}\Yv_h^{-1}\Xv_{V,\cdot}^\top )^{-1}\Xv_{V,\cdot} \Yv_h^{-1} \\
    \Xv_{V,\cdot}(\lambda \nabla^2 R + \Xv^\top \Lv_h
    \Xv)^{-1}\Xv_{V,\cdot}^\top =& \Lv_{h,V}^{-1} - \Lv_{h,V}^{-1}
    (\Lv_{h,V}^{-1} + \Xv_{V,\cdot}\Yv_h^{-1}\Xv_{V,\cdot}^\top)^{-1} \Lv_{h,V}^{-1}.
\end{align*}

By Lemma \ref{lemma:nonsmooth-loss-hessian-converge}, $\Yv_h \rightarrow \Yv :=
\lambda \nabla^2 R + \Xv_{S,\cdot}^\top \Lv_S \Xv_{S,\cdot}$, $\Lv_{h,V}^{-1} \rightarrow
\mathbf{0}$, we have
\begin{align*}
    \Hv_{h, S,S}\Lv_{h,S} \rightarrow& \Xv_{S,\cdot} (\Yv^{-1} -
    \Yv^{-1}\Xv_{V,\cdot}^\top (\Xv_{V,\cdot},
    \Yv^{-1}\Xv_{V,\cdot}^\top )^{-1}\Xv_{V,\cdot} \Yv^{-1})\Xv_{S,\cdot}^\top
    \Lv_S, \\
    \Hv_{h, S,V}\Lv_{h,V} \rightarrow& \Xv_{S,\cdot} \Yv^{-1}\Xv_{V,\cdot}^\top
    (\Xv_{V,\cdot}\Yv^{-1}\Xv_{V,\cdot}^\top )^{-1}, \\
    \Hv_{h, V,S}\Lv_{h,S} \rightarrow& \mathbf{0} \\
    \Hv_{h, V,V}\Lv_{h,V} \rightarrow& \Iv_V.
\end{align*}

This is not enough, however, noticing that in the final formula of the smooth
case, we need $\frac{H_{h, ii}}{1 - L_{h, ii}H_{h, ii}}$ but for $i\in V$,
$1 - L_{h,ii}H_{h,ii}\rightarrow 0$ and $H_{h, ii} \rightarrow 0$. So
further we have
\begin{align*}
    & \Lv_{h,V}(\Iv_V - \Hv_{h, VV}\Lv_{h,V}) \nonumber \\
    =&
    \Lv_{h,V}(\Iv_V - (\Lv_{h,V}^{-1} - \Lv_{h,V}^{-1} (\Lv_{h,V}^{-1} +
    \Xv_{V,\cdot}\Yv_h^{-1}\Xv_{V,\cdot}^\top )^{-1}\Lv_{h,V}^{-1}) \Lv_{h,V}) \nonumber \\
    =&
    (\Lv_{h,V}^{-1} + \Xv_{V,\cdot}\Yv_h^{-1}\Xv_{V,\cdot}^\top )^{-1} \nonumber \\
    \rightarrow&
    (\Xv_{V,\cdot}\Yv^{-1}\Xv_{V,\cdot}^\top )^{-1}.
\end{align*}

As a result, we have
\begin{equation*}
    \frac{H_{h,ii}}{1 - L_{h,ii}H_{h,ii}}
    \rightarrow
    \left\{
        \begin{array}{ll}
            \frac{\xv_i^\top (\Yv^{-1} - \Yv^{-1}\Xv_{V,\cdot}^\top (\Xv_{V,\cdot}
            \Yv^{-1}\Xv_{V,\cdot}^\top )^{-1}\Xv_{V,\cdot} \Yv^{-1})\xv_i}{1 - \xv_i (\Yv^{-1} -
            \Yv^{-1}\Xv_{V,\cdot}^\top (\Xv_{V,\cdot} \Yv^{-1}\Xv_{V,\cdot}^\top )^{-1}\Xv_{V,\cdot}
            \Yv^{-1})\xv_i \ddot{\ell}_i}, & i \in S, \\
            \frac{1}{[(\Xv_{V,\cdot}\Yv^{-1}\Xv_{V,\cdot}^\top )^{-1}]_{ii}}, &
            i \in V. \\
        \end{array}
    \right.
\end{equation*}

For $\dot{\ell}_h(\xv_i^\top\estim{\betav}_h; y_i)$, as $h\rightarrow 0$, Lemma
\ref{lemma:nonsmooth-loss-gradient-converge} implies the limit value the smooth
gradients would converge to. Notice that for $j \in V$, we solve for the
subgradient by applying least square formula to the 1st order optimality
equation. The final results easily follow.
\end{proof}

\section{Derivation of the Dual for Generalized LASSO}
\label{append:sec:generalized-lasso-dual}

In this section we derive the dual form of the generalized LASSO stated in
the main paper. We recall that for a given matrix $\Dv \in \RR^{m \times p}$,
the generalized LASSO is given by:
\begin{equation*}
    \min_{\betav} \frac{1}{2} \sum_{j = 1}^n (y_j - \xv_j^\top \betav)^2
    + \lambda \norm{\Dv \betav}_1.
\end{equation*}

Introduce dummy variables $\zv \in \RR^n$, $\wv \in \RR^m$, and consider the following
equivalent constrained optimization problem:
\begin{equation*}
    \begin{gathered}
    \min_{\betav, \zv, \wv} \frac{1}{2}\norm{\zv}_2^2 + \lambda \norm{\wv}_1, \\
    \text{subject to: } \yv - \Xv\betav = \zv \text{ and } \Dv \betav = \wv.
    \end{gathered}
\end{equation*}

We may now consider the Lagrangian form of the optimization problem, introducing dual variables
$\dualv \in \RR^n$ and $\uv \in \RR^m$, the dual problem is
\begin{align*}\label{eq:genlasso-dual:lagrangian}
    &
    \max_{\dualv, \uv} \min_{\betav, \zv, \wv} \frac{1}{2} \norm{\zv}_2^2 + \lambda \norm{\wv}_1
    + \dualv^\top(\yv - \Xv\betav - \zv) + \uv^\top(\Dv \betav - \wv) \nonumber \\
    =&
    - \min_{\dualv, \uv} \big[ \max_{\zv} \{\thetav^\top \zv
    - \frac{1}{2}\|\zv\|_2^2 \} + \max_{\wv} \{ \uv^\top \wv - \lambda
    \norm{\wv}_1 \} + \max_{\betav} \{\thetav^\top\Xv\betav -
\uv^\top\Dv\betav\} - \thetav^\top\yv \big].
\end{align*}

Consider the three subproblems within square brackets respectively, we have
\begin{equation*}
    \begin{gathered}
        \max_{\zv} \{\dualv^\top \zv - \frac{1}{2}\norm{\zv}_2^2 \} = \frac{1}{2} \norm{\dualv}_2^2, \\
        \max_{\wv} \{\uv^\top \wv - \lambda \norm{\wv}_1\} = \begin{cases}
            0 & \text{ if } \norm{\uv}_\infty \leq \lambda, \\
            \infty & \text{ otherwise. }
        \end{cases}
    \end{gathered}
\end{equation*}
where $ \thetav^\top\Xv\betav - \uv^\top\Dv\betav $ is unbounded unless
$\Xv^\top \dualv = \Dv^\top \uv$.
Finally, we substitute the above results into our Lagrangian dual problem to obtain:
\begin{equation*}
\begin{gathered}
    \min_{\dualv, \uv} \frac{1}{2} \norm{\dualv}_2^2 - \dualv^\top \yv, \\
    \text{subject to: }  \Dv^\top \uv = \Xv^\top \dualv \text{ and }
    \norm{\uv}_\infty \leq \lambda.
\end{gathered}
\end{equation*}
which is equivalent to the stated dual problem.

\section{Proof of Nuclear Norm ALO Formula}
\label{append:sec:nuclear-norm}

In this section, we prove Theorem \ref{thm:matrix-nonsmooth-alo}.
We consider the following matrix sensing formulation
\begin{equation*}\label{eq:append:matrix-smooth-main}
    \estim{\Bv} = \arg\min_{\Bv} \sum_{j=1}^n \ell(\langle \Xv_j,
    \Bv\rangle; y_j)^2 + \lambda R(\Bv).
\end{equation*}
where $R$ is a unitarily invariant function, which will be
explained and studied in more detail in Section
\ref{append:ssec:unitary-matrix-function}. This section is laid out as follows: in
Section \ref{append:ssec:unitary-matrix-function}, we briefly discuss basic properties
of unitarily invariant functions; In Section \ref{append:ssec:smooth-unitary} we do
ALO for smooth unitarily invariant penalties; In Section
\ref{append:ssec:nuclear-proof} we prove Theorem \ref{thm:matrix-nonsmooth-alo}
where nuclear norm is considered.

\subsection{Properties of Unitarily Invariant Functions}
\label{append:ssec:unitary-matrix-function}

Let $\Bv \in \RR^{p_1 \times p_2}$, and consider the SVD of $\Bv$ as $\Bv = \Uv
\diag[\sigmav] \Vv^\top$ with $\Uv \in \mathbb{R}^{p_1 \times p_1}$, $\Vv \in
\mathbb{R}^{p_2 \times p_2}$. We say that a function $R : \RR^{p_1 \times
p_2} \rightarrow \RR$ is unitarily invariant if there exists an absolutely
symmetric function $f : \RR^{\min(p_1, p_2)} \rightarrow \RR$ such that:
\begin{equation*}
    R(\Bv) = f(\sigmav),
\end{equation*}
where we say that $f: \RR^q \rightarrow \RR$ is absolutely symmetric if for
any $\xv \in \RR^q$, any permutation $\tau$ and signs $\bm{\epsilon} \in \{ -1, 1 \}^q$ we have:
\begin{equation*}
    f(x_1, \dotsc, x_q) = f(\epsilon_1 x_{\tau(1)}, \dotsc, \epsilon_q x_{\tau(q)}).
\end{equation*}

The properties of $R$ and $f$ are closely related, and in particular we will make
use of the following lemma relating their convexity, smoothness and derivatives, proved in
\cite{lewis1995convex}.

\begin{lemma}[\cite{lewis1995convex}]\label{lemma:unitary1}
    Let $R(\Bv) = f(\sigmav)$ with $\Bv=\Uv\diag[\sigmav] \Vv^\top$ its SVD.
    There is an one-to-one correspondence between unitarily
    invariant matrix functions $R$ and symmetric functions $f$. Furthermore the
    convexity and/or differentiability of $f$ are equivalent to the convexity
    and/or
    differentiability of $R$ respectively. If $R$ is differentiable, its
    derivative is given by:
    \begin{equation}\label{eq:matrix-differential}
        \nabla R(\Bv) = \Uv \diag[\nabla f(\sigmav)]\Vv^\top.
    \end{equation}

    When $f$ is not differentiable, a similar result holds with gradient
    replaced by subdifferentials.
    \begin{equation}\label{eq:matrix-subdifferential}
        \partial R(\Bv) = \Uv \diag[\partial f(\sigmav)]\Vv^\top.
    \end{equation}
\end{lemma}

Based on this lemma, we know that as long as $f$ is convex and/or
smooth, the corresponding matrix function will be convex and/or smooth. This enables
us to produce convex and smooth unitarily invariant approximation to non-smooth
unitarily invariant matrix regularizers.

In addition to the gradient of the unitarily invariant matrix functions, we
also need their Hessians. We show this result in the following Theorem
\ref{thm:matrix-twice-diff} for a sub-class of unitarily invariant functions.
\begin{theorem}\label{thm:matrix-twice-diff}
    Consider a unitarily invariant function with form $R(\Bv) =
    \sum_{j=1}^{ \min(p_1, p_2)} f(\sigma_j)$, where $f$ is a
    smooth function on $\mathbb{R}$ and $\Bv=\Uv\diag[\sigmav]\Vv^\top$ is its
    SVD with $\Uv \in \mathbb{R}^{p_1 \times p_1}$, $\Vv \in \mathbb{R}^{p_2
    \times p_2}$. Further assume that all the $\sigma_j$'s
    are different from each other and nonzero. Let $p_3 = \min(p_1, p_2)$, $p_4
    = \max(p_1, p_2)$. Then the Hessian matrix $\nabla^2 R(\Bv) \in
    \mathbb{R}^{p_1p_2 \times p_1p_2}$ takes the following form
    \begin{equation}\label{eq:matrix-hessian-general}
        \nabla^2 R(\Bv)
        =
        \Qv
        \Bigg[\begin{array}{ccc}
                A_1 & 0 & 0 \\
                0 & A_2 & 0 \\
                0 & 0 & A_3
        \end{array}\Bigg]
        \Qv^\top,
    \end{equation}
    where the first block $A_1 \in \mathbb{R}^{p_3 \times p_3}$. $A_1$ is
    diagonal with $A_{1, (ss, ss)} = f''(\sigma_s)$, $1\leq s\leq p_3$. The
    second block $A_2 \in \mathbb{R}^{p_3(p_3 - 1) \times p_3(p_3 - 1)}$. For
    $1 \leq s\neq t \leq p_3$, $A_{2,(st, st)} = A_{2,(ts, ts)} =
    \frac{\sigma_s f'(\sigma_s) - \sigma_t f'(\sigma_t)}{\sigma_s^2 -
    \sigma_t^2}$, $A_{2, (st, ts)} = A_{2, (ts, st)} = - \frac{\sigma_s
    f'(\sigma_t) - \sigma_t f'(\sigma_s)}{\sigma_s^2 - \sigma_t^2}$; The third
    block $A_3 \in \mathbb{R}^{(p_4 - p_3)p_3 \times (p_4 - p_3)p_3}$; $A_{3,
    (st, st)} = \frac{f'(\sigma_t)}{\sigma_t}$ for $1 \leq t \leq p_3 < s \leq
    p_4$. Except for these specified locations, all other components of $A_1,
    A_2, A_3$ are zero. $\Qv$ is an orthogonal matrix with $\Qv_{\cdot, st} =
    \vecop(\uv_s\vv_t^\top)$ where $\uv_s$, $\vv_t$ are the $s$\tsup{th}
    column of $\Uv$ and $t$\tsup{th} column of $\Vv$ respectively. $\vecop(\cdot)$
    denotes the vectorization operator, which aligns all the components of a
    matrix into a long vector.
\end{theorem}
\begin{remark}
    Since here we are talking about the Hessian matrix of functions on matrix
    space, we linearize these matrices and treat them as vectors. It would be
    helpful if we visualize the correspondence between each blocks in
    \eqref{eq:matrix-hessian-general} and the component indices in the original
    matrix $\Bv$. Specifically we have Figure \ref{fig:matrix-general}.
    \begin{figure}[!htb]
        \begin{center}
            \begin{tikzpicture}[scale=1.0]
                \path [fill=cyan] (0, 0) rectangle (2, 1);
                \path [fill=green] (0, 1) rectangle (2, 3);
                \path [fill=orange] (0, 3) -- (0.1, 3) -- (2, 1.1) -- (2, 1) --
                (1.9, 1) -- (0, 2.9) -- (0, 3);
                \draw (0, 0) rectangle (2, 3);
                \draw [dashed, thick] (0, 1) -- (2, 1);
                \draw [dashed, thick] (0, 3) -- (2, 1);
                
                \draw [fill] (0.75, 2.25) circle [radius=0.04];
                \node [right] at (0.75, 2.25) {\tiny $(s_1, s_1)$};
                \draw [fill] (0.3, 2) circle [radius=0.04];
                \node [below] at (0.4, 2) {\tiny $(s_2, t_2)$};
                \draw [fill] (1, 2.7) circle [radius=0.04];
                \node [right] at (1, 2.7) {\tiny $(t_2, s_2)$};
                \draw [fill] (1.5, 0.5) circle [radius=0.04];
                \node [left] at (1.5, 0.5) {\tiny $(s_3, t_3)$};
                
                \draw [ultra thick, ->] (2.5, 1.5) -- (3.8, 1.5);
                
                \path [fill=lightgray] (5, -0.5) rectangle (9, 3.5);
                \path [fill=cyan] (8, 0.5) rectangle (9, -0.5);
                \path [fill=green] (6, 2.5) rectangle (8, 0.5);
                \path [fill=orange] (5, 3.5) rectangle (6, 2.5);
                \draw [dotted, thick] (5, 2.5) -- (8, 2.5);
                \draw [dotted, thick] (6, 0.5) -- (9, 0.5);
                \draw [dotted, thick] (6, 0.5) -- (6, 3.5);
                \draw [dotted, thick] (8, -0.5) -- (8, 2.5);
                \draw [dashed, thick] (5, 3.5) -- (9, -0.5);
                \draw (5, -0.5) rectangle (9, 3.5);
                
                \draw [fill] (5.3, 3.2) circle [radius=0.04];
                \draw [dotted, thick] (5, 3.2) -- (5.3, 3.2);
                \draw [dotted, thick] (5.3, 3.5) -- (5.3, 3.2);
                \node [left] at (5, 3.2) {\tiny $(s_1, s_1)$};
                \node [right, rotate=90] at (5.3, 3.5) {\tiny $(s_1, s_1)$};
                
                \draw [fill] (6.5, 2.0) circle [radius=0.04];
                \draw [fill] (7.4, 1.1) circle [radius=0.04];
                \draw [fill] (7.4, 2.0) circle [radius=0.04];
                \draw [fill] (6.5, 1.1) circle [radius=0.04];
                \draw [dotted, thick] (5, 2.0) -- (7.4, 2.0);
                \draw [dotted, thick] (5, 1.1) -- (7.4, 1.1);
                \draw [dotted, thick] (6.5, 3.5) -- (6.5, 1.1);
                \draw [dotted, thick] (7.4, 3.5) -- (7.4, 1.1);
                \node [left] at (5, 2.0) {\tiny $(s_2, t_2)$};
                \node [right, rotate=90] at (6.5, 3.5) {\tiny $(s_2, t_2)$};
                \node [left] at (5, 1.1) {\tiny $(t_2, s_2)$};
                \node [right, rotate=90] at (7.4, 3.5) {\tiny $(t_2, s_2)$};
                
                \draw [fill] (8.6, -0.1) circle [radius=0.04];
                \draw [dotted, thick] (5, -0.1) -- (8.6, -0.1);
                \draw [dotted, thick] (8.6, 3.5) -- (8.6, -0.1);
                \node [left] at (5, -0.1) {\tiny $(s_3, t_3)$};
                \node [right, rotate=90] at (8.6, 3.5) {\tiny $(s_3, t_3)$};
                
                \node [below] at (5.3, 3.2) {\scriptsize $a$};
                \node [below left] at (6.5, 2.0) {\scriptsize $b$};
                \node [right] at (7.4, 1.1) {\scriptsize $b$};
                \node [right] at (7.4, 2.0) {\scriptsize $c$};
                \node [below] at (6.5, 1.1) {\scriptsize $c$};
                \node [below] at (8.6, -0.1) {\scriptsize $d$};

                \draw [thin, ->] (9.6, 3) -- (6.1, 3);
                \draw [thin, ->] (9.6, 1.5) -- (8.1, 1.5);
                \draw [thin, ->] (9.6, 0) -- (9.1, 0);
                \node [right] at (9.6, 3) {$A_1$};
                \node [right] at (9.6, 1.5) {$A_2$};
                \node [right] at (9.6, 0) {$A_3$};

                \node [below] at (1, -0.5) {$\Uv^\top\Bv\Vv=\diag[\sigmav]$};
                \node [below] at (7, -0.5) {$\Qv^\top\nabla^2 R(\Bv)\Qv$};
            \end{tikzpicture}
        \end{center}
        \caption{An illustration of the correspondence between the structure of
            the original matrix and the structure of the Hessian matrix of $R$.
            As we have mentioned in Theorem \ref{thm:matrix-twice-diff},
            $a=f''(\sigma_{s_1})$,
            $b = \frac{\sigma_{s_2} f'(\sigma_{s_2}) - \sigma_{t_2}
            f'(\sigma_{t_2}) } {\sigma_{s_2}^2 - \sigma_{t_2}^2 }$,
            $c = - \frac{\sigma_{s_2} f'(\sigma_{t_2}) - \sigma_{t_2}
            f'(\sigma_{s_2}) } {\sigma_{s_2}^2 - \sigma_{t_2}^2}$; $d =
            \frac{f'(\sigma_{t_3})}{\sigma_{t_3}}$.} \label{fig:matrix-general}
    \end{figure}
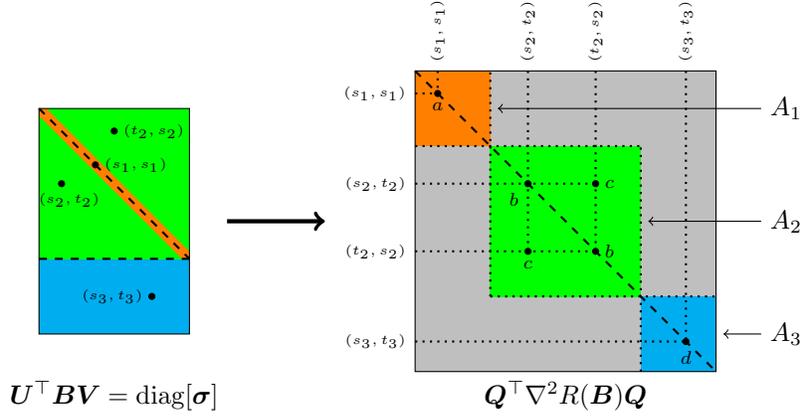
\end{remark}

\begin{proof}
    First by Lemma \ref{lemma:unitary1}, the gradient $\nabla R(\Bv)$ takes the
    following form
    \begin{equation*}
        \nabla R(\Bv) = \Uv \diag[\{f'(\sigma_j)\}_{j}] \Vv^\top.
    \end{equation*}
    
    In order to find the differential of $\nabla R(\Bv)$, we use the similar
    techniques and notations described in Lemma IV.2 and Theorem IV.3 in
    \cite{candes2013unbiased}. To simplify our derivation, we assume $p_1 \geq
    p_2$. This does not affect the correctness of our final conclusion.

    We characterize the differential of the gradient as a linear form.
    Specifically, along a certain direction $\Deltav \in \mathbb{R}^{p_1 \times
    p_2}$, by Lemma IV.2 in
    \cite{candes2013unbiased}, we have
    \begin{equation}\label{eq:matrix-svd-differential}
        d\Uv[\Deltav] = \Uv\Omegav_{\Uv}[\Deltav]
        ,\quad
        d\Vv[\Deltav] = \Vv\Omegav_{\Vv}[\Deltav]^{\top}
        ,\quad
        d\sigma_s[\Deltav] = [\Uv^\top\Deltav\Vv]_{ss}.
    \end{equation}
    where $\Omegav_{\Uv}$ and $\Omegav_{\Vv}$ are assymmetric matrices (thus
    their diagonal values are 0) which can be found by solving the following
    equation systems:
    \begin{equation}\label{eq:matrix-svd-differential2}
        \left[
            \begin{array}{c}
                \Omegav_{\Uv, st}[\Delta] \\
                \Omegav_{\Vv, st}[\Delta]
            \end{array}
        \right]
        =
        -\frac{1}{\sigma_s^2 - \sigma_t^2}
        \left[
            \begin{array}{cc}
                \sigma_t & \sigma_s \\
                - \sigma_s & - \sigma_t
            \end{array}
        \right]
        \left[
            \begin{array}{c}
                (\Uv^\top \Deltav \Vv)_{st} \\
                (\Uv^\top \Deltav \Vv)_{ts}
            \end{array}
        \right]
        , \quad
        \text{if } s \neq t, s \leq p_2,
    \end{equation}
    and
    \begin{equation}\label{eq:matrix-svd-differential3}
        \Omegav_{\Uv, st}[\Delta]
        =
        \frac{(\Uv^\top\Deltav\Vv)_{st}}{\sigma_t}
        , \quad
        \text{if } s \neq t, s > p_2.
    \end{equation}

    The differential of $\nabla R(\Bv)$ along a certian direction $\Deltav$ can
    then be calculated through the chain rule as that
    \begin{align}\label{eq:matrix-reg-differential}
        &d\nabla R(\Bv)[\Deltav] \nonumber \\
        =&
        d\Uv[\Deltav]\diag[\{f'(\sigma_j)\}_{j}]\Vv^\top
        +
        \Uv\diag[\{f''(\sigma_j)d\sigma_j[\Deltav]\}_{j}]\Vv^\top
        +
        \Uv\diag[\{f'(\sigma_j)\}_{j}]d\Vv[\Deltav]^\top \nonumber \\
        =&
        \Uv\big(\Omegav_{\Uv}[\Deltav]\diag[\{f'(\sigma_j)\}_{j}]
        +
        \diag[\{f''(\sigma_j)d\sigma_j[\Deltav]\}_{j}]
        +
        \diag[\{f'(\sigma_j)\}_{j}]\Omegav_{\Vv}[\Deltav] \big)\Vv^\top.
    \end{align}

    In the original formula obtained from the primal approach, the Hessian is
    calculated under the canonical bases
    \footnote{$\Ev_{st}$ is defined as a $p_1 \times p_2$ matrix with all of
    its components being 0 except the $(s, t)$ location being 1.}
    $\{\Ev_{st}\}_{s, t}$. In order to
    simplify the calculation of the Hessian, we instead use the orthonormal bases
    $\{\uv_s\vv_t^\top\}_{s,t}$, and then transform back to
    $\{\Ev_{st}\}_{s, t}$.

    The $(kl, st)$ location of the Hessian matrix under $\{\uv_s\vv_t\}_{s, t}$
    bases can be calculated by
    \begin{equation}\label{eq:hess-component}
        \langle \uv_k\vv_l^\top, d\nabla R(\Bv)[\uv_s\vv_t^\top] \rangle.
    \end{equation}

    Pluggin equation \eqref{eq:matrix-reg-differential}
    into \eqref{eq:hess-component} we obtain that
    \begin{align*}
        & \langle \uv_k\vv_l, d\nabla R(\Bv)[\uv_s\vv_t^\top] \rangle \nonumber
        \\
        =&
        \langle \Ev_{kl},
        \Omegav_{\Uv}[\uv_s\vv_t^\top]\diag[\{f'(\sigma_j)\}_j]
        +
        \diag[\{f''(\sigma_j)d\sigma_j[\uv_s\vv_t^\top]\}_{j}]
        +
        \diag[\{f'(\sigma_j)\}_j]\Omegav_{\Vv}[\uv_s\vv_t^\top] \rangle \nonumber \\
        =&
        \left\{
            \begin{array}{ll}
                f''(\sigma_t)d\sigma_t[\uv_t\vv_t^\top], & s = t = k = l,
                \\
                \Omegav_{\Uv, kl}[\uv_s\vv_t^\top]f'(\sigma_l)
                +
                f'(\sigma_k)\Omegav_{\Vv, kl}[\uv_s\vv_t^\top], & k \neq l, k
                \leq p_2, \\
                \Omegav_{\Uv, kl}[\uv_s\vv_t^\top]f'(\sigma_l), & 1\leq l \leq
                p_2 < k \leq p_1.
            \end{array}
        \right.
    \end{align*}

    By \eqref{eq:matrix-svd-differential}, we have $d\sigma_j[\uv_s\vv_t^\top]
    = [\Ev_{st}]_{jj} = \delta_{sj}\delta_{tj}$. In addition, $(\Uv^\top
    \uv_s\vv_t^\top\Vv^\top)_{kl} = (\Ev_{st})_{kl} = \delta_{sk}\delta_{tl}$,
    $(\Uv^\top \uv_s\vv_t^\top\Vv^\top)_{lk} = (\Ev_{st})_{lk} =
    \delta_{sl}\delta_{tk}$. Hence by \eqref{eq:matrix-svd-differential2} and
    \eqref{eq:matrix-svd-differential3}, we have that
    \begin{equation*}
        \Omegav_{\Uv, kl}[\uv_s\vv_t^\top]
        =
        - \frac{\delta_{sk}\delta_{tl}\sigma_l +
        \delta_{sl}\delta_{tk}\sigma_k}{\sigma_k^2 - \sigma_l^2}
        ,\quad
        \Omegav_{\Vv, kl}[\uv_s\vv_t^\top]
        =
        \frac{\delta_{sk}\delta_{tl}\sigma_k +
        \delta_{sl}\delta_{tk}\sigma_l}{\sigma_k^2 - \sigma_l^2}
        , \quad
        \text{if } s \neq t, s \leq p_2,
    \end{equation*}
    and
    \begin{equation*}
        \Omegav_{\Uv, kl}[\uv_s\vv_t^\top]
        =
        \frac{\delta_{sk}\delta_{tl}}{\sigma_l}
        , \quad
        \text{if } s \neq t, s > p_2.
    \end{equation*}

    Based on all these, we can obtain that
    \begin{equation*}
        \langle \uv_k\vv_l, d\nabla R(\Bv)[\uv_s\vv_t^\top] \rangle
        =
        \left\{
            \begin{array}{ll}
                f''(\sigma_t), & s = t = k = l, \\
                \frac{\sigma_s f'(\sigma_s) - \sigma_tf'(\sigma_t)}{\sigma_s^2
                - \sigma_t^2}, & s \neq t, s \leq p_2, (k, l) = (s, t), \\
                - \frac{\sigma_s f'(\sigma_t) - \sigma_tf'(\sigma_s)}{\sigma_s^2
                - \sigma_t^2}, & s \neq t, s \leq p_2, (k, l) = (t, s), \\
                \frac{f'(\sigma_t)}{\sigma_t}, & s \neq j, s > p_2,
                    (k,l)=(s,t), \\
                0, & \text{otherwise.}
            \end{array}
        \right.
    \end{equation*}

    Notice that we obtained the above expressions under the orthonormal bases
    $\{\uv_s\vv_t^\top\}_{s, t}$. In order to get the Hessian form under the
    canonical bases $\{\Ev_{st}\}_{s, t}$, let $\Qv \in \mathbb{R}^{p_1p_2 \times
    p_1p_2}$, with each column $\Qv_{\cdot, st} = \vecop(\uv_s\vv_t^\top)$.
    Denote the matrix form under the canonical bases by $\nabla^2 R(\Bv)$ and that
    under $\{\uv_s\vv_t^\top\}_{s,t}$ by $\widetilde{\nabla^2 R(\Bv)}$. We then have
    that
    \begin{equation*}
        \nabla^2 R(\Bv)
        =
        \Qv \widetilde{\nabla^2 R(\Bv)} \Qv^\top.
    \end{equation*}
    
    This completes our proof.
\end{proof}

\subsection{ALO for Smooth Unitarily Invariant Penalties}
\label{append:ssec:smooth-unitary}

In this following two sections, we discuss ALO formula for unitarily invariant
regularizer $R$ of the form:
\begin{equation*}\label{eq:separable-matrix-norm}
    R(\Bv) = \sum_{j=1}^{\min(p_1, p_2)} r(\sigma_j),
\end{equation*}
where $r$ is a convex and even scalar function. The nuclear norm, Frobenius and numerous
other matrix norms all fall in this category. For this section, we consider $r$
as a smooth function. In the next section, we consider
the case of the nuclear norm where $r$ is nonsmooth.

Consider the matrix regression problem:
\begin{equation*}
    \estim{\Bv}
    =
    \arg\min_\Bv \sum_{j=1}^n \ell(\langle \Xv_j, \Bv \rangle; y_j) + \lambda
    R(\Bv).
\end{equation*}

Let $\estim{\Bv}=\estim{\Uv}\diag[\estim{\sigmav}]\estim{\Vv}^\top$.
By pluggin the Hessian form we obtained in Theorem \ref{thm:matrix-twice-diff}
into \eqref{eq:primal-alo-smooth}, \eqref{eq:smooth-H},
we have the following ALO formula
\begin{equation}\label{eq:matrix-general-alo}
    \langle \Xv_i, \surrogi{\Bv} \rangle
    =
    \langle \Xv_i, \estim{\Bv} \rangle
    +
    \frac{H_{ii}}{1 - H_{ii}\ddot{\ell}(\langle \Xv_i, \estim{\Bv}
    \rangle; y_i)}\dot{\ell}(\langle \Xv_i, \estim{\Bv}\rangle; y_i).
\end{equation}
where
\begin{equation*}
    \Hv
    :=
    \cb{\tilde{X}}\Big[\cb{\tilde{X}}^\top \diag[\ddot{\ell}(\langle \Xv_j,
    \estim{\Bv}\rangle; y_j)]\cb{\tilde{X}} + \lambda \Qv \cb{G} \Qv^\top\Big]^{-1}
    \cb{\tilde{X}}^\top,
\end{equation*}
with the matrix $\cb{\tilde{X}} \in \mathbb{R}^{n \times p_1p_2}$, $\cb{G} \in
\mathbb{R}^{p_1p_2 \times p_1p_2}$. Each row $\cb{\tilde{X}}_{j,\cdot} =
\vecop(\Xv_j)$. $\cb{G}$ is defined by
\begin{equation}\label{eq:matrix-reg-formula-G}
    \cb{G}_{kl, st}
    =
    \left\{
        \begin{array}{ll}
            r''(\estim{\sigma}_t), & s = t = k = l, \\
            \frac{\estim{\sigma}_s r'(\estim{\sigma}_s) - \estim{\sigma}_tr'(\estim{\sigma}_t)}
            {\estim{\sigma}_s^2 - \estim{\sigma}_t^2},
            & i \neq t, s \leq p_2, (k, l) = (s, t), \\
            - \frac{\estim{\sigma}_s r'(\estim{\sigma}_t) - \estim{\sigma}_tr'(\estim{\sigma}_s)}
            {\estim{\sigma}_s^2 - \estim{\sigma}_t^2},
            & s \neq t, s \leq p_2, (k, l) = (t, s), \\
            \frac{r'(\estim{\sigma}_t)}{\estim{\sigma}_t},
            & s \neq t, s > p_2, (k,l)=(s,t), \\
            0, & \text{otherwise.}
        \end{array}
    \right.
\end{equation}

Notice that $[\cb{\tilde{X}}\Qv]_{j, st} = \langle \Xv_j,
\estim{\uv}_s\estim{\vv}_t^\top \rangle = \estim{\uv}_s^\top \Xv_j
\estim{\vv}_t$, we have $[\cb{\tilde{X}}\Qv]_{j,\cdot} =
\vecop(\estim{\Uv}^\top\Xv_j\estim{\Vv})$. Let $\cb{X}=\cb{\tilde{X}}\Qv$. This
gives us the following nicer form of the $\Hv$ matrix:
\begin{equation*}
    \Hv
    :=
    \cb{X}\Big[\cb{X}^\top \diag[\ddot{\ell}(\langle \Xv_j,
    \estim{\Bv}\rangle; y_j)]\cb{X} + \lambda \cb{G} \Big]^{-1}
    \cb{X}^\top.
\end{equation*}

\subsection{Proof of Theorem \ref{thm:matrix-nonsmooth-alo}:
ALO for Nuclear Norm}
\label{append:ssec:nuclear-proof}

For the nuclear norm, we have:
\begin{equation*}
    \ell(u; y) = \frac{1}{2}(u - y)^2
    ,\quad
    R(\Bv) = \sum_{j=1}^{\min(p_1, p_2)} \sigma_j.
\end{equation*}

Let $P(\Bv) = \frac{1}{2}\sum_{j=1}^n (y_j - \langle \Xv_j, \Bv \rangle)^2 +
\lambda\|\Bv\|_*$ denote the primal objective. For the full data optimizer
$\estim{\Bv}$ with SVD $\estim{\Bv} = \estim{\Uv} \diag[\estim{\sigmav}]
\estim{\Vv}$, let $m=\rank(\estim{\Bv})$, the number of nonzero
$\estim{\sigma}_j$'s. Furthermore, suppose that we have
the following assumption on the full data solution $\estim{\Bv}$.

\begin{assumption}
    Let $\estim{\Bv}$ be the full-data minimizer, and let
    $\estim{\Bv} = \estim{\Uv} \diag[\estim{\sigmav}] \estim{\Vv}^\top$
    be its SVD.
    \begin{enumerate}
    \item
        $\estim{\Bv}$ is the unique optimizer of the nuclear norm minimization
        problem,

    \item
        For all $j$ such that $\estim{\sigma}_j = 0$, the subgradient
        $g_r[\estim{\sigma}_j]$ at $\estim{\sigma}_j$ satisfies
        $g_r[\estim{\sigma}_j] < 1$.
    \end{enumerate}
\end{assumption}

Since the nuclear norm is nonsmooth, we consider a smoothed version of it. For a matrix and
its SVD $\Bv = \Uv\diag[\sigmav]\Vv^\top$, and a smoothing parameter $\epsilon > 0$,
 define the following smoothed version of nuclear norm as
\begin{equation*}
    R_\epsilon(\Bv) = \sum_{j=1}^{\min(p_1, p_2)} r_{\epsilon}(\sigma_j), \text{ where }
    r_{\epsilon}(x) = \sqrt{x^2 + \epsilon^2}.
\end{equation*}

Let $P_\epsilon(\Bv) = \frac{1}{2}\sum_{j=1}^n (y_j - \langle \Xv_j,
\Bv \rangle)^2 + \lambda R_\epsilon(\Bv)$ denote the smoothed primal objective, and
let $\estim{\Bv}_\epsilon$ be the minimizer of $P_\epsilon$.
Note that instead of using the general kernel smoothing strategy we mentioned
in the previous section, in this specific case we consider this
choice $R_\epsilon$ for technical convenience. There are no essential
differences between the two smoothing schemes. Finally, let $r(x) = \abs{x}$

Lemma \ref{lemma:unitary1} guarantees the smoothness and convexity of the function
$R_\epsilon$. Additionally, $r_\epsilon$ verifies several desirable properties:
\begin{enumerate}
    \item
        $\dot{r}_\epsilon(x) = \frac{x}{\sqrt{x^2 + \epsilon^2}}$,
        $\ddot{r}_\epsilon(x) = \frac{\epsilon^2}{(x^2 +
        \epsilon^2)^{\frac{3}{2}}}$;
    \item
        $r(x) < r_\epsilon(x) < r(x) + \epsilon$. 
\end{enumerate}

In particular, we note that the second property implies that
$\sup_x |r(x) - r_\epsilon(x)| \leq \epsilon$ and that
$\sup_{\Bv} |R(\Bv) - R_\epsilon(\Bv)| \leq \epsilon \min(p_1, p_2)$.

We now go through a similar strategy as in Appendix
\ref{append:ssec:primal-nonsmooth-reg} to consider the limit case as $\epsilon
\rightarrow 0$.

\paragraph{Convergence of the optimizer ($\estim{\Bv}_\epsilon \rightarrow \estim{\Bv}$)}
By definition of $\estim{\Bv}$ as the minimizer of the primal objective, we have that:
\begin{equation*}
\lambda \|\estim{\Bv}\|_*
\leq \frac{1}{2}\sum_j (y_j - \langle \Xv_j, \estim{\Bv} \rangle)^2 + \lambda\|\estim{\Bv}\|_*
\leq \frac{1}{2}\|\yv\|_2^2.
\end{equation*}

Similarly, we have that $\estim{\Bv}_\epsilon$ verifies:
\begin{align*}
    \lambda\norm{\estim{\Bv}_\epsilon}_*
    &\leq \lambda R(\estim{\Bv}_\epsilon)
    \leq \lambda R_\epsilon(\estim{\Bv}_\epsilon) + \lambda\epsilon\min(p_1,
    p_2)\nonumber \\
&\leq \frac{1}{2}\sum_j (y_j - \langle \Xv_j, \estim{\Bv}_\epsilon \rangle)^2
    + \lambda R_\epsilon(\estim{\Bv}_\epsilon) + \lambda \epsilon \min(p_1, p_2) \nonumber \\
&\leq \frac{1}{2}\|\yv\|_2^2 + \lambda \epsilon \min(p_1, p_2).
\end{align*}

Thus, for all $\epsilon \leq 1$ both $\estim{\Bv}$ and $\estim{\Bv}_\epsilon$ are contained
in a compact set given by $\lambda\|\Bv\|_* \leq \frac{1}{2}\|\yv\|_2^2 +
\lambda\min(p_1, p_2)$.

In particular, any subsequence of $\estim{\Bv}_\epsilon$ contains a convergent
sub-subsequence, let us abuse notations and still use $\estim{\Bv}_\epsilon$ for
this convergent sub-subsequence. The uniform bound between $R$ and $R_\epsilon$
implies that:
\begin{equation*}
    P(\lim_{\epsilon\rightarrow 0}\estim{\Bv}_\epsilon)
    =\lim_{\epsilon\rightarrow 0}P(\estim{\Bv}_\epsilon)
    =\lim_{\epsilon \rightarrow 0} P_\epsilon(\estim{\Bv}_\epsilon)
    \leq \lim_{\epsilon \rightarrow 0} P_\epsilon(\estim{\Bv})
    = P(\estim{\Bv}).
\end{equation*}

By the uniqueness of the optimizer $\estim{\Bv}$, we have
\begin{equation*}
    \lim_{\epsilon \rightarrow 0} \estim{\Bv}_\epsilon = \estim{\Bv}.
\end{equation*}

This is true for all such subsequences, which confirms what we want to prove.

\paragraph{Convergence of the gradient ($\nabla R_\epsilon(\estim{\Bv}_\epsilon) \rightarrow g_{\|\cdot\|_*}
(\estim{\Bv})$)}
Let $g_{\|\cdot\|_*}$ denote the subgradient of the nuclear norm $\|\cdot\|_*$ in
the first order optimality condition of $\estim{\Bv}$.
By the continuity of $\dot{\ell}$ and the first order condition, we have:
\begin{equation}\label{eq:nuclear:frobenius-bound}
    \big\| g_{\|\cdot\|_*}(\estim{\Bv}) - \nabla
    R_\epsilon(\estim{\Bv}_{\epsilon}) \big\|_F
    =
    \Bigg\| \sum_{j=1}^n \langle \Xv_j, \estim{\Bv} - \estim{\Bv}_\epsilon
    \rangle \Xv_j \Bigg\|_F \rightarrow 0.
\end{equation}

Let $\estim{\Bv}_\epsilon = \estim{\Uv}_\epsilon \diag[\estim{\sigmav}_\epsilon] \estim{\Vv}_\epsilon$
denote the SVD of $\estim{\Bv}_\epsilon$. By Lemma \ref{lemma:unitary1} we have:
\begin{align*}
    g_{\|\cdot\|_*} (\estim{\Bv})
    &=
    \estim{\Uv} \diag(\{g_r[\estim{\sigma}_j]\}_{j}) \estim{\Vv}^\top, \nonumber \\
    \nabla R_\epsilon(\estim{\Bv}^{\epsilon}) &= \estim{\Uv}_\epsilon
    \diag(\{\dot{r}_\epsilon(\estim{\sigma}_{\epsilon,j})\}_{j})
    \estim{\Vv}_\epsilon^\top.
\end{align*}
where $g_r[x] = 1$ if $x > 0$ and $0 \leq g_r[x] \leq 1$ if $x = 0$.

We wish to translate the limit in matrix norm
\eqref{eq:nuclear:frobenius-bound} to a limit on their singular
values. In order to do this, we use the following lemma from Weyl
\cite{weyl1912das} or Mirsky \cite{mirsky1960symmetric}. We note that our
conclusion may follow from either, although we include both for completeness.

\begin{lemma}[\cite{weyl1912das},\cite{mirsky1960symmetric}]
    \label{lemma:sv-control}
    Let $A$ and $B$ be two rectangular matrices of the same shape.
    Let $\sigma_j$ denote the $j$\textsuperscript{th} largest eigenvalue, then we have that
    for all $j$:
    \begin{gather*}
        \abs{\sigma_j(A) - \sigma_j(B)} \leq \norm{A - B}_2, \nonumber \\
        \sqrt{\sum_j (\sigma_j(A) - \sigma_j(B))^2} \leq \norm{A - B}_F.
    \end{gather*}
\end{lemma}

By Lemma \ref{lemma:sv-control}, we have that $\estim{\sigma}_{\epsilon, j} \rightarrow
\estim{\sigma}_j$ and $\frac{\estim{\sigma}_{\epsilon,j}}{\sqrt{\estim{\sigma}_{\epsilon, j}^2 +
\epsilon^2}} \rightarrow g_r[\estim{\sigma}_j]$ as $\epsilon \rightarrow 0$. Additionally,
by the assumption $g_r[\estim{\sigma}_j] < 1$ if $\estim{\sigma}_j=0$, we have that:
\begin{equation}\label{eq:nuclear-singular-asymp-behavior}
    \frac{\estim{\sigma}_{\epsilon, j}}{\epsilon}
    \rightarrow
    \begin{cases}
        +\infty, & \text{if } \estim{\sigma}_j > 0, \\
        < +\infty, & \text{if } \estim{\sigma}_j = 0. \\
    \end{cases}
\end{equation}

This further implies the matrices $\cb{G}_{\epsilon}$
defined as in \eqref{eq:matrix-reg-formula-G} for $R_\epsilon$ satisifies:
\begin{equation}\label{eq:nuclear-reg-formula-G}
    \lim_{\epsilon \rightarrow 0}
    \cb{G}_{\epsilon, kl, ij}
    =
    \left\{
        \begin{array}{ll}
            0, & s = t = k = l \leq m, \\
            \infty, & s = t = k = l > m, \\
            \frac{1}{\estim{\sigma}_s + \estim{\sigma}_t}, & 1 \leq s \neq t \leq m,(k,l)=(s,t), \\
            \frac{1}{\estim{\sigma}_s}, & 1 \leq s \leq m < t \leq p_2, (k, l) = (s, t), \\
            \frac{1}{\estim{\sigma}_t}, & 1 \leq t \leq m < s \leq p_2, (k, l) = (s, t), \\
            -\frac{1}{\estim{\sigma}_s + \estim{\sigma}_t}, & 1 \leq s \neq t \leq m, (k,l)=(t,s), \\
            -\frac{g_r[\estim{\sigma}_t]}{\estim{\sigma}_s}, & 1 \leq s \leq m < t \leq p_2,
            (k, l) = (t, s), \\
            -\frac{g_r[\estim{\sigma}_s]}{\estim{\sigma}_t}, & 1 \leq t \leq m < s \leq p_2,
            (k, l) = (t, s), \\
            \frac{1}{\estim{\sigma}_t}, & 1 \leq t \leq m \leq p_2 < s \leq p_1, (k,l)=(s,t), \\
            \infty, & m < t \leq p_2 < s \leq p_1, (k,l)=(s,t), \\
            0, & \text{otherwise.}
        \end{array}
    \right.
\end{equation}

Notice that in \eqref{eq:nuclear-reg-formula-G} we missed a piece of blocks
corresponding to $m < s \neq t \leq p_2$, $(k, l) = (s, t)$ or $(k, l) =
(t, s)$. We need to process this blocks separately. We will show that the
inverse of the corresponding blocks in $\cb{G}_{\epsilon}$ converges to
0. As a result, we can ignore this part according to Lemma
\ref{lemma:woodbury-block}.

Each $2\times 2$ sub-matrix within that blocks in $\cb{G}_{\epsilon}$ takes the
form
\begin{equation*}
    \frac{1}{\estim{\sigma}_{\epsilon, s}^2 - \estim{\sigma}_{\epsilon, t}^2}
    \left[
        \begin{array}{cc}
            \estim{\sigma}_{\epsilon, s}\dot{r}_\epsilon(\estim{\sigma}_{\epsilon, s}) -
            \estim{\sigma}_{\epsilon, t} \dot{r}_\epsilon(\estim{\sigma}_{\epsilon, t})
            &
            - \estim{\sigma}_{\epsilon, s}\dot{r}_\epsilon(\estim{\sigma}_{\epsilon, t}) +
            \estim{\sigma}_{\epsilon, t} \dot{r}_\epsilon(\estim{\sigma}_{\epsilon, s}) \\
            - \estim{\sigma}_{\epsilon, s}\dot{r}_\epsilon(\estim{\sigma}_{\epsilon, t}) +
            \estim{\sigma}_{\epsilon, t} \dot{r}_\epsilon(\estim{\sigma}_{\epsilon, s})
            &
            \estim{\sigma}_{\epsilon, s}\dot{r}_\epsilon(\estim{\sigma}_{\epsilon, s}) -
            \estim{\sigma}_{\epsilon, t} \dot{r}_\epsilon(\estim{\sigma}_{\epsilon, t})
        \end{array}
    \right].
\end{equation*}

It is easy to verify that the inverse of the above matrix takes the following
form
\begin{equation}\label{eq:nuclear-inverse-sub22}
\frac{1}{\dot{r}^2(\estim{\sigma}_{\epsilon, s}) - \dot{r}^2(\estim{\sigma}_{\epsilon, t})}
    \left[
        \begin{array}{cc}
            \estim{\sigma}_{\epsilon, s}\dot{r}_\epsilon(\estim{\sigma}_{\epsilon, s}) -
            \estim{\sigma}_{\epsilon, t} \dot{r}_\epsilon(\estim{\sigma}_{\epsilon, t})
            &
            \estim{\sigma}_{\epsilon, s}\dot{r}_\epsilon(\estim{\sigma}_{\epsilon, t}) -
            \estim{\sigma}_{\epsilon, t} \dot{r}_\epsilon(\estim{\sigma}_{\epsilon, s}) \\
            \estim{\sigma}_{\epsilon, s}\dot{r}_\epsilon(\estim{\sigma}_{\epsilon, t}) -
            \estim{\sigma}_{\epsilon, t} \dot{r}_\epsilon(\estim{\sigma}_{\epsilon, s})
            &
            \estim{\sigma}_{\epsilon, s}\dot{r}_\epsilon(\estim{\sigma}_{\epsilon, s}) -
            \estim{\sigma}_{\epsilon, t} \dot{r}_\epsilon(\estim{\sigma}_{\epsilon, t})
        \end{array}
    \right].
\end{equation}

For the two distinct component values in the matrix in
\eqref{eq:nuclear-inverse-sub22}, we have that
\begin{equation*}
    \frac{\estim{\sigma}_{\epsilon, s}\dot{r}_\epsilon(\estim{\sigma}_{\epsilon, s}) -
    \estim{\sigma}_{\epsilon, t} \dot{r}_\epsilon(\estim{\sigma}_{\epsilon, t})}
    {\dot{r}^2(\estim{\sigma}_{\epsilon, s}) - \dot{r}^2(\estim{\sigma}_{\epsilon, t})}
    =
    \frac{
        \frac{\estim{\sigma}_{\epsilon, s}^2}{\sqrt{\estim{\sigma}_{\epsilon, s} + \epsilon^2}}
        -
        \frac{\estim{\sigma}_{\epsilon, t}^2}{\sqrt{\estim{\sigma}_{\epsilon, t} + \epsilon^2}}
    }
    {
        \frac{\estim{\sigma}_{\epsilon, s}^2}{\estim{\sigma}_{\epsilon, s} + \epsilon^2}
        -
        \frac{\estim{\sigma}_{\epsilon, t}^2}{\estim{\sigma}_{\epsilon, t} + \epsilon^2}
    }
    =
    \epsilon\frac{\frac{u_{\epsilon, s}}{\sqrt{1 - u_{\epsilon, s}}}
    - \frac{u_{\epsilon, t}}{\sqrt{1 - u_{\epsilon, t}}}}
    {u_{\epsilon, s} - u_{\epsilon, t}}
    =
    \epsilon \frac{1 - \frac{1}{2}\tilde{u}_{\epsilon}}{(1 -
    \tilde{u}_{\epsilon})^{\frac{3}{2}}}
    \rightarrow 0,
\end{equation*}
where we did a change of variable $u=\frac{\estim{\sigma}^2}{\estim{\sigma}^2 + \epsilon^2}$
and $\tilde{u}_{\epsilon}$ is a value between $u_{\epsilon, s}$ and
$u_{\epsilon, t}$ where we apply Taylor expansion to function
$\frac{x}{\sqrt{1-x}}$. The last convergence to 0 is
obtained by noticing that $\lim_{\epsilon \rightarrow 0} u_{\epsilon, s},
\lim_{\epsilon \rightarrow 0} u_{\epsilon, t} \in [0, 1)$ due to
\eqref{eq:nuclear-singular-asymp-behavior}.

Similarly we have the following analysis for the off-diagonal term
\begin{equation*}
    \frac{\estim{\sigma}_{\epsilon, s}\dot{r}_\epsilon(\estim{\sigma}_{\epsilon, t}) -
    \estim{\sigma}_{\epsilon, t} \dot{r}_\epsilon(\estim{\sigma}_{\epsilon, s})}
    {\dot{r}^2(\estim{\sigma}_{\epsilon, s}) - \dot{r}^2(\estim{\sigma}_{\epsilon, t})}
    =
    \frac{
        \frac{\estim{\sigma}_{\epsilon, s}\estim{\sigma}_{\epsilon, t}}
        {\sqrt{\estim{\sigma}_{\epsilon, t} + \epsilon^2}}
        -
        \frac{\estim{\sigma}_{\epsilon, s}\estim{\sigma}_{\epsilon, t}}
        {\sqrt{\estim{\sigma}_{\epsilon, s} + \epsilon^2}}
    }
    {
        \frac{\estim{\sigma}_{\epsilon, s}^2}{\estim{\sigma}_{\epsilon, s} + \epsilon^2}
        -
        \frac{\estim{\sigma}_{\epsilon, t}^2}{\estim{\sigma}_{\epsilon, t} + \epsilon^2}
    }
    =
    \frac{\estim{\sigma}_{\epsilon, s}\estim{\sigma}_{\epsilon, t}}{\epsilon}
    \frac{\sqrt{1 - u_{\epsilon, t}} - \sqrt{1 - u_{\epsilon, s}}}
    {u_{\epsilon, s} - u_{\epsilon, t}}
    =
    \frac{\estim{\sigma}_{\epsilon, s}\estim{\sigma}_{\epsilon, t}}{\epsilon^2}
    \frac{\epsilon}{2\sqrt{1 - \bar{u}_\epsilon}}
    \rightarrow 0,
\end{equation*}
where $\bar{u}_{\epsilon}$ is a value between $u_{\epsilon, s}$ and
$u_{\epsilon, t}$ where we use Taylor expansion to $\sqrt{1-x}$. The last convergence to 0 is
obtained based on the same reason as the previous one.

Let $E:=\{kl : k \leq m \text{ or } l \leq m\}$, by Lemma
\ref{lemma:woodbury-block}, we have
\begin{equation*}
    \Hv_\epsilon
    \rightarrow
    \cb{X}_{\cdot,E}\Big[\cb{X}_{\cdot,E}^\top
    \cb{X}_{\cdot,E} + \lambda\cb{G} \Big]^{-1}
    \cb{X}_{\cdot,E}^\top := \Hv,
\end{equation*}
where $\cb{G}$ is defined in \eqref{eq:nuclear-hessian-G}.

Finally, we obtain our approximation of leave-$i$-out prediction by substituting the
above formula of $\Hv$ into the general formula \eqref{eq:matrix-general-alo}.

\begin{remark}
    Similar to what we did in Figure \ref{fig:matrix-general}, it is helpful to
    visualize the structure of $\cb{G}$ in correspondence to the blocks of the
    original matrix. Specifically we have Figure \ref{fig:nuclear-general}.
    \begin{figure}[!th]
        \begin{center}
            \begin{tikzpicture}[scale=1.0]
                \path [fill=cyan] (0, 0) rectangle (1.2, 1);
                \path [fill=darkgray] (1.2, 0) rectangle (2, 1.8);
                \path [fill=olive] (0, 1) rectangle (1.2, 1.8);
                \path [fill=olive] (1.2, 1.8) rectangle (2, 3);
                \path [fill=green] (0, 1.8) rectangle (1.2, 3);
                \path [fill=orange] (0, 3) -- (0.1, 3) -- (1.2, 1.9) -- (1.2,
                1.8) -- (1.1, 1.8) -- (0, 2.9) -- (0, 3);
                \draw (0, 0) rectangle (2, 3);
                \draw [dashed, thick] (0, 1) -- (1.2, 1);
                \draw [dashed, thick] (0, 1.8) -- (2, 1.8);
                \draw [dashed, thick] (1.2, 0) -- (1.2, 3);
                \draw [dashed, thick] (0, 3) -- (1.2, 1.8);
                
                \draw [fill] (0.5, 2.5) circle [radius=0.04];
                \node [left] at (0.5, 2.5) {\tiny $(s_1,s_1)$};
                \draw [fill] (0.3, 2) circle [radius=0.04];
                \node [left] at (0.3, 2) {\tiny $(s_2, t_2)$};
                \draw [fill] (1, 2.7) circle [radius=0.04];
                \node [right] at (1, 2.7) {\tiny $(t_2, s_2)$};
                \draw [fill] (0.6, 1.3) circle [radius=0.04];
                \node [below] at (0.6, 1.4) {\tiny $(s_3, t_3)$};
                \draw [fill] (1.7, 2.4) circle [radius=0.04];
                \node [below] at (1.6, 2.4) {\tiny $(t_3, s_3)$};
                \draw [fill] (0.6, 0.6) circle [radius=0.04];
                \node [below] at (0.6, 0.6) {\tiny $(s_4, t_4)$};
                \draw [fill] (1.5, 0.5) circle [radius=0.04];
                \node [above] at (1.6, 0.5) {\tiny $(s_5, t_5)$};
                
                \draw [ultra thick, ->] (2.5, 1.5) -- (3.8, 1.5);
                
                \path [fill=lightgray] (5, -0.5) rectangle (9, 3.5);
                \path [fill=orange] (5, 3.5) rectangle (5.6, 2.9);
                \path [fill=green] (5.6, 2.9) rectangle (6.5, 2);
                \path [fill=olive] (6.5, 2) rectangle (7.5, 1.0);
                \path [fill=cyan] (7.5, 1.0) rectangle (8.4, 0.1);
                \path [fill=darkgray] (8.4, 0.1) rectangle (9, -0.5);
                \path [fill=white!70!lightgray] (8.4, 0.1) rectangle (9, 3.5);
                \path [fill=white!70!lightgray] (8.4, 0.1) rectangle (5, -0.5);
                \draw [dashed, thin] (5, 3.5) -- (8.4, 0.1);

                \draw [dotted, thick] (5, 2.9) -- (6.5, 2.9);
                \draw [dotted, thick] (5.6, 3.5) -- (5.6, 2);
                \draw [dotted, thick] (5.6, 2) -- (7.5, 2);
                \draw [dotted, thick] (6.5, 2.9) -- (6.5, 1.0);
                \draw [dotted, thick] (6.5, 1.0) -- (8.4, 1.0);
                \draw [dotted, thick] (7.5, 2) -- (7.5, 0.1);
                \draw [dotted, thick] (7.5, 0.1) -- (9, 0.1);
                \draw [dotted, thick] (8.4, 1.0) -- (8.4, -0.5);
                \draw [dotted, thick] (5, 0.1) -- (8.4, 0.1);
                \draw [dotted, thick] (8.4, 0.1) -- (8.4, 3.5);
                \draw (5, -0.5) rectangle (9, 3.5);
                
                \draw [fill] (5.3, 3.2) circle [radius=0.04];
                \draw [dotted, thick] (5, 3.2) -- (5.3, 3.2);
                \draw [dotted, thick] (5.3, 3.5) -- (5.3, 3.2);
                \node [left] at (5, 3.2) {\tiny $(s_1, s_1)$};
                \node [right, rotate=90] at (5.3, 3.5) {\tiny $(s_1, s_1)$};
                
                \draw [fill] (5.9, 2.6) circle [radius=0.04];
                \draw [fill] (5.9, 2.2) circle [radius=0.04];
                \draw [fill] (6.3, 2.2) circle [radius=0.04];
                \draw [fill] (6.3, 2.6) circle [radius=0.04];
                \draw [dotted, thick] (5, 2.6) -- (6.3, 2.6);
                \draw [dotted, thick] (5, 2.2) -- (6.3, 2.2);
                \draw [dotted, thick] (5.9, 3.5) -- (5.9, 2.2);
                \draw [dotted, thick] (6.3, 3.5) -- (6.3, 2.2);
                \node [left] at (5, 2.6) {\tiny $(s_2, t_2)$};
                \node [right, rotate=90] at (5.9, 3.5) {\tiny $(s_2, t_2)$};
                \node [left] at (5, 2.2) {\tiny $(t_2, s_2)$};
                \node [right, rotate=90] at (6.3, 3.5) {\tiny $(t_2, s_2)$};

                \draw [fill] (6.8, 1.7) circle [radius=0.04];
                \draw [fill] (6.8, 1.3) circle [radius=0.04];
                \draw [fill] (7.2, 1.7) circle [radius=0.04];
                \draw [fill] (7.2, 1.3) circle [radius=0.04];
                \draw [dotted, thick] (5, 1.7) -- (7.2, 1.7);
                \draw [dotted, thick] (5, 1.3) -- (7.2, 1.3);
                \draw [dotted, thick] (6.8, 3.5) -- (6.8, 1.3);
                \draw [dotted, thick] (7.2, 3.5) -- (7.2, 1.3);
                \node [left] at (5, 1.7) {\tiny $(s_3, t_3)$};
                \node [right, rotate=90] at (6.8, 3.5) {\tiny $(s_3, t_3)$};
                \node [left] at (5, 1.3) {\tiny $(t_3, s_3)$};
                \node [right, rotate=90] at (7.2, 3.5) {\tiny $(t_3, s_3)$};
 
                \draw [fill] (8.0, 0.5) circle [radius=0.04];
                \draw [dotted, thick] (5, 0.5) -- (8.0, 0.5);
                \draw [dotted, thick] (8.0, 3.5) -- (8.0, 0.5);
                \node [left] at (5, 0.5) {\tiny $(s_4, t_4)$};
                \node [right, rotate=90] at (8.0, 3.5) {\tiny $(s_4, t_4)$};
 
                \node [below] at (5.3, 3.2) {\scriptsize $a$};
                \node [below left] at (5.95, 2.70) {\scriptsize $b$};
                \node [right] at (6.25, 2.22) {\scriptsize $b$};
                \node [below left] at (5.95, 2.25) {\scriptsize $c$};
                \node [above right] at (6.23, 2.55) {\scriptsize $c$};
                \node [below left] at (6.85, 1.80) {\scriptsize $d$};
                \node [right] at (7.15, 1.32) {\scriptsize $d$};
                \node [below left] at (6.85, 1.35) {\scriptsize $e$};
                \node [above right] at (7.15, 1.65) {\scriptsize $e$};
                \node [below] at (8.0, 0.5) {\scriptsize $f$};

                \draw [thin, ->] (9.6, 1.8) -- (8.5, 1.8);
                \draw [thin, ->] (9.6, -0.2) -- (9.1, -0.2);
                \node [right] at (9.6, 1.8) {$\cb{G}$};
                \node [right] at (9.6, -0.2) {removed};

                \node [below] at (1, -0.5)
                {$\estim{\Uv}^\top\estim{\Bv}\estim{\Vv}$};
                \node [below] at (7, -0.5) {$\Qv^\top\nabla^2 R(\Bv)\Qv$};
            \end{tikzpicture}
        \end{center}
        \caption{An illustration of the correspondence between the structure of
            the original matrix and the structure of the $\cb{G}$ matrix.
            As we have mentioned in Theorem \ref{thm:matrix-twice-diff}, $a=0$,
            $b = \frac{1}{\estim{\sigma}_{s_2} + \estim{\sigma}_{t_2}}$,
            $c = - \frac{1}{\estim{\sigma}_{s_2} + \estim{\sigma}_{t_2}}$,
            $d = \frac{1}{\estim{\sigma}_{t_3}}$,
            $e=-\frac{g_r[\estim{\sigma}_{s_3}]}{\estim{\sigma}_{t_3}}$,
            $f = \frac{1}{\estim{\sigma}_{t_4}}$.}
            \label{fig:nuclear-general}
    \end{figure}
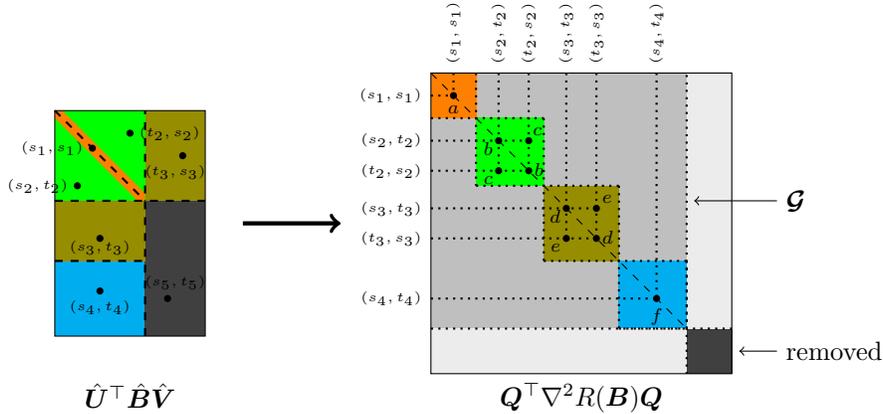
\end{remark}

\section{Details of the Numerical Experiments}
\label{append:sec:experiment}

\subsection{Simulated Data}
\subsubsection{Support Vector Machine}
For all SVM simulations the data is generated according to a Gaussian logistic
model: the design matrix $\Xv$ is generated as a matrix of i.i.d.
$\mathcal{N}(0, 1)$; the true parameter $\betav$ is i.i.d. $\mathcal{N}(0, 9)$,
and each response $y_i$ is generated as an independent Bernoulli with
probability $p_i$ given by the following logistic model:
\begin{equation*}
    \log \frac{p_i}{1 - p_i} = \xv_i^\top \betav.
\end{equation*}

The $n > p$ scenario is generated with $n = 300$ and $p = 80$, and the $n < p$
scenario is generated with $n = 300$ and $p = 600$. We consider a sequence of 
$40$ different values of $\lambda$ ranging between $e^{4}\sim e^{12}$, with their logarithm equally
spaced between $[4, 12]$.

The model is fitted using the \texttt{sklearn.svm.linearSVC} function in Python
package \texttt{scikit-learn} \cite{scikit-learn}, which is implemented by the
\texttt{LibSVM} package \cite{chang2011libsvm}.

For using the \texttt{sklearn.svm.linearSVC}, we set
\texttt{tolerance=$10^{-6}$} and
\texttt{max\_iter=10000}. We identify an observation as a support vector
if $|1 - y_i \xv_i^\top \estim{\betav}| < 10^{-5}$.

\subsubsection{Fused LASSO}
For all fused LASSO, each component of the design matrix $\Xv$ is generated
from i.i.d. $\mathcal{N}(0, 0.05)$. For the true parameter $\betav$, we generated
it through the following process: given a number $k<p$, we generate a
sparse vector $\betav_0$ with a random sample of $k$ of its components i.i.d.
from $\mathcal{N}(0, 1)$. Then we construct a new vector $\betav_1$ as the
cumulative sum of $\betav_0$: $\beta_{1, i} = \sum_{j=1}^i \beta_{0, j}$;
Finally we normalize $\betav_1$ such that it has standard deviation 1. Note
that $\betav_1$ is a piecewise constant vector.
The response $\yv$ is generated as $\yv = \Xv\betav + \bm{\epsilon}$, where
$\bm{\epsilon}$ denotes i.i.d. random gaussian noise from $\mathcal{N}(0, 0.25)$.
For our simulation, we use $k=20$ (so piecewise constant with 20 pieces). The
$n > p$ scenario is generated with $n = 200$ and $p = 100$, whereas the $n < p$
scenario is generated with $n = 200$ and $p = 400$.

The model is fitted through a direct translation of the generalized LASSO model
into the package \texttt{CVX} \cite{cvx}. We use the default tolerance and
maximal iteration. We identify the location $i$ such that $\estim{\beta}_{i+1}
= \estim{\beta}_{i}$ by checking if $|\estim{\beta}_{i + 1} - \estim{\beta}_i|
< 10^{-8}$. For $n > p$, we consider a sequence of 40 tuning parameters from
$10^{-2} \sim 10^2$; For $n < p$, we consider a sequence of 30 tuning
parameters from $10^{-1} \sim 10$. Both are equally spaced on the log-scale.

\subsubsection{Nuclear Norm Minimization}

For all nuclear norm simulations the data is generated according to the Gaussian
low-rank model; each observation matrix $\Xv_j$ is generated as an i.i.d. $\mathcal{N}(0,1)$ matrix. The true
parameter matrix $\Bv$ is generated as a low rank matrix, by setting $k=1$ in
the following formula
\begin{equation*}
    \Bv = \sum_{l = 1}^k \zv_l \wv_l^\top,
\end{equation*}
where $\zv, \wv$ are independent of each other. $\zv \sim \mathcal{N}(0,
\Iv_{p_1})$, $\wv \sim \mathcal{N}(0, \Iv_{p_2})$. Hence, the rank of
$\Bv$ in our experiments is equal to $1$.
The response $\yv$ is generated as $y_j = \langle \Xv_j, \Bv \rangle + \epsilon_j$,
where $\epsilon_j$ is i.i.d. $\mathcal{N}(0, 0.25)$.

The $n > p$ scenario is generated with $n = 600$, and $\Bv \in \RR^{20 \times 20}$ (i.e.
$p = 400$). The $n < p$ scenario is generated with $n = 200$, and $\Bv \in
\RR^{20 \times 20}$ again. For both settings, we consider a
sequence of 30 tuning parameters from $5\times 10^{-1}\sim 5\times 10$, equally
spaced on the log-scale.

The model is fitted using an implementation of a proximal gradient algorithm as
described in \cite{Lan2011}, implemented using the Matlab package
\texttt{TFOCS} \cite{tfocs}. The threshold we use to identify singular values
with value 0 is $10^{-3} \times \lambda_{\max}(\estim{\Bv})$, where
$\lambda_{\max}$ is the maximal singular value of $\estim{\Bv}$.

\subsubsection{LASSO Experiment}\label{sssec:lasso-experiment}
In our LASSO simulations, we use the
setting where $n = 300$, $p = 600$, and the true model is sparse with $k =
60$ non-zeros. These non-zeros are i.i.d. $\mathcal{N}(0, 1)$.  

In the misspecification example, the elements of $\Xv$ are i.i.d. $\mathcal{N}(0, 1 / k)$. $\yv$ is generated according to
the following non-linear model:
\begin{equation*}
    y_j = f(\xv_j^\top \betav + \epsilon_j),
\end{equation*}
where $\bm{\epsilon} \sim \mathcal{N} (\mathbf{0}, 0.25 \Iv_{n})$, and  the function $f$ is given by:
\begin{equation*}
    f(x) = \begin{cases}
        \sqrt{x} & \text{ if } x \geq 0, \\
        -\sqrt{-x} & \text{ otherwise.}
    \end{cases}
\end{equation*}

In the heavy-tailed noise example, the elements of $\Xv$ are i.i.d. $\mathcal{N}(0, 1 / k)$. $\yv$ is generated according to
\begin{equation*}
\yv = \Xv\betav + \bm{\epsilon},
\end{equation*}
where the ``heavy-tailed'' noise $\epsilon_j$ is generated
according to a Student-$t$ distribution with three degrees of freedom, and
rescaled such that its variance is $\sigma^2=0.25$.

In the correlated design example, $\yv$ is generated according to
\begin{equation*}
\yv = \Xv\betav + \bm{\epsilon},
\end{equation*}
where $\bm{\epsilon} \sim \mathcal{N} (\mathbf{0}, 0.25 \mathbf{I})$, and the
``correlated design'' $\Xv$ is generated with each row $\xv_j$ being
sampled independently according to a multivariate normal distribution $\xv_j
\sim \mathcal{N}(0, \Cv / k)$, where $\Cv$ is the Toeplitz matrix, given by:
\begin{equation*}
    \Cv = \begin{pmatrix}
        \rho & \rho^2 & \ldots & \rho^p \\
        \rho^2 & \rho & \ldots & \rho^{p - 1} \\
        \vdots & \ldots & \ddots & \vdots \\
        \rho^p & \rho^{p - 1} & \ldots & \rho
    \end{pmatrix}.
\end{equation*}
$\rho$ is set to $0.8$ in our experiments. For all settings, we consider a
sequence of 25 tuning parameters from $3.16\times 10^{-3} \sim 3.16\times
10^{-2}$, equally spaced under log-scale.

All models were solved using the \texttt{glmnet}  package in
Matlab \cite{qian2013glmnet}. We identify the zero locations of $\estim{\betav}$ by checking
$|\beta_j| > 10^{-8}$.

\subsubsection{Timing of ALO}
For comparing the timing of ALO with that of LOOCV, we consider the LASSO
problem with correlated design similar to the one we introduced in Section
\ref{sssec:lasso-experiment}. Specifically, each row of the design matrix has a
Toeplitz covariance matrix with $\rho = 0.8$. The true coefficient vector
$\betav$ has $\frac{\min(n, p)}{2}$ nonzero components, with each nonzero
component of $\betav$ being selected independently from $\pm 1$ with
probability $0.5$. The noise $\epsilon \sim \mathcal{N}(0, 0.5\Iv_n)$. For each
pair of $(n, p)$, we choose a sequence of 50 tuning parameters ranging from
$\lambda_0$ to $10^{-2.5}\lambda_0$, where $\lambda_0 = \|\Xv^\top
\yv\|_{\infty}$. Note that for this choice of $\lambda$ all the regression
coefficients are equal to zero.

The timing of one single fit on the full dataset, the ALO risk estimates and
the LOOCV risk estimates are reported in Table \ref{tab:timing} of the main
paper. To obtain the timing of a single fit we run the corresponding function
of glmnet along the entire tuning parameter path and record the total time
consumed. This process is then repeated for 10 random seeds to obtain the
average timing. Every time an estimate is obtained we use our formula to obtain
ALO. Hence, the time reported for ALO in Table \ref{tab:timing} is again
obtained from an average of $10$ Monte Carlo samples. To obtain the computation
time of LOOCV, we only use $5$ random seeds.

As expected, averaged time for LOOCV is close to $n$ times the time required
for a single fit. On the other hand, among all the settings we considered in
Table \ref{tab:timing}, ALO takes less than twice the time of a single fit.

\subsection{Real-World Data}
\label{ssec:real-world-data}
In this section, we apply our ALO methods to three real-world datasets: Gisette
digit recognition \cite{guyon2005result}, the tumor colon tissues gene
expression \cite{alon1999broad} and the South Africa heart disease data
\cite{rossouw1983coronary, EOSL:chapter4}. All
the three datasets have binary response, so we consider classification
algorithms. The information of the three datasets is listed in Table
\ref{tab:real-data-meta-info} below. The column of number of effective features
records the number of features after data preprocessing, including removing
duplicates and missing columns.
\begin{table}[!h]
    \begin{center}
        \caption{Information of the three datasets.}
        \label{tab:real-data-meta-info}
        \begin{tabular}{l|p{1.6cm}p{1.6cm}p{1.6cm}l}
            dataset & \# samples & \# features & \# effective features & model used \\
            \hline
            gisette & 6000 & 5000 & 4955 & SVM \\
            tumor colon & 62 & 2000 & 1909 & logistic + LASSO \\
            heart disease & 462 & 9 & 9 & logistic + LASSO
        \end{tabular}
    \end{center}
\end{table}

For gisette, since $n=6000$ is too large for LOOCV, we randomly subsample 1000
observations and apply linear SVM on it. For the tumor colon tissues and South
Africa heart disease dataset, we apply logistic regression with LASSO penalty.
The results are shown in Figure \ref{fig:real-data}. The accuracy of ALO is
verified on gisette and the heart disease dataset. However, the behavior of ALO
is more complicated for the tumor colon tissues dataset. First ALO gives very
close estimates to LOOCV for relatively large tuning values, but deviates from
LOOCV risk estimates and bends upward after $\lambda$ decreases to a certain
value. Second, we note that the optimal tuning is still correctly captured by
ALO.
\begin{figure}[!t]
    \begin{center}
        \setlength\tabcolsep{2pt}
        \renewcommand{\arraystretch}{0.3}
        \begin{tabular}{r|rrr}
            & \multicolumn{1}{c}{\small gisette}
            & \multicolumn{1}{c}{\small heart disease}
            & \multicolumn{1}{c}{\small colon tumor} \\
            \hline
            \rotatebox{90}{\small \hspace{1.1cm} lasso risk} &
            \includegraphics[scale=0.4]{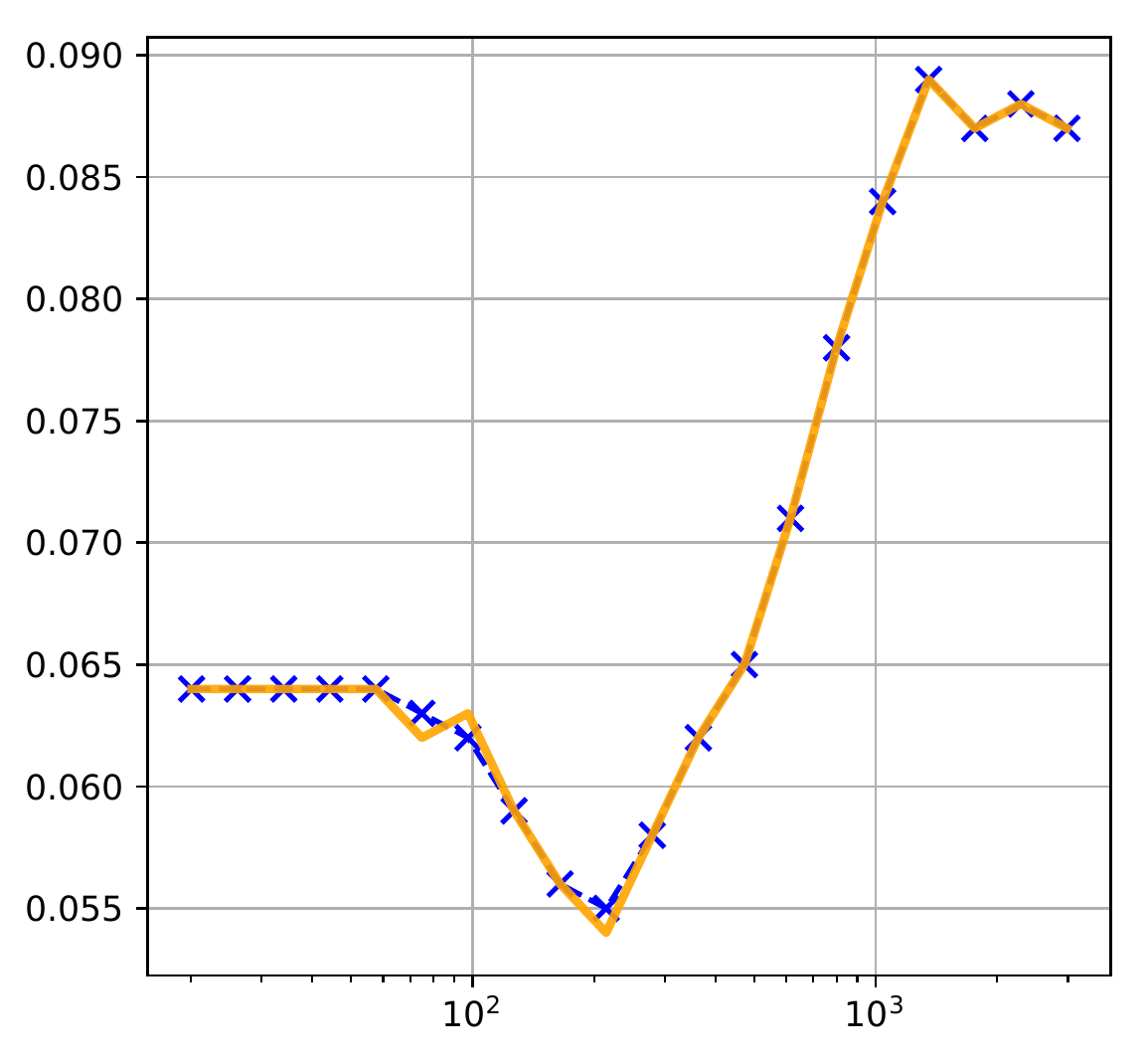} & 
            \includegraphics[scale=0.4]{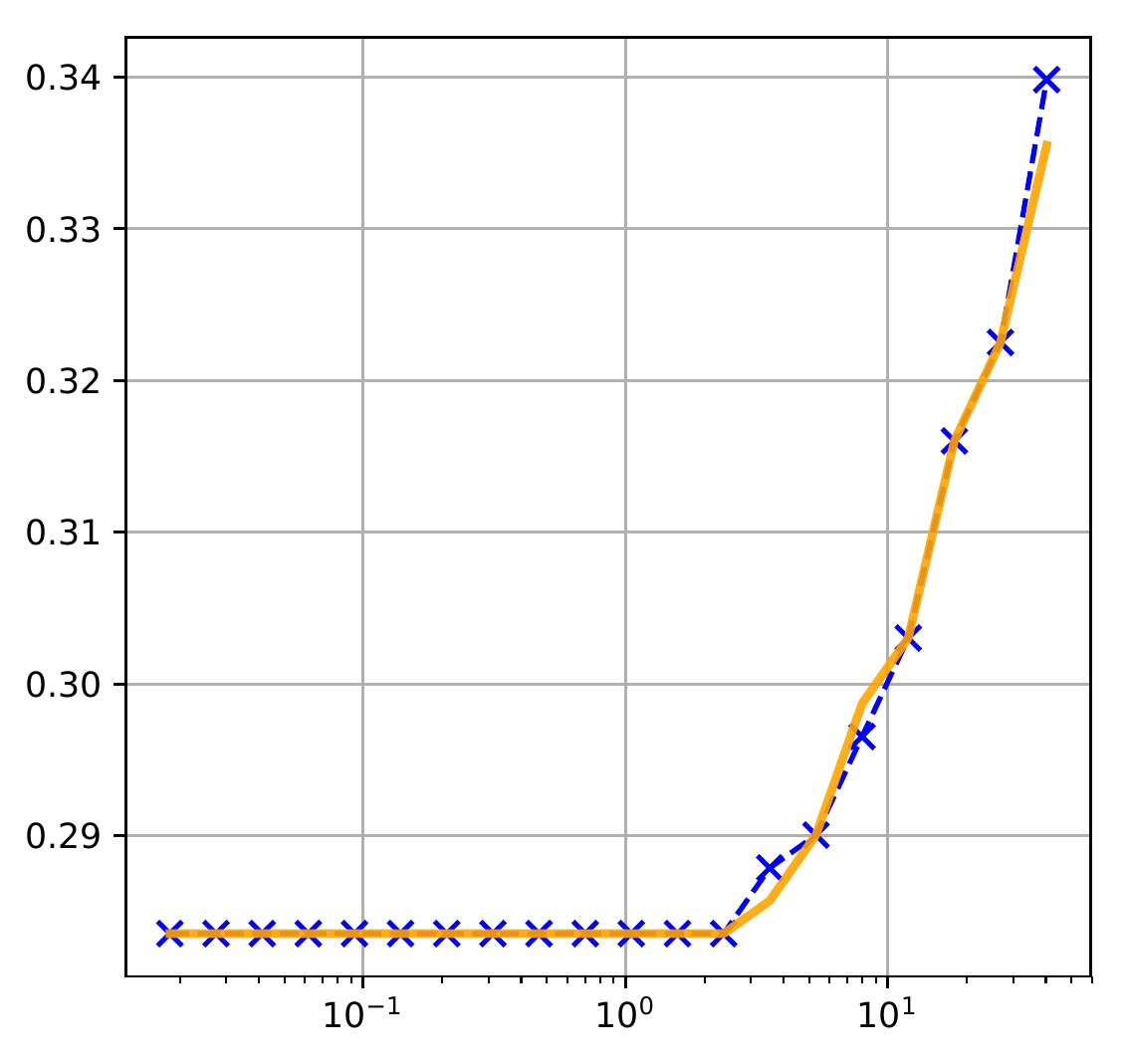} & 
            \includegraphics[scale=0.4]{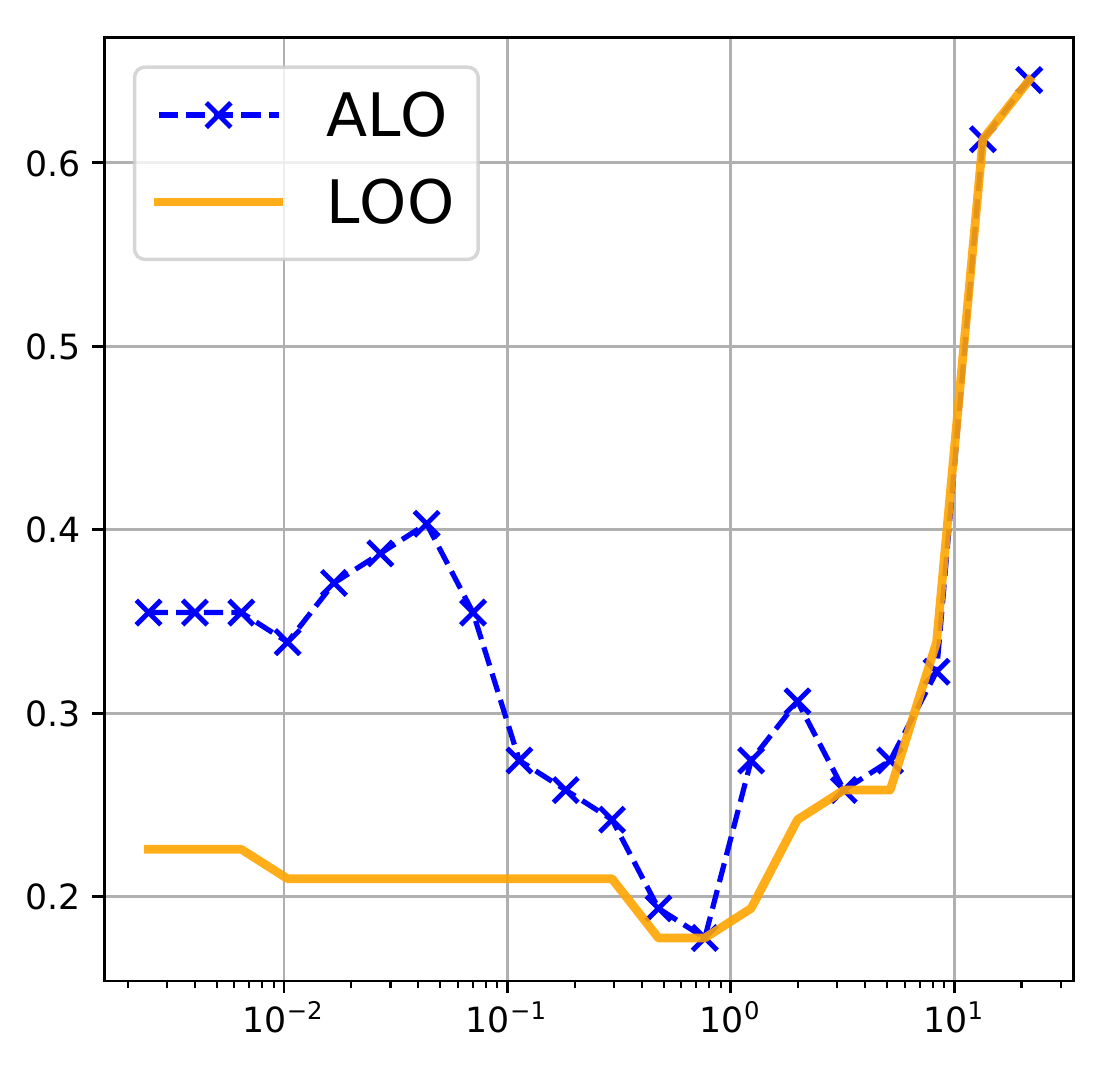} \\
        \end{tabular}
        \caption{Risk estimates of from ALO versus LOOCV for the three
            datasets: gisette, South Africa coronary heart disease and colon
            tumor gene expression. The $x$-axis is the tuning parameter value
            $\lambda$ on $\log$-scale, the $y$-axis is the risk estimates under
            0-1 loss.} \label{fig:real-data}
    \end{center}
\end{figure}

There are a few factors which may affect the performance of ALO. First, as
implied by the theoretical guarantee on smooth models, the closeness between
ALO and LOOCV is a high-dimensional phenomenon, which takes place for
relatively large $n$ and $p$. This requirement of high-dimensionality is less
stringent for $n > p$, when strong convexity of the loss function is to some
extent guaranteed, but becomes more significant for $n < p$. On the other
hand, from our simulation in Section \ref{sec:experiments} and the real-data
examples in this section, we can see that when $\frac{n}{p}$ is not much
smaller than 1 (compared to the $\frac{n}{p}$-ratio in the colon tissue
dataset), a few hundreds of observation and features are enough to guarantee
the accuracy of ALO risk estimates. Finally, the deviation of ALO estimates
tends to happen when the tuning $\lambda$ becomes smaller than a certain value,
typically in the case of $n < p$. As we have mentioned in the main text, for
most nonsmooth regularizers, small tuning values induce dense solutions. In most
high dimensional datasets, these dense solutions are not favorable in $n < p$
case. From our experiments, this deviation mostly happens after correctly
capturing the optimal tuning values.

\end{document}